\documentclass[a4paper,11pt]{article}

\usepackage[utf8]{inputenc}
\usepackage{color}
\usepackage{hyperref}
\usepackage{amsmath}
\usepackage{amsfonts}
\usepackage{lipsum}
\usepackage{amsxtra}
\usepackage{amssymb}
\usepackage{amsmath} 
\usepackage{bbm}
\usepackage{amsfonts}
\usepackage{float}
\usepackage{subfigure}
\usepackage[english]{babel}
\usepackage[onelanguage,linesnumbered,noline,ruled,commentsnumbered]{algorithm2e} 
\usepackage{url}
\usepackage{empheq}
\usepackage{natbib}
\usepackage{authblk}
\usepackage[top=3.5cm, bottom=3.5cm, left=2.5cm, right=2.5cm]{geometry}

\newtheorem{proposition}{Proposition}
\newtheorem{proof}{Proof}
\newtheorem{remark}{Remark}

\def\mub{\boldsymbol{\mu}}
\def\Sig{\boldsymbol{\Sigma}}
\def\xb{\mathbf{x}}
\def\Ab{\mathbf{A}}

\def\ub{\mathbf{u}}

\def\Iden{\mathbf{I}}

\newcommand\Ind{1\!\!1}
\def\mub{\boldsymbol{\mu}}
\def\Sig{\boldsymbol{\Sigma}}
\def\xb{\mathbf{x}}
\def\Ab{\mathbf{A}}
\def\ub{\mathbf{u}}

\def\Iden{\mathbf{I}}

\definecolor{highlight1}{RGB}{204, 52, 41}
\definecolor{highlight2}{RGB}{41, 193, 204}

\begin{document}
\title{A Flexible EM-like Clustering Algorithm for Noisy Data}

\author[1,2]{Violeta Roizman}
\author[2,3]{Matthieu Jonckheere}
\author[1]{Frédéric Pascal}
\affil[1]{Laboratoire des Signaux et Systèmes (L2S), CentraleSupélec-CNRS-Université Paris-Sud, Université Paris-Saclay, 3, rue Joliot Curie, 91192, Gif-sur-Yvette, France}
 \affil[2]{Instituto de Cálculo,
Universidad de Buenos Aires, Intendente G\"{u}iraldes 2160, Ciudad Universitaria-Pabellón II, Buenos Aires, Argentina} 
  \affil[3]{IMAS-CONICET, Buenos Aires, Argentina}



\maketitle

\begin{abstract}%
Though very popular, it is well known that the EM for GMM algorithm suffers from non-Gaussian distribution shapes, outliers and high-dimensionality. In this paper, we design a new robust clustering algorithm that can efficiently deal with noise and outliers in diverse data sets. As an EM-like algorithm, it is based on both estimations of clusters centers and covariances. In addition, using a semi-parametric paradigm, the method estimates an unknown scale parameter per data-point. This allows the algorithm 
to accommodate for heavier tails distributions and outliers without significantly loosing efficiency in various classical scenarios. 
We first derive and analyze the proposed algorithm in the context of elliptical distributions, showing in particular important insensitivity properties to the underlying data distributions.
We then study the convergence and accuracy of the algorithm by considering first synthetic data. Then, we show that the proposed algorithm outperforms other classical unsupervised methods of the literature such as $k$-means, the EM for Gaussian mixture models and its recent modifications or spectral clustering when applied to real data sets as MNIST, NORB and \textit{20newsgroups}.
\end{abstract}




\section{Introduction}
\label{sec:intro}
The clustering task consists in arranging a set of elements into groups with homogeneous properties/features that capture some important structure of the whole set. As other unsupervised learning tasks, clustering has become of great interest due to the considerable increase in the amount of unlabeled data in the recent years. 
As the characteristics of real-life data---in geometrical and statistical terms---are very diverse, an intensive research effort has been dedicated to define various clustering algorithms which adapt to some particular features and structural properties. We refer to \cite{henning} and  \textcolor{black}{the clustering review by \cite{sklearn}}, 
for discussions on the different methods and on how to choose one depending on the settings.
Among the different types of clustering algorithms, 
{the  Expectation-Maximization (EM) procedure to estimate the parameters of an underlying Gaussian Mixture Model (GMM) \citep[see  for instance the review work by][]{MCLACHLAN1982199} is a very popular method as its model-based nature typically allows other algorithms to be outperformed when the data is low dimensional and the clusters have elliptical shapes.
This model represents the distribution of the data as a random variable given by a mixture of Gaussian distributions. The corresponding clustering criterion is simple: all points drawn from a given normal distribution are considered to belong to the same cluster. The  Expectation-Maximization algorithm (EM) \citep{EMalgo} is a general statistical method used to estimate the parameters of a probabilistic model, based on the maximization of the likelihood. It is an iterative algorithm with two main steps: the expectation part and the maximization part.} In particular for the GMM case, closed-form expressions exist to obtain parameters estimations at the maximization step.\\

However, its performance decreases significantly in various scenarios of particular interest for machine learning applications:
\begin{itemize}

\item
When the data distribution has heavier (or lighter) tails than the Gaussian one and/or in presence of outliers or noise as in Figure \ref{badnoise} \citep[see for instance][]{badperf}. This phenomenon can be simply explained by the non-robustness of the estimators that are computed by the algorithm: means and sample covariance matrices \citep{maronna}.

\item
The presence of different scales in the data might
complicate the global ordering of the observations around their closest
centers (for instance through Mahalanobis distances).
The usual normalization procedure for the estimation of covariance matrices might be too rigid
to get satisfactory clustering results in the presence of significant variability intra and inter-clusters \citep{Tclust}.

\item
When the dimension increases (even in the Gaussian case), the estimation of the covariance matrix is crucially affected by the high-dimensionality as it has been shown by \cite{BOUVEYRON201452}. Some solutions in that direction include regularization and parsimonious models that restrict the shape of the covariance matrix in order to decrease the number of parameters to be estimated \citep{CELEUX1995781}.

\end{itemize}
In order to improve the performance of the GMM-EM clustering algorithm in the context of noisy and diverse data, two main strategies were contemplated. One consists in modifying the model to take into account the noise and the other one is to keep the original model and replace the estimators by others that are able to deal with outliers \citep{nongaussian2}.
In that line of research, several variations of the Gaussian mixture model have been developed. 
In particular, some variations target the problem of mixtures of more general distributions, which allow to model a wider range of data, and possibly allowing for the presence of noise and outliers. 
Regarding the use of non-Gaussian distributions, \cite{Peel2000} proposed an important model defined as a mixture of multivariate $t$-distributions. In this work, the authors suggested an algorithm ($t$-EM or EMMIX in the literature) to estimate the parameters of the mixture with known and unknown degrees of freedom by maximizing the likelihood and addressed the clustering task. More recently, \cite{hyper, hyper2, hyper3} considered hyperbolic and skew $t$-distributions. \\

\begin{figure}[tb]
    \centering
    \includegraphics[width = 0.8\textwidth]{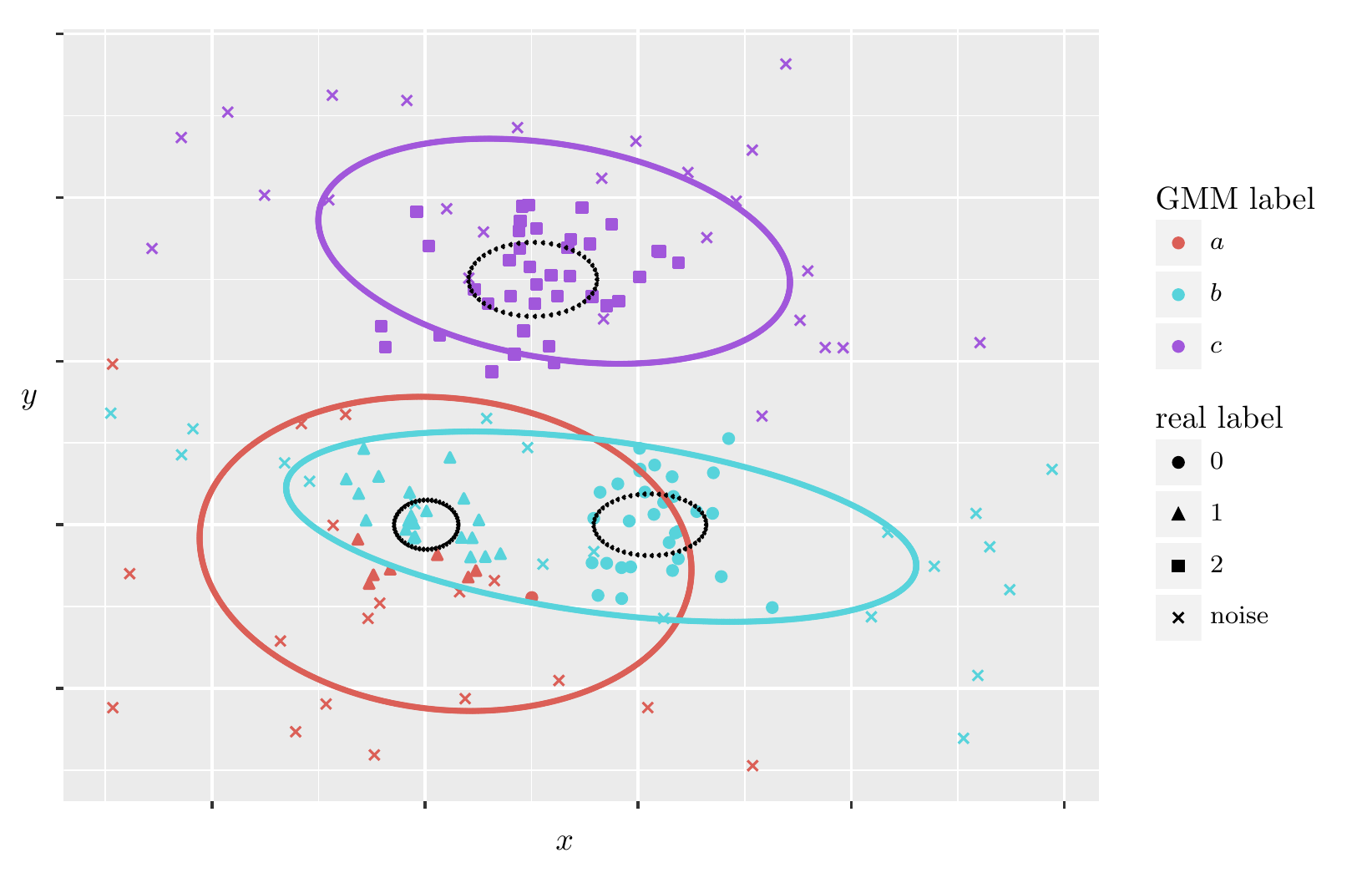}  
    \caption{Clustering result of applying the classic EM for GMM in the presence of uniform noise. The shape of each observation represents the real label (the cross represents the noise). The color of each point represents the assigned label. The dashed ellipses represent the real clusters and the solid  ellipses represent the contours of the estimated distributions.}
    \label{badnoise}
\end{figure}


Other robust clustering approaches worth mentioning are models which add an extra term to the usual Gaussian likelihood and algorithms with modifications inspired by usual robust techniques as robust point estimators, robust scales, weights for observations and trimming techniques. For instance, \cite{unif} considered the presence of a uniform noise as background while \cite{RIMLE} proposed RIMLE, a pseudo-likelihood based algorithm that filters the low density areas.  \cite{Spatial-EM} replaced the usual mean and sample covariance by the spatial median and the rank covariance matrix (RCM).  \cite{ktau} introduced a robust scale that is used to define a $k$-means-like algorithm that can deal with outliers. Furthermore, \cite{Gonzalezmezcla} proposes a robust mixture of distributions estimation based on robust functionals. Moreover, in the work of \cite{weights1}, \cite{weights2} and \cite{weighted} different weights for the observations were proposed where small weights correspond, as usual in the robust literature, to observations that are far from the cluster centers. Finally, trimming algorithms such as TCLUST \citep{Tclust} leave out a proportion of data points that are far from all the means in order to better estimate the parameters in the M-step.\\

This article aims at defining an algorithm that can both outperform traditional ones under an assumption of diverse data and be adaptive to a large class of underlying distributions. Following the path of robust statistical approaches, we propose to complement it by using a semi-parametric setting, allowing us to reach an 
 important flexibility for the data distributions.
Our method is also inspired by the robust applications of the Elliptical Symmetric (ES) distributions \citep{boente, ollila2012complex}.
Of course, elliptical distributions have been widely used in many applications where Gaussian distributions were not able to approximate the underlying data distribution because of presence of heavy tails or outliers \citep{CFAR, radar}. This general family includes, among others, the class of compound-Gaussian distributions that contains  Gaussian, $t-$ and $k-$ distributions \citep{CG, CG2, c1} as well as the class of Multivariate Generalized Gaussian Distributions \citep{pascal2013parameter}.\\ 


In this paper, we present 
\begin{itemize}
    \item A general mixture model, involving one scale parameter per data point, leading to an important flexibility in the resulting clustering algorithm. While not being parameters of interest for the clustering task, those parameters are estimated and their estimators are fully analyzed since they play an indirect role in the clustering algorithm; 
    \item A clustering algorithm with the following characteristics: 1/ it follows the two steps expectation and maximization of EM algorithms, 2/ at the E-step, it provides estimated conditional probabilities, robust in the sense that the expected conditional log-likelihood leads to estimators independent of the distributions shapes 3/ at the M-step, it derives estimations of clusters centers and covariance matrices which turn out to be robust.
\end{itemize}

There are hence two types of estimations.  On the one hand, as all EM-like algorithms, we perform an estimation of the parameters of interest: clusters proportions, means and covariance matrices.
On the  other hand, we use the  estimation of scale (or nuisance) parameters, (which are not of direct interest) to improve the estimations of the parameters of interest as well as robustify the estimation of the probability for an observation to belong to a given cluster. More precisely, we show in this paper that, under mild assumptions, those probabilities estimates do not depend on the shape of data distributions, making the algorithm generic, simple and robust. It can be noticed that the scale/nuisance parameters could also be used for classification and
outlier detection purposes by discriminating data and helping data assignment \citep{mahalanobiseusip} .\\ 

A key feature of the proposed algorithm is to be self-contained in the sense that no \textit{extra-parameters} need to be tuned as it is the case for aforementioned approaches (\textit{e.g.,} penalty parameters, rejection thresholds,\textcolor{black}{ and} other distribution parameters such as shapes or the degrees of freedom).\\

In the sequel, we include practical and theoretical studies that provide evidence about the algorithm performance.
In particular, we theoretically justify the efficiency of our algorithm using various arguments:
\begin{enumerate}
    \item  When the underlying model belongs to the class of elliptical distributions, with different means and dispersion matrix per cluster but with cluster-independent density generators (even different ones within clusters)
    then the estimation of membership probabilities does not depend on each specific density function. This is a consequence of the fact that those probabilities estimations do not depend on the scale factors of the covariance matrix but only on the scatter/dispersion matrices. Hence, the algorithm makes no mismatch error when the density generator is unknown, whenever this assumption is fulfilled. This is shown in Proposition \ref{prop.pik}
    
    \item
Even when the density function is different for every cluster, there are regimes, where the mismatch error can be controlled. We give an example using  $t-$distributions with various degrees of freedom. See Proposition \ref{prop.student}.
    
\item
    Finally, though the estimation of covariance matrix becomes clearly challenging in high-dimensional settings, estimations of the nuisance parameters get typically more accurate and faster when the dimension grows large, using a simple law of large numbers in the dimension. See Proposition \ref{prop.hd}.

\end{enumerate}


 
 From a practical perspective,
 the induced clustering performance is largely improved compared to k-means, the EM algorithm for GMM and HDBSCAN \citep{HDBSCAN1, HDBSCAN2, mcinnes2017hdbscan} 
  when applied to real data sets such as MNIST variations \citep{MNIST}, NORB \citep{NORB} and \textit{20newsgroups} \citep{20newsgroup}. In agreement with the proposed results, previous works on classification of the MNIST dataset suggest the non-gaussianity of the clusters \citep{Liao2017ALD}. Compared to spectral clustering and $t$-EM, TCLUST and RIMLE our algorithm performs similarly in classic cases and much better in others. Furthermore, the proposed algorithm is able to provide accurate estimations of location and dispersion parameters even in the presence of heavy tailed distributions or additive noise as proved in simulations where our algorithm beats the other compared models. \\

The rest of the paper is organized as follows. 
In Section 2, after introducing in details the models of interest, we present the clustering algorithm and discuss some of its important aspects, notably by proving convergence results on the parameters estimation. Section 3 is devoted to the experimental results, which allow us to show the improved performance of the proposed method for different synthetic and real data sets in comparison with other commonly used methods. Finally, conclusions and perspectives are stated in Section 4. \\



\section{Model and Theoretical Justifications}\label{sec.themodel}

In this section, we present a detailed description of the underlying theoretical model and the proposed clustering algorithm. Given $\{\mathbf{x}_i\}_{i=1}^n$ a set of $n$ data points in $\mathbb{R}^m$, let us start by considering them as independent samples drawn from a mixture of distributions with the following probability density function (pdf):

\begin{equation}
f_{i}(\textbf{\textit{x}}_i)=\sum_{k=1}^K \pi_k f_{i,\theta_k}(\textbf{\textit{x}}_i) \ \ \text{with} \ \ \sum_{k=1}^{K}\pi_k=1,
\label{model}
\end{equation}
where $\pi_k$ represents the proportion of the $i^{\mathrm{\underline{th}}}$ distribution associated with some parameters $\theta_k$ in the mixture.
The notation $f_{i,\theta_k}$ is used for simplicity and stands for a pdf $f_{i,k,\theta_k}$ that may depend in principle on cluster $k$ and in general on some ``cluster parameters'' grouped in $\theta_k$ as well as on some extra nuisance parameter $\tau_{ik}$. We remark that the subscript $i$ is used in $f_{i,\theta_k}$ to stress that distributions can be different from one observation to another.

\begin{remark}
Let us underline the level of generality of the model: the $K$ clusters are only characterized by parameters $\theta_k$ while the shape of the distributions can change from one observation to another. We explain the relevance of such a general structure in the next paragraph, where we fix a set of distributions for the $f_{i,\theta_k}$.
\end{remark}

In the sequel, we consider a very large class of distributions in order to generalize the classical Gaussian mixture model: the Elliptically Symmetric (ES) distributions. The pdf of an $m$-dimensional random vector $\mathbf{x}_i$ that is ES-distributed with mean $\mub_k$ and covariance matrix $\tau_{ik}\Sig_k$ can be written as
\begin{equation}
\label{pdf}
f_{i,\theta_k}(\textbf{\textit{x}}_i)= A_{ik} |\Sig_k |^{-1/2}\tau_{ik}^{-m/2}\,g_{i,k}\left(\frac{
(\textbf{\textit{x}}_i-\mub_k)^T \Sig_k^{-1}(\textbf{\textit{x}}_i-\mub_k)}{\tau_{ik}}\right) 
\end{equation} 
where $A_{ik}$ is a normalization constant,
$g_{i,k}:[0,\infty)\rightarrow [0,\infty)$ is any function (called the density generator) such that \eqref{pdf} defines a pdf. The matrix $\Sig_k$ reflects the structure of the covariance matrix of $\textbf{\textit{x}}_i$.  Note that the covariance matrix is equal to $\Sig_k$ up to a scale factor if the distribution has a finite second-order moment  \citep[see for details][]{ollila2012complex}. This is denoted $ES(\mub_k, \tau_{ik}\Sig_k, g_{i,k}(.))$.
Note that the clustering parameter is $\theta_k= (\pi_k, \mub_k,\Sig_k) $ while the nuisance parameter is $\tau_{ik}$. For convenience, we denote by $\theta=(\theta_1, \ldots,\theta_K)$ the set of all clustering parameters.\\

At this stage, some comments have to be mentioned:
\begin{enumerate}
\item
When $g_{i,k}=g, \ \forall i,k$, then we retrieve a classic paradigm for clustering modelling. All the points follow the same ES distribution but each class has a different mean and covariance matrix. However, note that this model is much more general than a Gaussian mixtures model as the class of ES distributions is much wider and includes in particular lighter and heavier tails than Gaussian ones. An important aspect of our results is that our algorithm is in that case {\bf insensitive} to the function $g$, and hence allows
to treat efficiently real data sets where $g$ is not known.
\item
When $g_{i,k}=g_i, \ \forall i,k$, we obtain a much more general model than the previous one, where, even if the distribution of the points do not depend on the classes except for their mean and covariances, the data within a class might follow e.g., a mixture of ES distributions. It has a practical importance since many data sets are compiled from different sources of data with different characteristics. In that case again, our algorithm is {\bf still insensitive} to the functions $g_i$ giving a lot of modelling flexibility and mismatch robustness.
\item When $g_{i,k}=g_k, \ \forall i,k$, then we consider one different ES distribution per class of data. For this non-standard settings, the clustering results {\bf do depend on $g_k$ which can be a practical obstacle to get sound results}. However, we show that our method can alleviate this dependence leading to good performance in some regimes.
\end{enumerate}

Elliptical distributions have been used in many applications where  one has to deal with the presence of heavy tails or outliers \citep{CFAR, radar}. This general family includes Gaussian, $t-$ and $k-$ distributions, among others \citep{CG, CG2, c1} .  Such modelling admits a Stochastic Representation Theorem.  A vector $\mathbf{x}_i \sim ES(\mub_k, \tau_{ik}\Sig_k, g_{i,k}(.))$ if and only if it admits the following stochastic representation \citep{Yao73}
\begin{equation}\label{rep-thm}
\mathbf{x}_i \overset{d}{=}\mub_k + \sqrt{Q_{ik}} \sqrt{\tau_{ik}}\Ab_k \ub_i,
\end{equation}
where the non-negative real random variable $Q_{ik}$, called the modular variate, is independent of the random vector $\ub_i$ that is uniformly distributed on the unit $m$-sphere and $\Ab_k\Ab_k^T$ is a factorization of $\Sig_k$ while  $\tau_{ik}$ is a deterministic but unknown nuisance parameter. \\

Note that, in this work, one considers that $\mathbf C_{ik} = \tau_{ik}\, \Sig_k$ can also depend on the $i^{\mathrm{\underline{th}}}$ observation, through the nuisance parameter. Now, for identifiability purposes, we assume that the distributions at hand have a second-order moment and that $\mathbf C_{ik}$ is the covariance matrix.
This assumption implies the particular normalization on $Q_{ik}$, that is 
$$E[Q_{ik}] = \text{rank}(\mathbf C_{ik}) \left( = \text{rank}(\Sig_k)\right) = m, \text{ when $\Sig_k$ is full rank},$$ 
following for instance \cite{ollilatyler}. In the sequel, we hence call $\mathbf C_{ik}$ the covariance matrix and $\Sig_k$ the scatter matrix.\\

Finally, an ambiguity remains in the scatter matrix $\Sig_k$. Indeed, for any positive real number $c$, $\left(\tau_{ik}, \Sig_k\right)$ and $\left(\tau_{ik}/c, c\,\Sig_k\right)$ lead to the same covariance matrix $\mathbf C_{ik}$. In this work, we choose to fix the trace of $\Sig_k$ to $m$. Other normalizations could have been chosen instead as for instance imposing a unit-determinant for $\Sig_k$ without affecting the clustering results.\\


\cite{ollilatyler} showed in the complex case that, given random sample from $\mathbf{x}_i \sim CES(\mathbf{0}_m, \tau_{ik}\Sig_k, g_{i,k}(.))$, the estimation of $\tau_{ik}$ using Maximum Likelihood Estimation (MLE) is decoupled from the estimation of $\Sig_k$.
Furthermore, the authors proved that the maximum likelihood estimator for $\Sig_k$ is the Tyler's estimator, regardless the functions $g_i$.
This is a remarkable result, underlying the universal character of the Tyler estimator in this class of distributions.
We will build on this distribution-free property of the Tyler's estimator which turns out to be central for our results.

\subsection{The M-step: Parameter Estimation for the Mixture Model}\label{model2}
 
Similarly to the EM for GMM, we extend the model with $n$ discrete variables $Z_i$ (with $i=1 \dots n$), that are not observed (corresponding to the so-called latent variables), representing the cluster label of each observation $\mathbf{x}_i$. We compute the label for each observation and cluster in the E-step, while in the M-step we estimate the parameters of interest 
$\theta=(\theta_k)_{k=1}^K$.\\

Given a sample $\textbf{\textit{x}}=(\textbf{\textit{x}}_1^T,...,\textbf{\textit{x}}_n^T)^T$, a set of parameters $\theta$, and the latent variables $Z=(Z_1,...,Z_n)^T$.  The expected conditional log-likelihood of the model is
 \begin{eqnarray}
&& E_{Z|\textbf{\textit{x}},\theta^{*}}[l(Z, \textbf{\textit{x}};\theta)] = \ \ \sum_{i=1}^n \sum_{k=1}^K P_{i, \theta*}(Z_i=k|\mathbf{x}_i=\textbf{\textit{x}}_i)\log(\pi_k f_{i,\theta_k}(\textbf{\textit{x}}_i)) 
 \label{exp-log-likelihood} \\
 && = \sum_{i=1}^n\sum_{k=1}^K  p_{ik} \left[\log(\pi_k) + \log(A_{ik})
+ \log\left(|\textbf{C}_{ik}|^{-1/2}g_i((\textbf{\textit{x}}_i-\mub_k)^T\textbf{C}_{ik}^{-1}(\textbf{\textit{x}}_i-\mub_k))\right)\right],    \nonumber
\end{eqnarray}
where, $p_{ik}=P_{i, \theta^*}(Z_i=k|\mathbf{x}_i=\textbf{\textit{x}}_i)$ with $\displaystyle \sum_{k=1}^K p_{ik} = 1$ and $\textbf{C}_{ik}=\tau_{ik} \Sig_k$. 

We now include two propositions that summarize the derivation of the estimators for all the parameters of the model. As underlined previously, a key step using the ideas in \cite{ollilatyler}, consists in factorizing the likelihood into two factors which further allows to describe fundamental properties of the estimators in the E and M steps. 
In Proposition \ref{proptaus}, we derive the estimator for the $\tau$ parameters. Then, in Proposition \ref{propmusigma}, we derive the rest of the parameters of the model.

\begin{proposition} \label{prop:est_tau}
Suppose $\mathbf x_1,...,\mathbf x_n$ an independent sample with $\mathbf x_i \sim ES(\mub_k, \tau_{ik}\Sig_k, g_{i,k}(.))$ for some $k \in \{1,...,K\}$. 
Suppose $\int t^{m/2}g_{i,k}(t) dt < \infty, \ \forall i,k$.
Then, the derivation of the maximum likelihood estimation of the $\tau_{ik}$ parameters is decoupled from the one of the rest of the estimators. For fixed parameters $\Sig_k$ and $\mub_k$, the $\tau_{ik}$'s estimators are computed as 
\begin{equation}
\widehat{\tau}_{ik} = \frac{(\mathbf{{x}}_i-\mub_k)^T\Sig_{k}^{-1}(\mathbf{{x}}_i-\mub_k)}{a_{ik}}, \   \forall 1 \leq i \leq n \text{ \ and \ } \forall 1\leq k \leq K, 
\label{eq.taua}
\end{equation}
\label{proptaus}
with $a_{ik} = \underset{t}{\arg\sup}\{t^{m/2}g_{i,k}(t)\}$.
\end{proposition}

\begin{proof}
See Appendix \ref{app-proof-prop1}.
\end{proof}

    

    

We now describe the ML estimators for $\theta$.

\begin{proposition}\label{propmusigma}
Given an independent random sample $\mathbf{x}_1,...,\mathbf{x}_n$, the latent variables $Z_i$, and expected conditional log-likelihood of the model stated before, the maximization w.r.t. $\theta_k$ for $k=1,\hdots, K$, leads to  the following equations that the estimators have to fulfill. The closed equations

\begin{equation}
\widehat{\pi}_k =  \frac{1}{n}\sum_{i=1}^n p_{ik}
 \label{pij}
\end{equation}
for the proportion of each distribution, 

\begin{equation}
\widehat{\mub}_k =  \sum\limits_{i=1}^n  c_{ik} \,\mathbf{x}_i   
\ \ \text{with} \ \ 
\displaystyle c_{ik}=\cfrac{\cfrac{p_{ik}}{(\mathbf{x}_i-\widehat{\mub}_k)^T\widehat{\Sig}_{k}^{-1}(\mathbf{x}_i-\widehat{\mub}_k)}}{\displaystyle\sum\limits_{l=1}^n\cfrac{p_{lk}}{(\mathbf{x}_l-\widehat{\mub}_k)^T\widehat{\Sig}_{k}^{-1}(\mathbf{x}_l-\widehat{\mub}_k)}},
\label{cij}
\end{equation}
for the mean of each distribution, and

\begin{equation}
\displaystyle\widehat{\Sig}_k = m \sum_{i=1}^n \frac{w_{ik}(\mathbf{x}_i-{\widehat{\mub}_k})(\mathbf{x}_i-{\widehat{\mub}_k})^T}{(\mathbf{{x}}_i-\widehat{\mub}_k)^T\widehat{\Sig}_{k}^{-1}(\mathbf{{x}}_i-\widehat{\mub}_k)}     
\ \ \text{with} \ \ 
\displaystyle w_{ik}=\cfrac{p_{ik}}{\displaystyle\sum_{l=1}^n p_{lk}},  
\label{wij}
\end{equation}
for the scatter matrices.

\end{proposition}

\begin{proof}
See Appendix \ref{app-proof-prop2}.
\end{proof}
\label{testi}


It follows from the derivation of Proposition \ref{propmusigma} that there is a system of two fixed-point equations given by

\begin{equation}
\widehat{\mub}_k = \cfrac{\sum\limits_{i=1}^n \cfrac{p_{ik} \mathbf{x}_i}{(\mathbf{x}_i-\widehat{\mub}_k)^T\widehat{\Sig}_k^{-1}(\mathbf{x}_i-\widehat{\mub}_k)}}{\sum\limits_{i=1}^n \cfrac{p_{ik}}{(\mathbf{x}_i-\widehat{\mub}_k)^T\widehat{\Sig}_k^{-1}(\mathbf{x}_i-\widehat{\mub}_k)}} 
\label{fix1}
\end{equation}
and 

\begin{equation}
\widehat{\Sig}_k= m \sum_{i=1}^n \frac{w_{ik}(\mathbf{x}_i-\widehat{\mub}_k)(\mathbf{x}_i-\widehat{\mub}_k)^T}{(\mathbf{x}_i-\widehat{\mub}_k)^T\widehat{\Sig}_k^{-1}(\mathbf{x}_i-\widehat{\mub}_k)},    
\label{fix2}
\end{equation}
that hold for the estimators of $\mub_k$ and $\Sig_k$, 
with $w_{ik}$ defined in Eq.~\eqref{wij}. In order to obtain these linked estimators, this system is iteratively solved as explained in Section \ref{sec.implement}.\\

We can now prove a fundamental property of the algorithm which is the monotonicity of the likelihood of the model. We later illustrate this property with simulations in Section \ref{sec.implement}. To establish more precise guarantees of convergence we would need a data-driven approach as it was developed for instance in \citep{wu2016convergence}. We leave this analysis for future work.

\begin{proposition}\label{prop.mono}
Given the expected log-likelihood in \eqref{exp-log-likelihood}, the observed likelihood 
$$l(\textbf{x};\theta) = \sum_{i=1}^n \log \sum_{k=1}^K \pi_k f_{i,\theta_k}(\textbf{\textit{x}}_i),$$ 
and assuming the convergence of the fixed-point equations system derived in Proposition \ref{propmusigma}, the steps defined by the estimator updates from Propositions \ref{proptaus} and \ref{propmusigma} lead to a succession $\{\theta^{t}\}_{t=1}^N$ with an  increasing likelihood.
\end{proposition}

\begin{proof}
See Appendix \ref{app-proof-mono}.
\end{proof}

It is important to notice that the derivation of estimators in our model results in usual robust estimators for the mean and covariance matrices.
More specifically, both can be assimilated to $M$-estimators with a certain $u$ function \citep{maronna}. Actually, both the expressions for the mean and the scatter matrix estimators are very close to the corresponding Tyler's $M$-estimator  
 \citep[see][for more details]{tyler1987distribution, frontera2014hyperspectral}. Main differences arise from the mixture model that leads to different weights involved by the different distributions. However, in case of clusters with equal probability, i.e., $p_{ik} = 1/K$ for $k=1,\hdots, K$ and $i=1,\hdots, n$, one retrieves exactly the Tyler's $M$-estimator for the scatter matrix while the mean estimator differs only from the square-root at the denominator
(see the explanation later on).
Although our estimators are derived as usual MLE (but) for parametrized (thanks to the $\tau_{ik}$) elliptical distributions, they are intrinsically robust. Indeed, as detailed in \citep{Bilodeau1999}, Tyler's and Maronna's $M$-estimators can either be obtained through MLE approaches for particular models (e.g., Student-$t$ $M$-estimators) or directly from other cost functions (e.g., Huber $M$-estimators) and all those estimators are by definition robust.

 Thus, this approach can be seen as a generalization of Tyler's $M$-estimators to the mixture case. Indeed, one has for $\widehat{\mub}_k$
$$\displaystyle \frac{1}{n}\sum\limits_{i=1}^n u_1\left( (\mathbf{x}_i-\widehat{\mub}_k)^T\widehat{\Sig}_k^{-1}(\mathbf{x}_i-\widehat{\mub}_k)\right) (\mathbf{x}_i-\widehat{\mub}_k) = \mathbf{0}, \text{ with } \displaystyle u_1(t)=\frac{p_{ik}}{t},$$
while $\widehat{\Sig}_k$ can be written as $$\widehat{\Sig}_k = \displaystyle\frac{1}{n}\sum\limits_{i=1}^n u_2\left((\mathbf{x}_i-\widehat{\mub}_k)^T\widehat{\Sig}_k^{-1}(\mathbf{x}_i-\widehat{\mub}_k)\right) (\mathbf{x}_i-\widehat{\mub}_k)(\mathbf{x}_i-\widehat{\mub}_k)^T$$ with $u_2(t)= \cfrac{m \, w_{ik}}{t}$ where $w_{ik} = \cfrac{n\, p_{ik}}{\sum_{l=1}^n p_{lk}}$.\\

This similarity to classical Tyler's estimators explains the robust character of our proposal. Indeed, the difference lies in the weights terms $p_{ik}$ and $w_{ik}$ appearing in the $u_j(.)$ functions traditionally introduced in the robust statistics literature. These naturally implies that those $u_1(.)$ and $u_2(.)$ functions continue to respect Tyler's conditions (although  \cite{tyler1987distribution} used $u_1(t^{1/2})$ instead of $u_1(t)$, see \cite{Bilodeau1999} for more details). 


The convergence of the fixed-point equations defining the $M$-estimators has been shown in \cite{maronna} but under a restrictive assumption on the $u$ function, which is not fulfilled in our case. On the other hand, Kent proved in \cite{kent1991redescending} that for fixed mean, there is convergence of the fixed-point equation for the covariance estimator under a normalization constraint. Finally, he also showed that for some $u$ function, the joint mean covariance estimations boil down to a constrained covariance estimation. Unfortunately, this trick does not work in our case. Hence, the case of joint convergence of the fixed-point equations for the Tyler's estimators is still an open-problem in statistics even in the case of one distribution (no mixture).

We later perform analysis and simulations that confirm  the robustness of the algorithm in practice. In particular, the setups in Section \ref{sec.syn} include distributions with heavy tails, different distributions and noise.\\


\subsection{The E-step: Computing the Conditional Probabilities}
\label{general-pik-case}
In contrast to the estimators derived in Proposition  \ref{propmusigma}, \eqref{eq.taua} shows that the estimation of the $\tau_{ik}$ parameters are linked to the functions $g_{i,k}$ that characterizes the corresponding Elliptical Symmetric distribution. 
We now give a central result for our algorithm.
The following proposition shows that the $p_{ik}$'s estimators do not depend on density generators when $g_{i,k}=g_i$.
\begin{proposition}\label{prop.pik}
Given an independent random sample $\mathbf{x}_i \sim ES(\mub_k, \tau_{ik}\Sig_k, g_i(.))$ for some $k \in {1,...,K}$, the resulting estimated conditional probabilities $\widehat{p}_{ik} = \widehat{P}_{\theta_k}(Z_i=k|\mathbf{x}_i=\textbf{\textit{x}}_i)$ have the following expression for all $i=1,\hdots, n$ and $k=1,\hdots, K$:
\begin{equation}
    \label{eq.generalpij2}
    \widehat{p}_{ik} = \frac{\widehat{\pi}_k \left((\mathbf{{x}}_i-\widehat{\mub}_k)^T\widehat{\Sig}_k^{-1}(\mathbf{{x}}_i-\widehat{\mub}_k)\right)^{-m/2}|\widehat{\Sig}_k |^{-1/2}} {\displaystyle \sum_{j = 1}^K \widehat{\pi}_j \left((\mathbf{{x}}_i-\widehat{\mub}_j)^T\widehat{\Sig}_j^{-1}(\mathbf{{x}}_i-\widehat{\mub}_j)\right)^{-m/2}|\widehat{\Sig}_j |^{-1/2}},
\end{equation}
where $\widehat{\pi}_k$, $\widehat{\mub}_k$ and $\widehat{\Sig}_k$ are given in Proposition \ref{propmusigma}.\\
\end{proposition}
\begin{proof}
See Appendix \ref{app-proof-prop-pik}.
\end{proof}

\begin{remark}~

\begin{itemize}

    \item Result of Proposition \ref{prop.pik} is of utmost importance since it allows to derive the conditional probabilities required in the E-step independently of the distributions $g_i$'s and of the $\tau_{ik}$'s parameters. In other words, for any independent ES-distributed observation $\textbf{\textit{x}}_i$ with mean $\mub_k$ and covariance matrix $\tau_{ik}\Sig_k$, a unique EM algorithm is derived that does not depend on the shapes of the various involved distributions. This is essential because the absence of precise knowledge on the specific data distribution is the most usual situation in a real life applications, while estimating it might degrade significantly the performance.
    \item Secondly, it evidences the fact that the particular normalization of the $\Sig$ estimator does not affect the probability computation in the E-step. In other words,  the normalization of the scatter matrices are not relevant for the clustering results. On the other hand, the normalization of $\widehat{\Sig}$ does affect the scale of the $\tau_{ik}$ parameters. Thus, using them to classify points or reject outliers needs to be treated with care and is out of the scope of this paper. 
    \item The particular case where the data points arise from a mixture of one ES distribution, $g_i=g, \, \forall 1\leq i \leq n$, is contained in Proposition \ref{prop.pik}. We remark the particular example, included in this case, when all the distributions are Gaussian. 
    If $\mathbf x_i \sim \mathcal{N}(\mub_k, \tau_{ik} \Sig_k)$ then the corresponding density generator is $g(t) = e^{-t/2}$. The corresponding maximizer is $\underset{t}{\arg\sup}\{t^{m/2}g(t)\} = m$, consequently the estimator is, as derived in \eqref{eq.taua}, as follows:
\begin{equation}
\widehat{\tau}_{ik} = \frac{(\mathbf{{x}}_i-\widehat{\mub}_k)^T\widehat{\Sig}_{k}^{-1}(\mathbf{{x}}_i-\widehat{\mub}_k)}{m}.
\end{equation}
\item The case where $g_i = g_k$ cannot be directly handled as a particular case of Proposition \ref{prop.pik}. Indeed, assuming each class is drawn by a common ES distribution $g_k$ implies in general  that extra-parameters, such as, for instance, the degree of freedom $\nu_k$ for $t$-distributions, the shape parameters for the $K$-distributions and for the generalized Gaussian distributions,  depend on $k$. Those parameters have to be estimated in the M-step. We give an example in the next section in the particular case of mixture of $t$-distributions.
\end{itemize}
\label{rem.pij}
\end{remark}

\subsection{Different Density Generator per Class}\label{sec.gk}


When the density generator depends on the  class,
our computations show that the $p_{ik}$
do depend on the $g_{k}$, as opposed to the previous case.\\

When the density generators are known (which is quite unrealistic in practice), this assumption naturally increases the clustering performance since extra \textit{a priori} information is added to the model. On the contrary, it implies a performance loss when the real data distribution is not the assumed one. \\

To illustrate the type of dependence reached in that case,
we derive the E-step for the particular case of a mixture of multivariate $t-$distributions with different degrees of freedom $\nu_k$. That is, the case where there are $K$ different $g_k$ functions, one for each cluster. The probability density function of each distribution is given by

\begin{eqnarray*}
f_{i, \theta_k}(\textbf{\textit{x}}_i) &=& \cfrac{\Gamma(\frac{\nu_k+m}{2})}{\Gamma(\frac{\nu_k}{2})|\Sig_k|^{1/2}} (\nu_k  \pi  \tau_{ik})^{-m/2} \left[1 + \frac{(\textbf{\textit{x}}_i - \mub_k)^T\Sig_k^{-1}(\textbf{\textit{x}}_i - \mub_k)}{\tau_{ik} \nu_k}\right]^{-(\nu_k+m)/2}.
\end{eqnarray*}

The next proposition states a quantitative approximation of the estimated conditional probabilities in terms of the ``Gaussian value'' (i.e. the value obtained for class independent $g_{i}$), when the $\nu_k$ parameters and $m$ grow at the same rate.

\begin{proposition}\label{prop.student}
Given an independent sample of a mixture of $K$ $t-$distributions, with $\mathbf{x}_i \sim t_{\nu_k}$, $\nu_k$ being the degrees of freedom. If for each $k,  \ \frac{\nu_k}{m} \approx c_k $, then

$$\widehat{p}_{ik}=\frac{ \widehat{\pi}_k \widehat{L}_{0ik} \sqrt{\frac{c_k}{1+c_k}} }{\sum\limits_{j=1}^K  \widehat{\pi}_j \widehat{L}_{0ij} \sqrt{\frac{c_j}{1+c_j}} } + O\left(\frac{1}{m}\right)$$

\end{proposition}

\begin{proof}
See Appendix \ref{app-proof-prop4}.
\end{proof}

This scenario includes of course the case where all the $\nu_k$ are equal (See Remark \ref{rem.pij}). Additionally, it includes when the degrees of freedom are large, with fixed dimension $m$, the Gaussian case as detailed in the proof. Finally, the other intermediate situations where all the $\nu$ parameters do not differ much from the dimension $m$, are hence shown to be very close to the Gaussian computation. If neither of these conditions apply, an \textit{ad hoc} estimation of the $\nu_k$ is of course possible and it has to be performed in the M-step.

\subsection{High-Dimensional Regime and estimation of $\tau_{ik}$}\label{HD}

\subsubsection{Gaussian data}

Related to Proposition \ref{prop.student} that considers the case where $m$ and $\nu_k$ grow at the same rate, we study in this section how the estimation of the nuisance parameters behaves when the dimension grows. Of course, it is well-known that the breakdown-point of the $\Sig$ estimator  gets smaller when the dimension grows. Nevertheless, an underlying law of large numbers allows to show  that the larger the dimension $m$, the better the $\tau$ estimation performance. For Gaussian data and under mild assumptions, if we take $\xb_i$ drawn from the cluster $k$, we can show that the $\widehat{\tau}_{ik}$ estimator converges to the true value of $\tau_{ik}$ when $m$ grows with $n$. This is more rigorously stated in the following proposition.


\begin{proposition}\label{prop.hd}
Suppose that
$$\mathbf x_i=\mub_k+\sqrt{\tau_{ik}}\Ab_k\mathbf{q}_i,$$
with a deterministic $\tau_{ik}\geq 0$, $\Ab_k^T\Ab_k = \Sig_k$, $\text{rank}(\Sig_k)=m$ and $\mathbf{q}_i \sim \mathcal{N}(0, \Iden_m)$. Assume that there exists a sequence of random variables $(t_i)_{i \in \mathbb{N}}$ that converges in distribution such that, for $\alpha \geq 0$, $n^\alpha\widehat{\mub}_k^T\widehat{\mub}_k \leq t_n$ and that $\widehat{\mub}_k$ converges in probability to $\mub_k$. Then, $(\widehat{\tau}_{ik} - \tau_{ik}) \underset{}{\sim} \mathcal{N}(0, 2\tau_{ik}^2/m)$ when $m$ and $n$ are large enough and fulfill the inequality $n>m(2m-1)$.    
\end{proposition}

\begin{proof}
See Appendix \ref{app-proof-prop5}
.
\end{proof}

\begin{remark}
\textcolor{black}{
\begin{itemize}
\item First, in the case where the mean parameter is known, recent Random Matrix Theory results \citep{6891244,COUILLET201556,ZHANG2016114} are in agreement with this phenomenon and prove results for Maronna's and Tyler's $M$-estimators when $m$ and $n$ grow together at a fixed rate, i.e., $m/n \to \gamma \in[0,1]$.  
\item Secondly, \textit{Proposition \ref{prop.hd}} gives theoretical justification for obtaining better results in high-dimensional settings since in such cases $\tau_{ik}$'s parameters will be more accurately estimated.
\end{itemize}}
\end{remark}

\subsubsection{Non-Gaussian data}

In the case of more general elliptic distributions,
one looses in general the convergence (in $m$) of $\widehat \tau_{ik}$ to a deterministic value.
Still, $\widehat \tau_{ik}$ converges under mild assumptions towards a limit which can be used in principle to handle
outliers detections via confidence intervals.
As it is not the main topic of this paper,
we just state a result for compound Gaussian distributions, illustrating the effect of the large dimension and the
type of research that could be fostered in future work.

\begin{proposition}
Suppose that
$$\mathbf x_i=\mub_k+\sqrt{\eta_i\,\tau_{ik}} \Ab_k\mathbf{q}_i,$$
with a deterministic $\tau_{ik}\geq 0$, $\Ab_k^T\Ab = \Sig_k$, $\text{rank}(\Sig_k)=m$, $\mathbf{q}_i \sim  \mathcal{N}(0, \Iden_m)$, and $\eta_i$ a positive random variable independent of $\mathbf q_i$.
Assume further the consistence of $\widehat{\mub}_k, \widehat{\Sig}_k$ and that 
$$\underset{t}{\arg\sup} \{ t^{m/2} g_i(t)\} /m \to 1 \text{ when } \ m \to \infty.$$
Then
 $$\widehat \tau_{ik} \xrightarrow[\ m \to \infty]{prob.}  \tau_{ik}\, \eta_i.$$
\end{proposition}

\begin{proof}
The Proof follows the same lines as the proof of Proposition \ref{app-proof-prop5} and is omitted.
\end{proof}

\begin{remark}
Note that as usual in ML estimation, one can find counter examples where $\underset{t}{\arg\sup} \{ t^{m/2} g_i(t)\}$ is not equivalent to $m$ when $m$ grows large. This is however a quite pathological situation and the assumption of the proposition is fulfilled for most practical cases. 
\end{remark}

\subsection{Implementation Details and Numerical Considerations}\label{sec.implement}

The general structure of the proposed algorithm is the same as the one of the classical EM for GMM. The differences between both algorithms lie in the $\hat p_{ij}$ expression and the recursive update equations for the parameter estimations. We design slightly different variations of the M-step and study the convergence, precision and speed. We do this considering that the estimators for $\mub$ and $\Sig$ are weighted versions of the classic estimators. More precisely, based on equations \eqref{fix1} and \eqref{fix2}, we propose four alternatives thinking in accelerating the convergence speed. These versions depend on two different aspects. One aspect consists in using the just computed estimation of the mean or the estimator from the previous iteration of the loop. The other facet is proposed to emphasize the weights of the data points in the computation of the estimators based on the Tyler's estimator. In Section \ref{sec.themodel}, we mention that the location and scatter estimators are close to Tyler's up to  the square root of the Mahalanobis distance when the location is unknown. We propose to modify the weights by adding this square root in order to mimic Tyler's estimator.
The different versions are defined as follows:

\begin{enumerate}
    \item Version 1: the parameter $\mub$ used to compute the estimator $\Sig$ is the one obtained in the \textbf{same iteration} of the fixed-point loop. 
    \item Version 2: the $\mub$-parameter is the one obtained in the \textbf{previous iteration}.
    \item Version 3: we propose an \textbf{accelerated method} where the quadratic forms \\ $(\mathbf{x}_i-\widehat{\mub}_k)^T\widehat{\Sig}_k^{-1}(\mathbf{x}_i-\widehat{\mub}_k)$ in the denominators of the fixed-point $\mub$ equations are replaced by their square root, corresponding to the original Tyler's $M$-estimators.
    \item Version 4: we implement the same \textbf{acceleration procedure} on top of the algorithm of Version 2. 
\end{enumerate}
For concreteness, in \textbf{Algorithm \ref{algochowliu}} we describe the complete algorithm in Versions 1 (left) and 4 (right). In the particular case described in Section \ref{sec.gk}, where all $g_{ik}$ functions are known, the $p_{ik}$ should be computed with the Bayes expression as in \eqref{eq.bayes}.\\
\begin{algorithm}[p]
\SetAlgoNoLine
\SetKwInOut{Input}{Input}
\SetKwInOut{Output}{Output}
\Input{Data  $\{\textbf{\textit{x}}_i\}_{i=1}^n$, $K$ the number of clusters}
\Output{Clustering labels $\mathcal{Z}=\{z_i\}_{i=1}^n$} 
Set initial random values  $\theta^{(0)}$\;
$l \leftarrow 1$\;
  \While {not convergence}{
  	\textbf{E-step:} Compute  $p^{(l-1)}_{ik}=P_{i, \theta^{(l-1)}}(Z_i=k|\mathbf{x}_i=\textit{\textbf{x}}_i)$ for each $1 \leq k \leq K$
  	\\ 



\begin{equation*}
    p_{ik}^{(l)} = \frac{\pi^{(l-1)}_k \left((\textbf{\textit{x}}_i-\mub^{(l-1)}_k)^T(\Sig_k^{(l-1)})^{-1}(\textbf{\textit{x}}_i-\mub^{(l-1)}_k)\right)^{-m/2}|\Sig^{(l-1)}_k |^{-1/2}} {\displaystyle \sum_{j = 1}^K \pi^{(l-1)}_j \left((\textbf{\textit{x}}_i-\mub^{(l-1)}_j)^T(\Sig_j^{(l-1)})^{-1}(\textbf{\textit{x}}_i-\mub^{(l-1)}_j)\right)^{-m/2}|\Sig^{(l-1)}_j |^{-1/2}}
\end{equation*}



  	\textbf{M-step:} \\
\ \ \ \ \  \textbf{For each} $1 \leq k \leq K$: \\
\ \ \ \ \ \ \ \ \ \  Update $\pi_k^{(l)}=\frac{1}{n}\sum_{i=1}^n p_{ik}^{(l)}$ and compute $w_{ik}^{(l)}= \frac{p_{ik}^{(l)}}{\sum_{l=1}^n p_{lk}^{(l)}}$\;
\ \ \ \ \ \ \ \ \ \  Set $\mub_k^{'}=\mub_k^{(l-1)}$ and $\Sig_k^{'}=\Sig_k^{(l-1)}$; 

\smallskip

\ \ \ \ \  \ \ \ \ \ \textbf{while} \textit{not convergence} \textbf{do}  	

\noindent
\begin{minipage}[t]{0.42\textwidth}
\begin{equation*}
\mub_k^{''} = \frac{\sum\limits_{i=1}^n \cfrac{p_{ik}^{(l)} \textit{\textbf{x}}_i}{(\textit{\textbf{x}}_i-\mub^{'}_k)^T({\Sig}_k^{'})^{-1}(\textit{\textbf{x}}_i-\mub^{'}_k)}}{\sum\limits_{i=1}^n \cfrac{p_{ik}^{(l)}}{(\textit{\textbf{x}}_i-\mub^{'}_k)^T(\Sig^{'}_k)^{-1}(\textit{\textbf{x}}_i-\mub^{'}_k)}}
\end{equation*}

\begin{equation*}
\hspace{-0.4cm}
\Sig^{''}_k= m \sum_{i=1}^n \frac{w_{ik}^{(l)}(\textit{\textbf{x}}_i-\mub^{{'}}_k)(\textbf{\textit{x}}_i-\mub^{{'}}_k)^T}{(\textit{\textbf{x}}_i-\mub^{'}_k)^T(\Sig_k^{'})^{-1}(\textit{\textbf{x}}_i-\mub^{{'}}_k)}    
\end{equation*}

\end{minipage}%
\hspace{0.3cm}
\vrule
\hspace{0.1cm}
\begin{minipage}[t]{0.42\textwidth}
\begin{equation*}
\mub_k^{''} = \frac{\sum\limits_{i=1}^n \cfrac{p_{ik}^{(l)} \textit{\textbf{x}}_i}{\textcolor{highlight1}{\Big[}(\textit{\textbf{x}}_i-\mub^{'}_k)^T(\Sig_k^{'})^{-1}(\textit{\textbf{x}}_i-\mub^{'}_k)\textcolor{highlight1}{\Big]}^{\textcolor{highlight1}{1/2}}}}{\sum\limits_{i=1}^n \cfrac{p_{ik}^{(l)}}{\textcolor{highlight1}{\Big[}(\textit{\textbf{x}}_i-\mub^{'}_k)^T(\Sig^{'}_k)^{-1}(\textit{\textbf{x}}_i-\mub^{'}_k)\textcolor{highlight1}{\Big]}^{\textcolor{highlight1}{1/2}}}}
\end{equation*}

\begin{equation*}
\Sig^{''}_k= m \sum_{i=1}^n \frac{w_{ik}^{(l)}(\textit{\textbf{x}}_i-\mub^{\textcolor{highlight1}{''}}_k)(\textit{\textbf{x}}_i-\mub^{\textcolor{highlight1}{''}}_k)^T}{(\textit{\textbf{x}}_i-\mub^{\textcolor{highlight1}{''}}_k)^T(\Sig_k^{'})^{-1}(\textit{\textbf{x}}_i-\mub^{\textcolor{highlight1}{''}}_k)}    
\end{equation*}

\end{minipage}

\hspace{0.2cm}

 \ \ \ \ \ \ \ \ \ \  \textbf{end}
 
 \ \ \ \ \ \ \ \ \ \ Update $\mub_k^{(l)} = \mub_k^{''}$ and $\Sig_k^{(l)} = \Sig_k^{''}$ and $\tau_{ik}^{(l)}$: 	

  	$l \leftarrow l+1$\;
  	}
 Set $z_i$ as the index $k$ that has the maximum $p_{ik}$ value\;  
 \caption{Scheme of the F-EM algorithm. The Version 1 of the M-step is on the left and the Version 4 is on the right. The differences are highlighted in red in the Version 4.}
 \label{algochowliu}
\end{algorithm}


The plots in Figure \ref{conv} show the convergence of the fixed-point equations for the estimation of $\mub$ and $\Sig$ for the different versions of the algorithm in two different setups with two distributions each. We separately study the two parameters in order to see if the variations differently affect each convergence. The first one is a simple case with two well-separated Gaussian distributions in dimension $m=10$ (means equal to $\mathbf{0}_m$ and 2*$\mathbf{1}_m$, and covariance matrices are the identity $\mathbf{\Iden}_m$ and a diagonal matrix with elements  0.25, 3.5, 0.25, 0.75, 1.5, 0.5, 1, 0.25, 1, 1). The second one is a mixture of two $t$-distributions with heavy tails and the same parameter ($\nu_1=\nu_2=3$). As one can see in both cases, the convergence is reached for all versions of the algorithm after approximately twenty iterations of the fixed-point loop. We see in Figure \ref{conv} that the speed of convergence is improved for Versions 3 and 4, as expected. On the other hand, we study in Figure~\ref{conv.like} the evolution of the log-likelihood in these two different scenarios. In the case of the multivariate $t$-distributions, we computed the likelihood with the true degrees of freedom ($\nu_k = 3$). This Figure shows an increasing likelihood in all cases and a faster convergence of the Version 1 of the model because the correct values for the mean/scatter estimators are reached faster even though its computation takes a bit longer than for Versions 3 and 4. The estimation accuracy of Versions 1 and 2 are by construction (ML-based) better than for Versions 3 and 4. Based on these figures and previous studies about fixed-point fast convergence \citep[see e.g.,][]{Fred}, Version 1 is kept since it follows the original proposal, and although slightly slower than Version 2 for the fixed-point loop, it is faster for the convergence of the algorithm. Furthermore, we fixed the number of iterations to 20 in all the experiments. Notice that increasing this number does not result in a significant increase in terms of clustering performance.  \\

\begin{figure}[!ht]
\centering
\subfigure{
\includegraphics[width=2.91in]{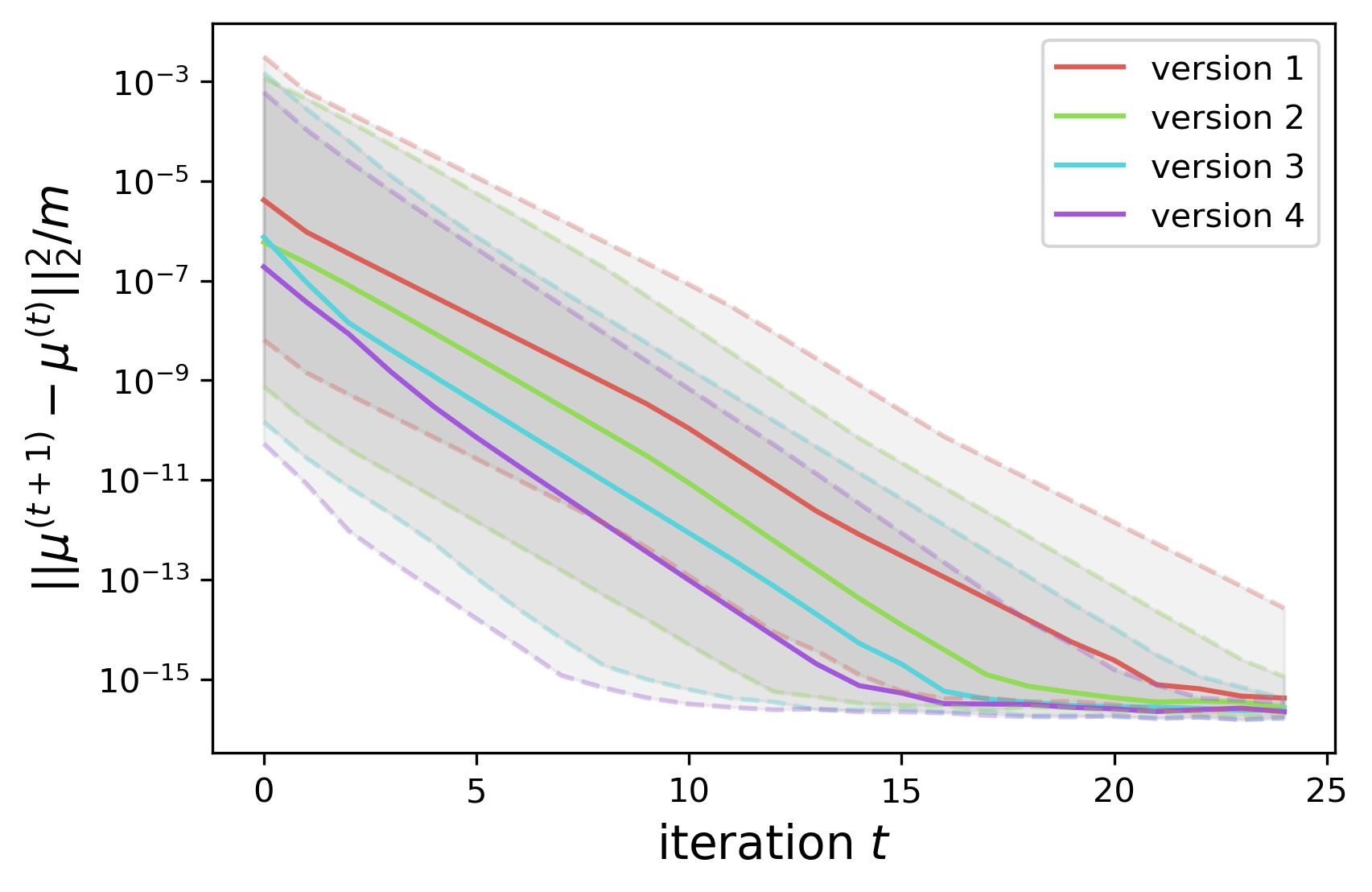}}
\subfigure{
\includegraphics[width=2.91in]{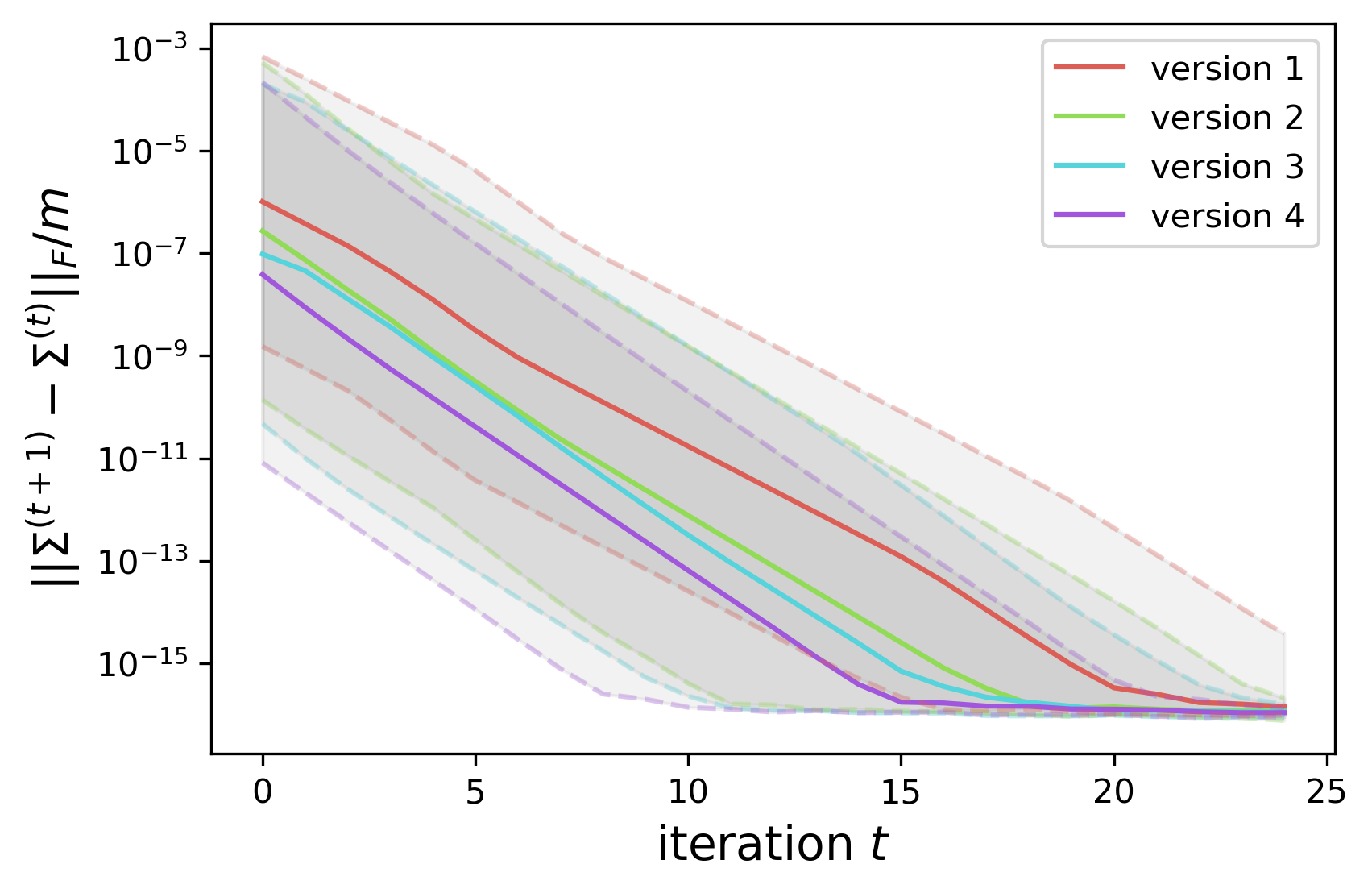}}

\subfigure{
\includegraphics[width=2.91in]{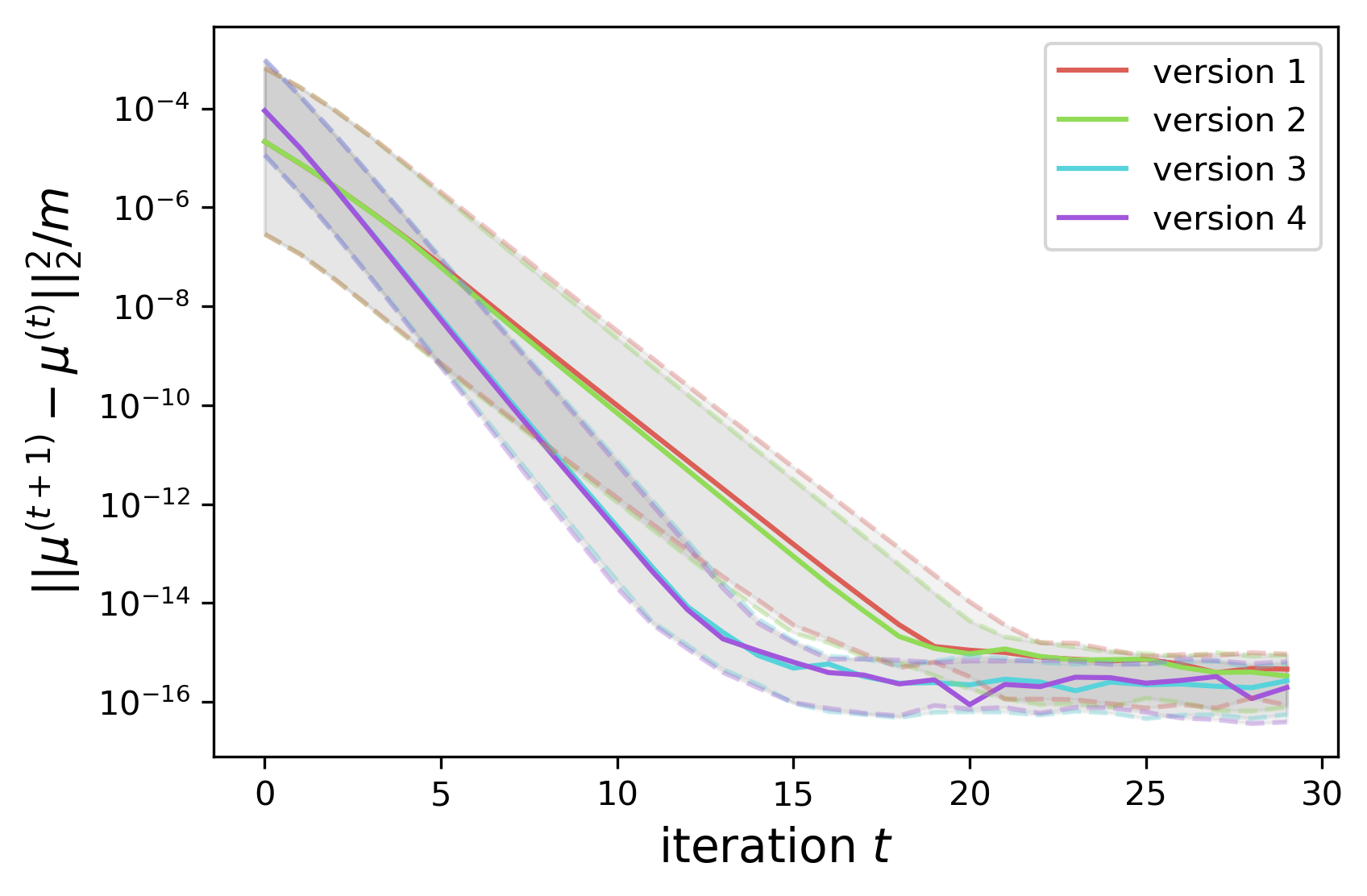}}
\subfigure{
\includegraphics[width=2.91in]{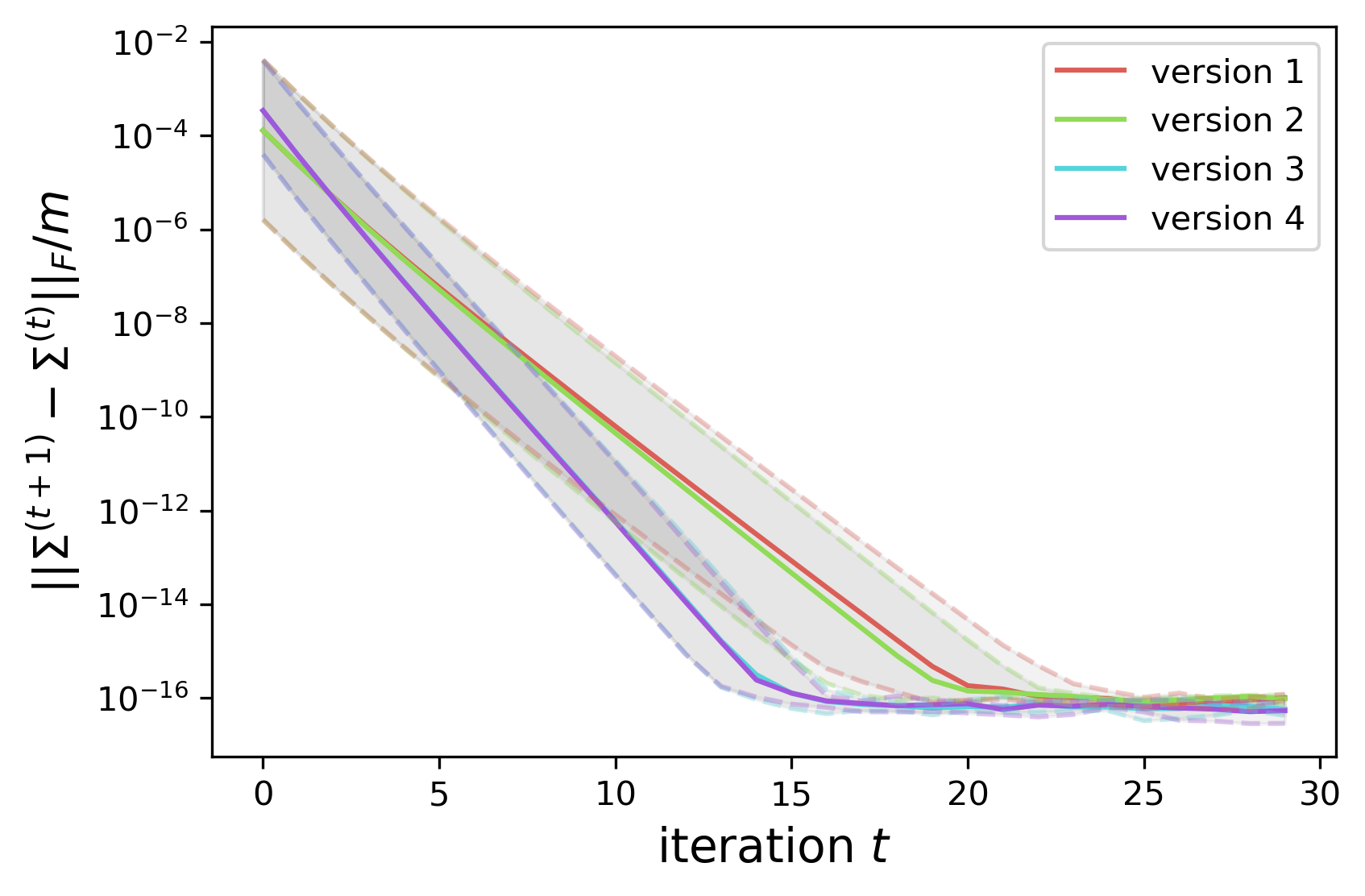}}
\centering
\caption{Convergence speed of the fixed-point equations for the estimation of $\mub$ (left) and $\Sig$ (right). The results in Gaussian case are plotted on the top and on the bottom and the ones for a mixture of  $t$-distributions are shown on the bottom. Each line represents the median of the values obtained on each iteration of the fixed-point iteration for all the iterations of the F-EM algorithms and for all clusters. The gray areas represent the quartile range of each iteration of each version.\vspace*{1cm}}
\label{conv}
\end{figure}  

\begin{figure}[!ht]
\centering
\includegraphics[width=0.48\textwidth]{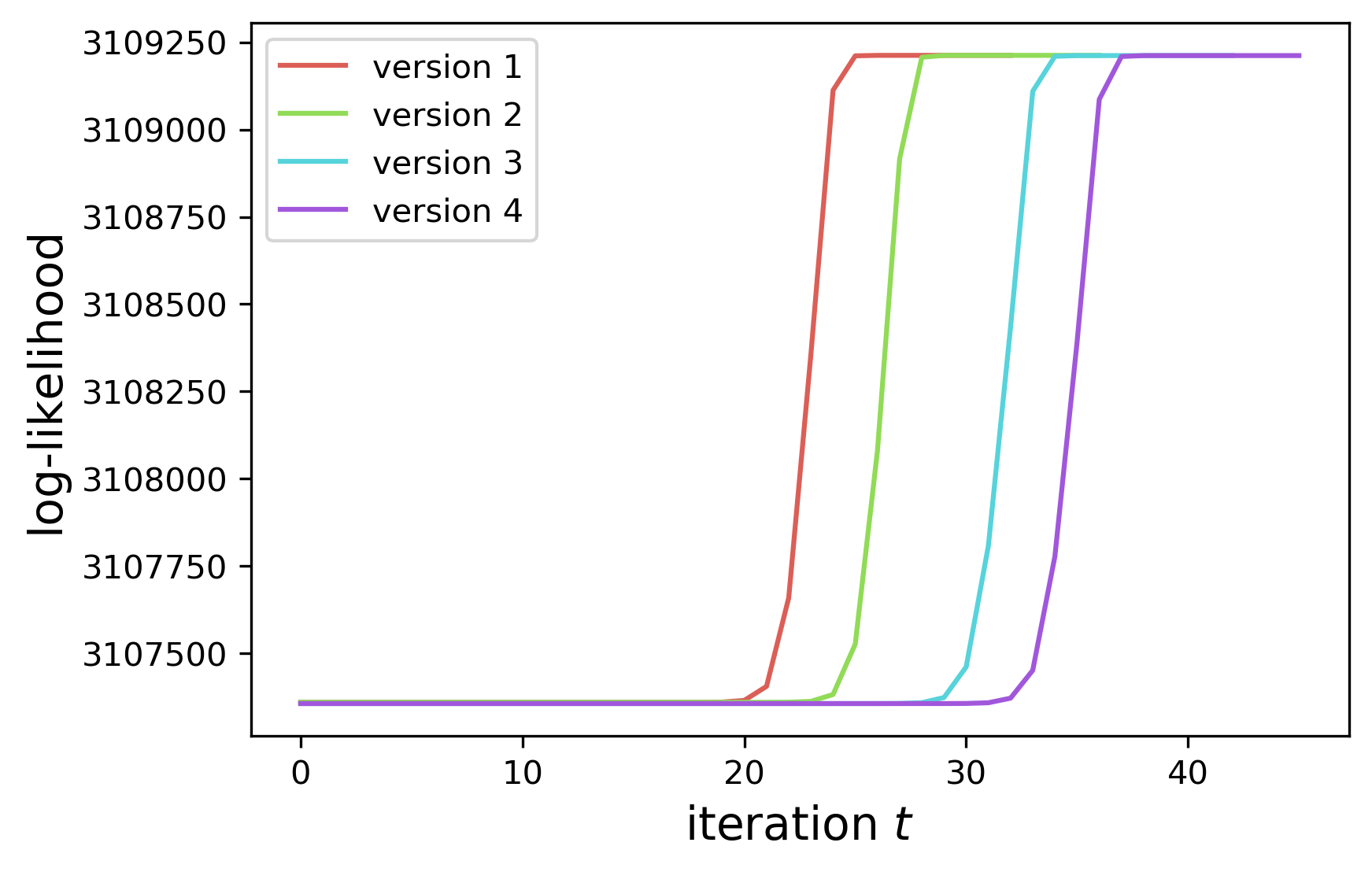}
\includegraphics[width=0.48\textwidth]{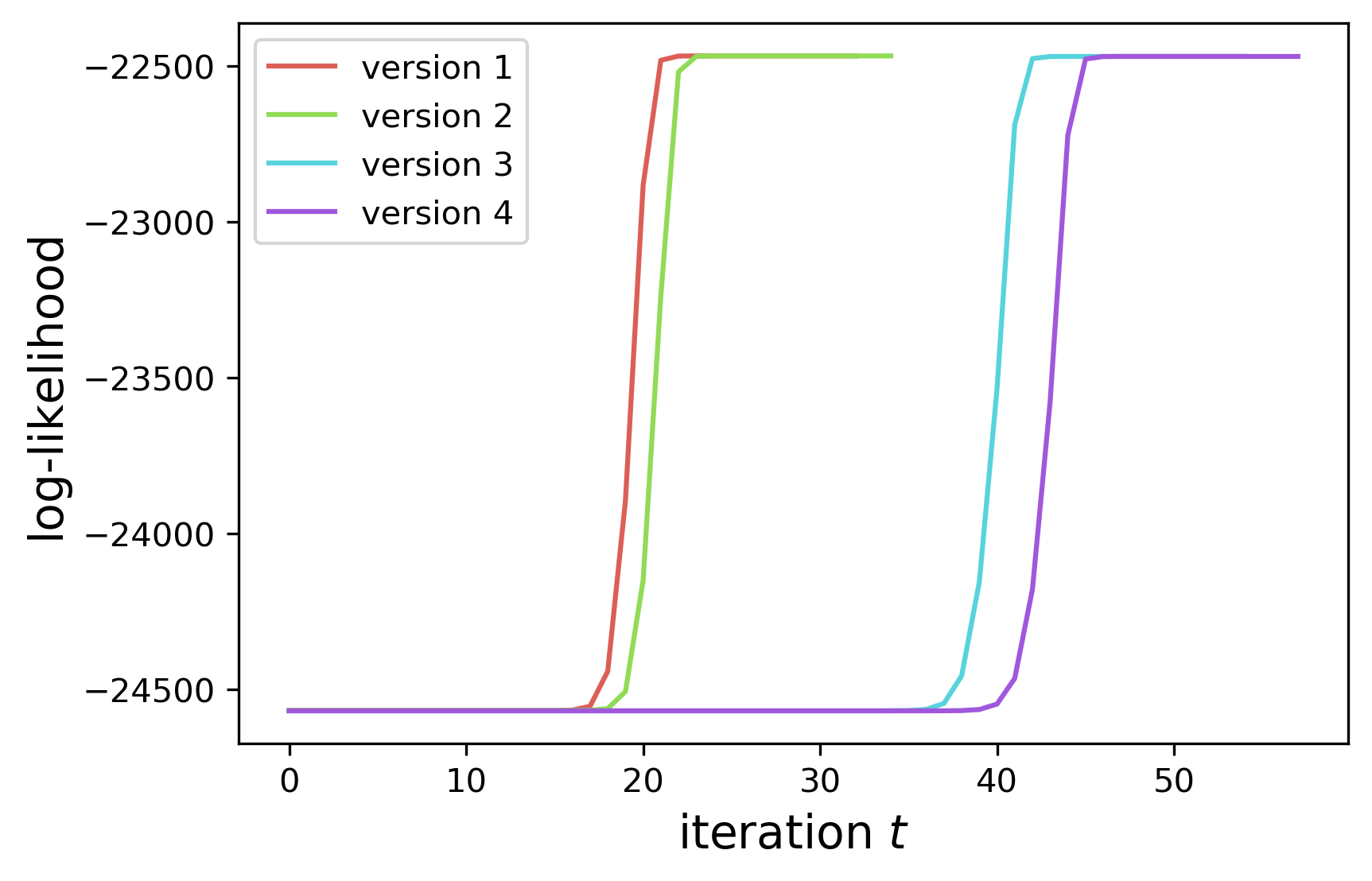}
\centering
\caption{Evolution of the log-likelihood of the models for the different versions. On the left, the results in the Gaussian case and on the right the results for a mixture of $t$-distributions.}
\label{conv.like}
\end{figure}

\textcolor{black}{
Let us now discuss initialization and thresholds used in the proposed algorithm. The mean parameters are initialized as the means resulting of the k-means algorithm. In the case where
k-means outputs clusters of only one point, we run k-means again leaving out the isolated points. Due to singularity problems, we take the initial scatter matrix as the identity matrix. We set the initial value of all $\tau$ parameters to one. For the convergence flag, we consider $10^{-6}$ for the threshold of the $l_2$-norm difference of consecutive estimators, and the maximum number of iterations of the fixed-point loop length is set to 20 based on previous discussion. Using the initialization described above, we obtain the same final clustering results for each run. In the low-dimensional case, we truncate the $\tau$ value in order to avoid numerical issues induced by points that are very close to the mean. That is, if $\tau$ is smaller than $10^{-12}$ we change its value to the selected threshold. The implementation in Python of the algorithm is available at the repository  \href{https://github.com/violetr/frem}{github.com/violetr/fem}}.\\

Furthermore, it is important to remark that, in our approach, the constraint on the trace of $\widehat{\Sig}$ ($\text{tr}(\widehat{\Sig})=m$) \textbf{does not} act as a regularization procedure, as it is usually the case in EM-like algorithms \citep{Tclust, RIMLE}. As mentioned in Remark \ref{rem.pij}, the trace constraint does not affect the clustering results. \\

\textcolor{black}{
Finally, regarding the complexity of the algorithm, it happens to be the same as the one of the classical EM algorithm for mixture of Gaussian distributions. The E-step has the same complexity of the usual algorithm. For the M-step, even though a nested loop is included to solve the fixed-point equations, the complexity is not increased since the number of iterations is constant and the main cost of each iteration corresponds to the scatter matrix inversion as in EM for GMM.}\\

\medskip

\section{Experimental Results}\label{experiments}

In this section, we present experiments obtained with both synthetic and real data. We study the convergence of the fixed point equations and the estimation error in the case of synthetic data (for which we know the true parameter values). We compare our results to the ones of the classical EM for GMM, EM for multivariate $t$-distributions, TCLUST \citep{Tclust} and RIMLE \citep{RIMLE}. Additionally, for the real data, we compare the clustering results with the ground truth labels for k-means, HDBSCAN and spectral clustering \citep{spectral}. The comparison between the former three and our algorithm is straight-forward because they all have in common only one main parameter (the number of clusters) that we fix and suppose known in our experiments. Regarding the implementations, we use Scikit-learn \citep{scikit-learn} for k-means and the Gaussian Mixture and the R package EMMIXskew \citep{emmix} for the mixture of $t-$distributions. Concerning TCLUST and the RIMLE  algorithms, we set the number of clusters and use the default values for the rest of the parameters. When possible, we avoided the artificial constraint on the TCLUST algorithm solution caused by the eigenvalue constraint threshold. We used the OTRIMLE version of RIMLE that selects the main parameter of the model with a data-driven approach \citep{otrimle}. For both of them, we use the R implementation provided by the authors. In the case of spectral clustering, we run the Scikit-learn implementation where it is necessary to tune an extra parameter in order to build the neighborhood graph. We set the number of neighbors in the graph equal to the number that maximizes the silhouette score \citep{silhouettes}. A fair comparison with HDBSCAN is even more difficult to set because the parameters to tune are completely different and less intuitive than those of the other algorithms. Once again, we select the best silhouette score pair of parameters by sweeping a grid of selected values. \\

We then quantify the differences of performance by using the usual metrics for the clustering task known as the adjusted mutual information (AMI) index and the adjusted rand (AR) index \citep{Vinh:2010:ITM:1756006.1953024}. For real datasets, one also provides the rate of correct classification when matching each clustering label with a classification label. In the case of real datasets, we also report the clustering classification rate as done in \cite{doi:10.1002/cyto.a.23030}. In some cases, we visualize the 2D embedding of the data obtained by the UMAP algorithm \citep{UMAP} colored with the resulting labels of the different clustering algorithms in order to better understand the nature of the data and the clustering results. This dimensional reduction algorithm has the same objective as t-SNE \citep{tSNE} but its implementation in Python is much faster.

\subsection{Synthetic Data}\label{sec.syn}

In order to compare the clustering performance of the different algorithms, data are simulated according to different distributions, different values for the $\tau_{ik}$'s and different parameters. The different setups are reported in Tables \ref{tab-param-1} and \ref{tab-param-2}. 
The setups 1 and 2 are mixtures of multivariate t-distributions.  Setup 3 is a mixture of $k-$distributions, $t-$distributions and Gaussian distributions. 
On the other hand, in Setup 4 we add  uniform noise background to three Gaussian distributions. This noise accounts for 10\% of the data.  
Finally, Setup 5 includes clusters that are a mixture of two distributions. In other words, all points from a given cluster are generated with the same parameters $\mub$ and $\Sig$ but we used different distributions. In this situation, we mixed generalized Gaussian distributions (noted $\mathcal{GN}$), $t-$distributions and Gaussian distributions (noted $\mathcal{N}$). 
In Table \ref{tab-param-2}, diag and diag$^\dagger$ are diagonal matrices with trace $m$ and diag$^*$ has trace $12$. Consequently, Setup 1 tests the performance of the algorithm in the presence of covariance matrices with different traces.\\



\begin{table}[ht]

\centering
\begin{tabular}{cccccc}
\hline
\textbf{Setup} & \textbf{$m$} & \textbf{$n$} & \textbf{distribution 1} & \textbf{distribution 2} & \textbf{distribution 3}        \\ \hline
\textbf{1}  & 8  & 1000 & $t$, $dof=3$       & $t$, $dof=3$       & $t$, $dof=3$              \\
\textbf{2}  & 8 & 1000  & $t$, $dof=10$      & $t$, $dof=10$      & $t$, $dof=10$   \\
\textbf{3}  &  40 & 1300 &     $K$, $dof=3$              &    $t$, $dof=6$    & $\mathcal{N}$                            \\ 
\textbf{4}  &  8 & 1200 &     $\mathcal{N}$              &    $\mathcal{N}$    & $\mathcal{N}$\\
\textbf{5}  &  6 & 1200 &     $0.7 \mathcal{N} + 0.3 \mathcal{GN}, s = 0.1$  &    $0.6 \mathcal{N} + 0.4 t, dof = 2.3$   & $\mathcal{N}$\\
\hline
\end{tabular}
\caption{\label{tab-param-1} Dimension and shape of the distributions in the five different setups. The distribution of each of the three clusters in each setup is specified. The distribution can be multivariate Gaussian ($\mathcal{N}$), Generalized Gaussian ($\mathcal{GN}$),  $t-$distribution or $k-$distribution. In the case of the latter three distributions the extra parameters ($dof$ or $s$) are indicated.}

\end{table}

\begin{table}[ht]

\centering
\begin{tabular}{ccccccc}
\hline
\textbf{Setup} & \textbf{$\mub_1$}     & \textbf{$\mub_2$} & \textbf{$\mub_3$}     & \textbf{$\Sig_1$} & \textbf{$\Sig_2$} & \textbf{$\Sig_3$} \\ \hline
\textbf{1}     & $\mathcal{U}_{(0,1)}$ & $6 * 1_m$           & $1.5 * 1_m+3 e_1$       & diag              & diag$^*$              & $\Iden_m/m*4$             \\
\textbf{2}     & $\mathcal{U}_{(0,1)}$ & $5 * 1_m$           & $1.5 * 1_m+\mathcal{N}(0,\varepsilon$) & diag              & diag$^\dagger$              & $\Iden_m$             \\
\textbf{3}     & $2 * 1_m$ & $6 * 1_m$           & $7 * 1_m$ & $toep(\rho=0.2)$              &  $\Iden_m$              & $toep(\rho=0.5)$  \\ 
\textbf{4}     & $5 * 1_m$ & $7 * 1_m$           & $9 * 1_m$ & $toep(\rho=0.2)$              &  $\Iden_m$              & $toep(\rho=0.5)$  \\  
\textbf{5}     & $\mathcal{U}_{(0,0.2)}$ & $2 * 1_m$           & $4 * 1_m+ 2 e_1$ & $toep(\rho=0.4)$              &         $\Iden_m$ & $toep(\rho=0.7)$ \\
\hline
\end{tabular}
\caption{\label{tab-param-2} Parameters of the distributions in the  different setups. The means are either deterministic vectors or stochastic random uniform or Gaussian vectors. The options for the scatter matrix are diagonal with different eigenvalues, the identity matrix or Toeplitz matrix where we specify the constant.}
\end{table}

\textcolor{black}{We repeat each experiment $nrep=200$ times and collect the mean and standard deviation of estimation errors. For the matrices, we compute the Frobenius norm of the difference between the real scatter matrix parameter and its estimation, divided by the matrix size. In order to make a fair comparison of the estimation performance, we take into account the estimations of $\Sig$'s up to a constant. In other words, we normalize all the $\Sig$ estimators to have the correct trace. The reported estimation error is computed as follows:}
$$\sqrt{\sum_{l=1}^m\sum_{o=1}^m \left((\Sig_k)_{lo} - \left(\frac{\widehat{\Sig}_k \mathrm{tr}(\Sig_k)}{\mathrm{tr}(\widehat{\Sig}_k)}\right)_{lo}\right)^2/ m^2}.$$
When estimating $\mub$, the $l_2$ norm of the error is computed.  
The $\mathbf{\pi}_k$ vectors, corresponding to the distribution proportions, are randomly chosen from a set of possibilities that avoid trivial and giant clusters. These cases are avoided due to the ill posed clustering problem that it implies.\\

In these experiments, we include all the considered algorithms that estimate parameters. Thus, we leave out of the comparison k-means, spectral clustering and HDBSCAN. Table \ref{tab.error.1}  shows the estimation error when estimating the main parameters of the model for all the setups. Furthermore, we report the clustering metrics in Table \ref{tab.metrics.1}.  Complementarily, figures \ref{fig.boxplots} and \ref{fig.boxplots2} visually summarize with boxplots the distribution of these measures. In most cases,  the EM for GMM (GMM-EM)  method has poor results and a high variance. \\ 

In setups 1 and 2, the distributions are multivariate $t-$Student and the difference between them is only in the degrees of freedom. In these setups, the proposed algorithm referred to as flexible EM algorithm (F-EM) and $t$-EM error values are smaller than GMM-EM values. This increase in the predictive performance can be simply explained by the robustness of the estimators in the case of heavy-tailed distributions or in the presence of outliers. It is interesting to confirm that, as in Setup 2, the considered distributions have larger degrees of freedom (tails are lighter), GMM-EM performs much better than in Setup 1. However, while TCLUST and RIMLE perform similarly in the Setup 1, RIMLE has a very big variance and worse estimation in the Setup 2. This phenomenon is due to the overestimation of point as noise/outliers. On the other hand, F-EM and $t$-EM perform very similarly in both settings, with a slight improvement of F-EM in the $\Sig$ estimation. 
As shown in both tables, even in the $t-$distributed case where the $t$-EM algorithm is completely adapted, our robust algorithm performs similarly in average. We remark that F-EM performs very good  as expected, even if in Setup 1 the traces are very different. Then, for Setups 3, in the case of mixture of three different distributions ($k$-distribution, $t$-distribution and Gaussian distribution), the F-EM algorithm outperforms the other algorithms in the majority of runs. As Figure \ref{fig.boxplots} shows, there are only  very few runs where F-EM had bad performance. Thus, it is important to notice that the model assumptions used to derive the F-EM algorithm, \textit{i.e.,} unknown $\tau_{ik}$'s and different distributions for each observation, is very general and it allows to successfully handle the case of mixtures of different distributions without additive computational cost, which appears to be an important contribution of this work.\\

Furthermore, Figure \ref{fig.boxplots2} shows the performance in Setup 4 and Setup 5. In Setup 4,  in which uniform background noise in the cube $[0, 14]^m$ is included, the best performances are the ones from TCLUST and RIMLE which appears reasonably since their design matches the data generation process. After them, F-EM has a very good performance taking into account that we do not reject outliers and as a consequence, that those are intrinsically misclassified. When we exclude the noise for the metric computation, the classification performance is equally good for these three algorithms, althougth the TCLUST algorithm is computed with the true proportion of outliers.  Besides, the parameter estimation is equally good for F-EM, RIMLE and TCLUST.  The performance analysis in this Setup (for which RIMLE and TCLUST are designed to provide the best performance) highlights the flexibility and the robustness of the proposed algorithm. Finally, Setup 5 displays a very good behaviour of F-EM and RIMLE compared to the rest of the algorithms. The performance of EM-GMM is really bad there because it cannot deal with outliers coming from heavy tails. The combination of two distributions for one cluster is difficult to fit for $t-$EM and TCLUST.  The model is too general for $t-$EM and TCLUST probably suffers from a noise rate that is not sufficient to avoid the heavy tails.\\ 

To conclude, the proposed algorithm shows by design very stable performance among a wide range of cases. Indeed, when the data perfectly follows a specific model such as \textit{e.g.,} a mixture of $t$-distributions, the best algorithm will be the ML-based one (in this case the $t$-EM algorithm). However, the F-EM algorithm does perform almost as good as the $t$-EM. But in various other scenarios (data drawn from different models, outliers in the data), the F-EM will clearly outperform traditional model-based algorithms that are not adaptive.\\

\begin{figure}
    \centering
    
    Setup 1
    
    \includegraphics[width=0.28\textwidth]{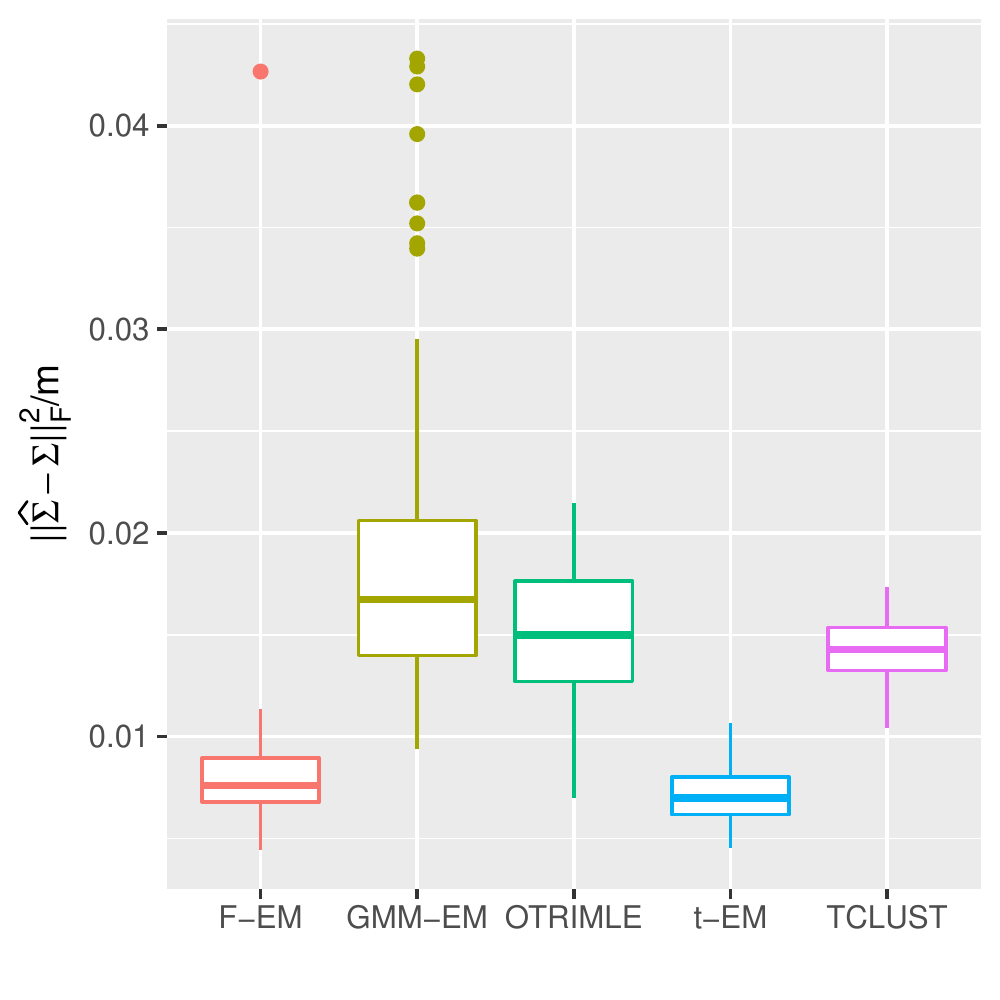}
    \includegraphics[width=0.28\textwidth]{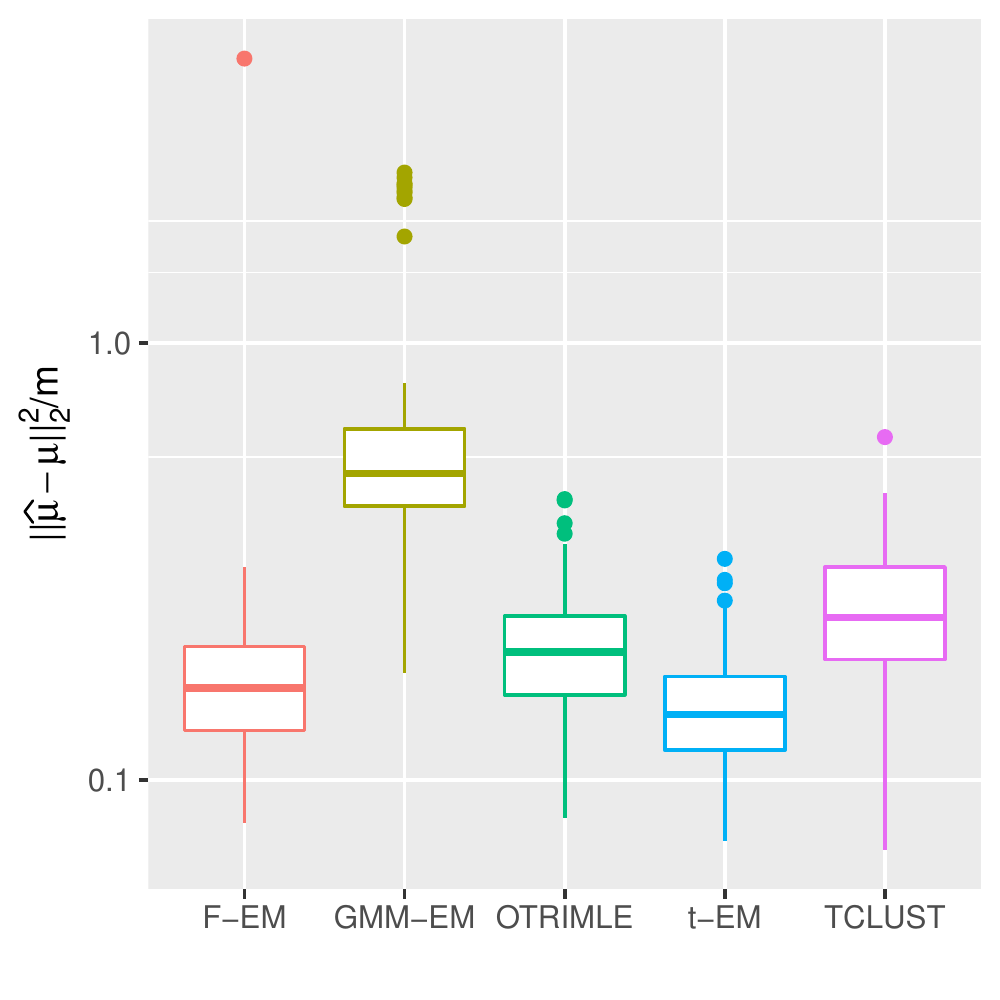}
    \includegraphics[width=0.28\textwidth]{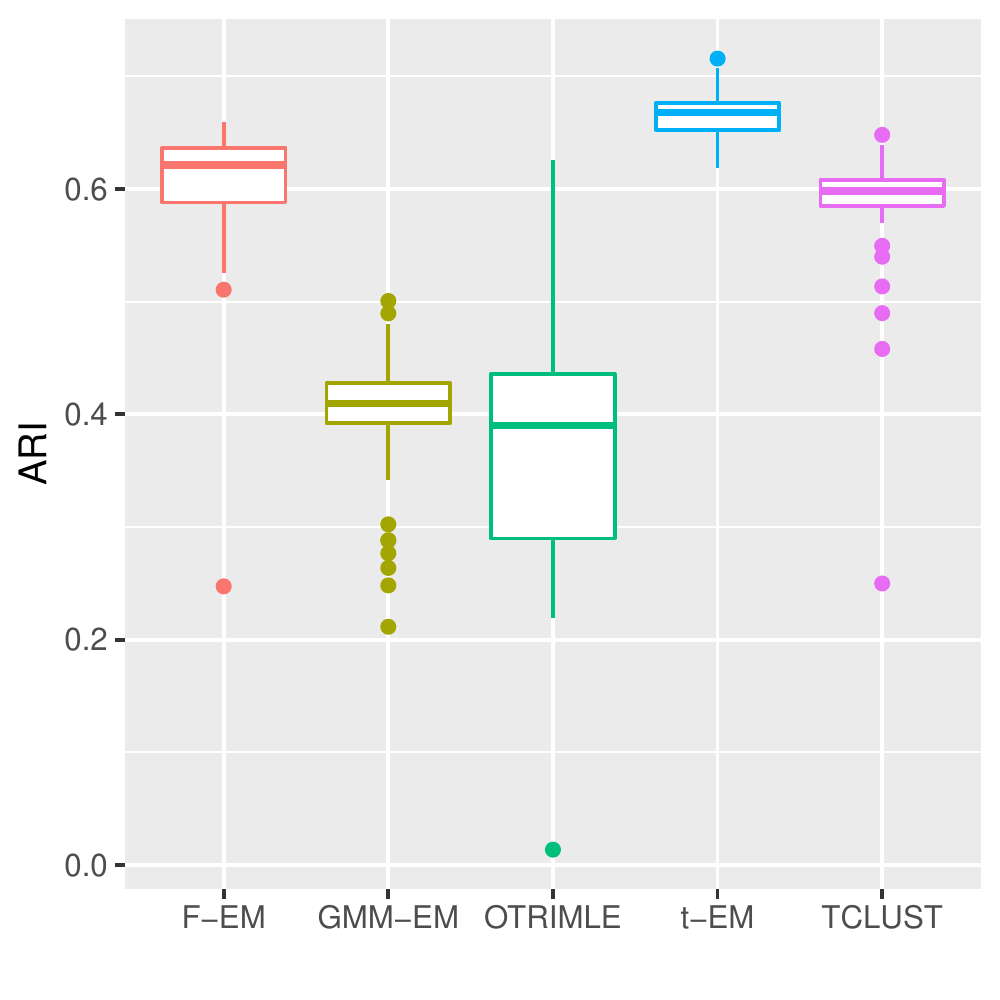}
    
    Setup 2
    
     \includegraphics[width=0.28\textwidth]{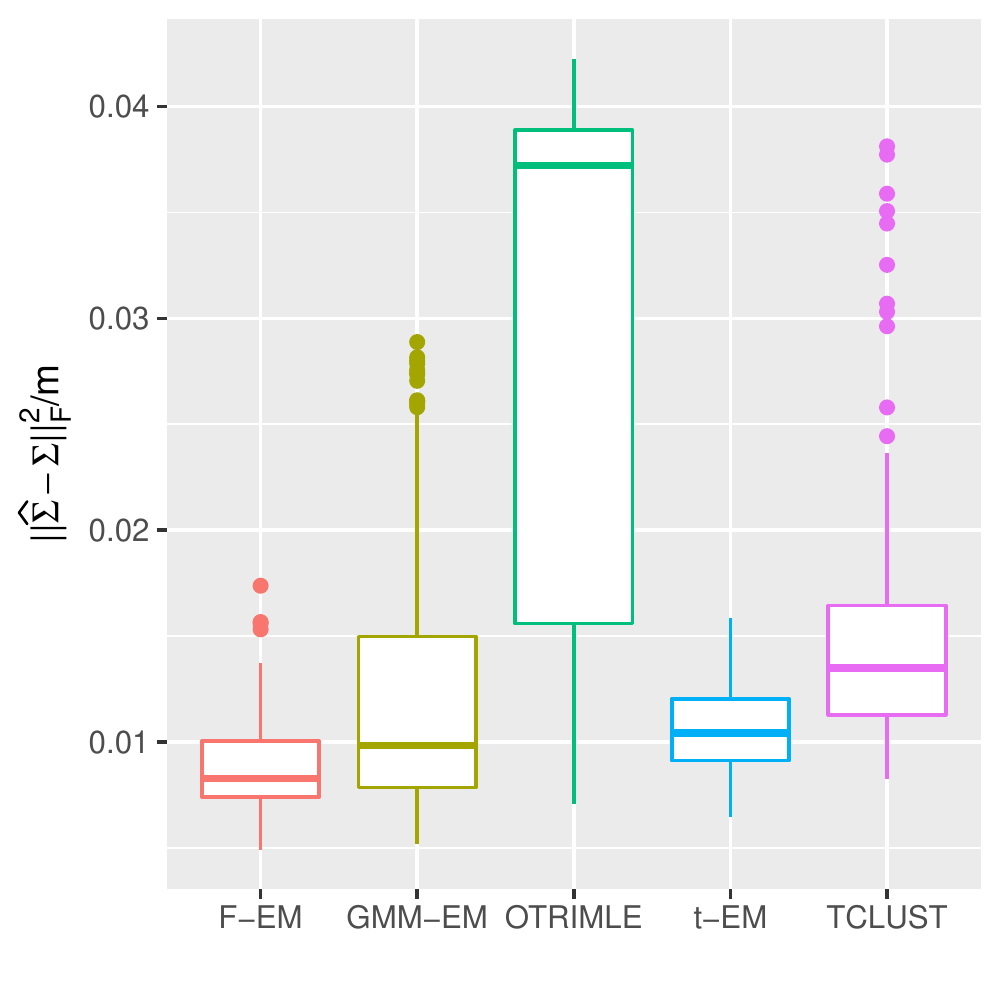}
    \includegraphics[width=0.28\textwidth]{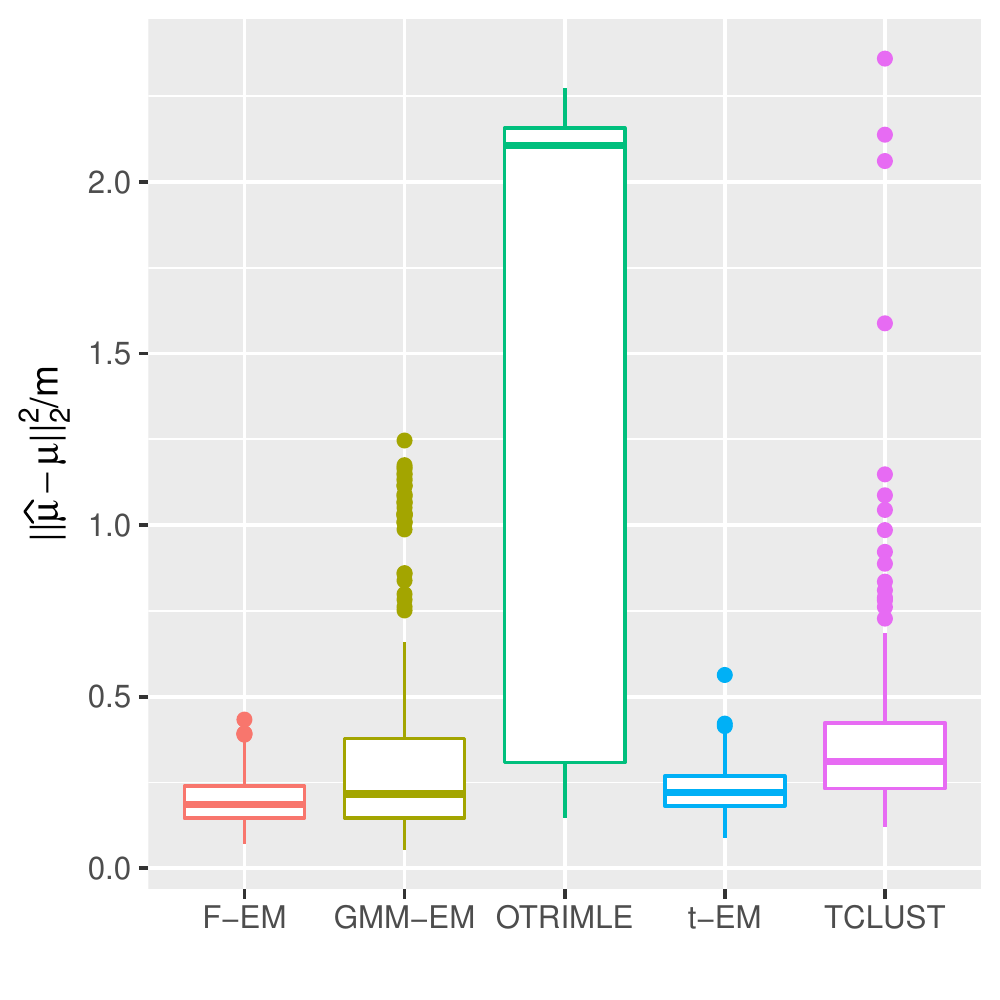}
    \includegraphics[width=0.28\textwidth]{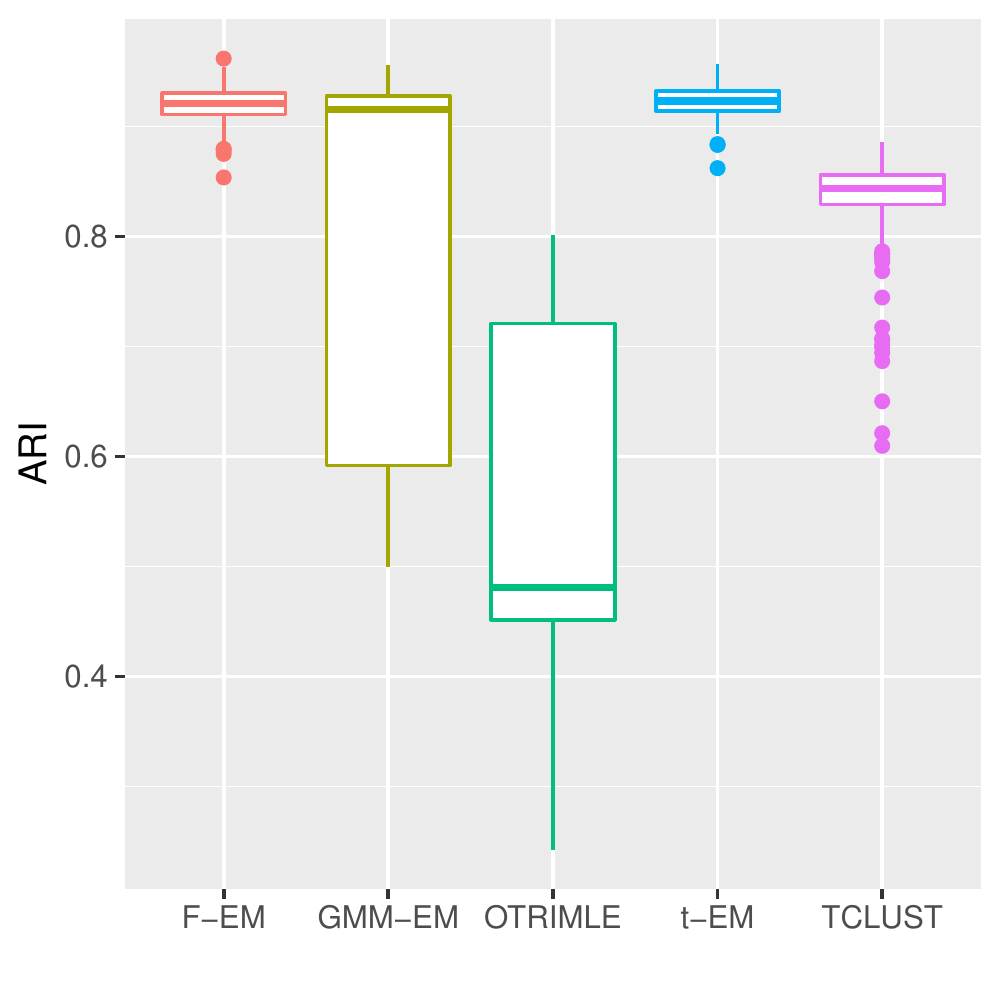}
    
    Setup 3
    
    \includegraphics[width=0.28\textwidth]{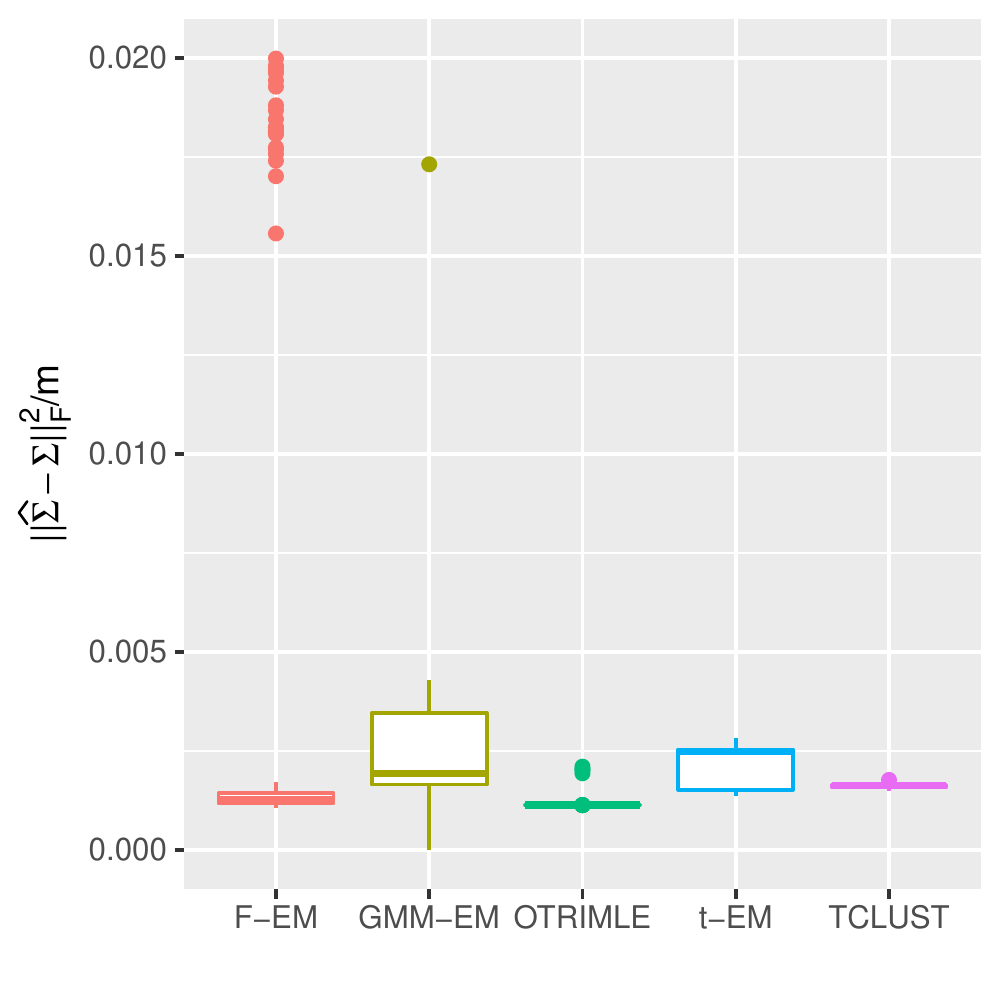}
    \includegraphics[width=0.28\textwidth]{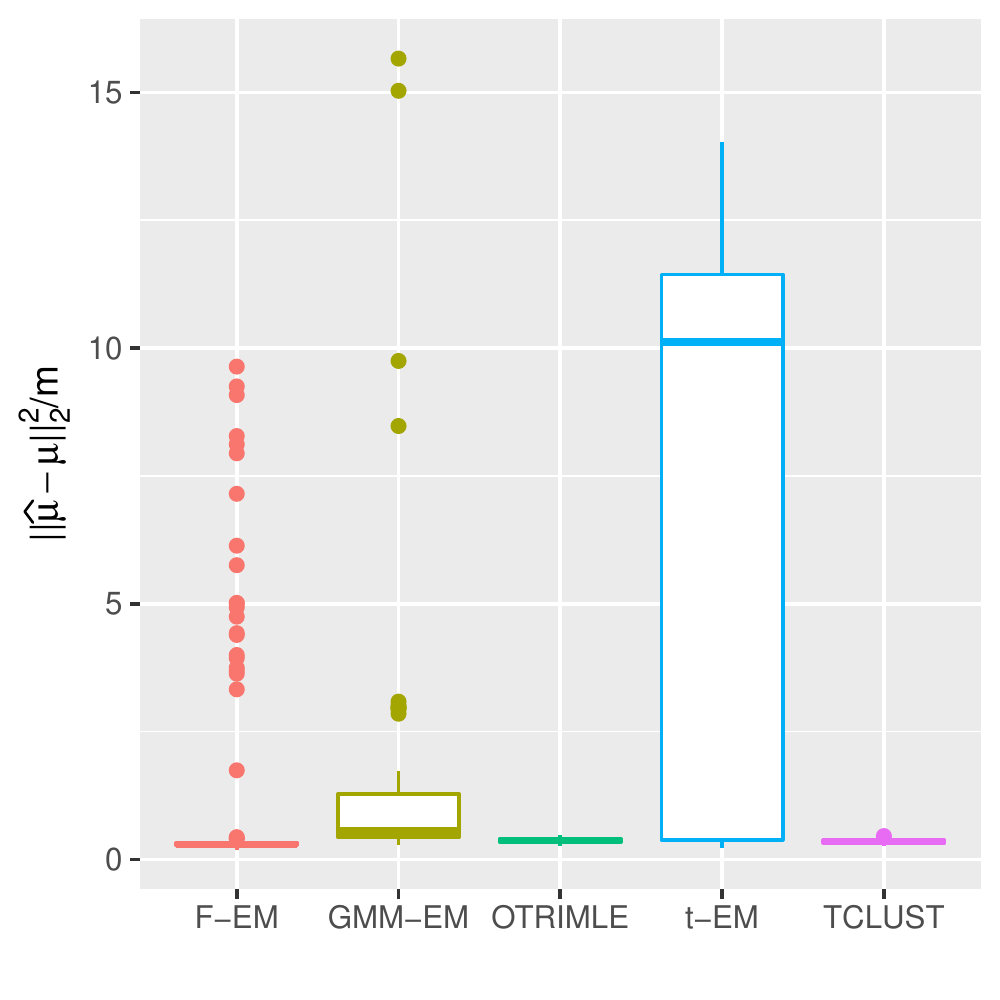}
    \includegraphics[width=0.28\textwidth]{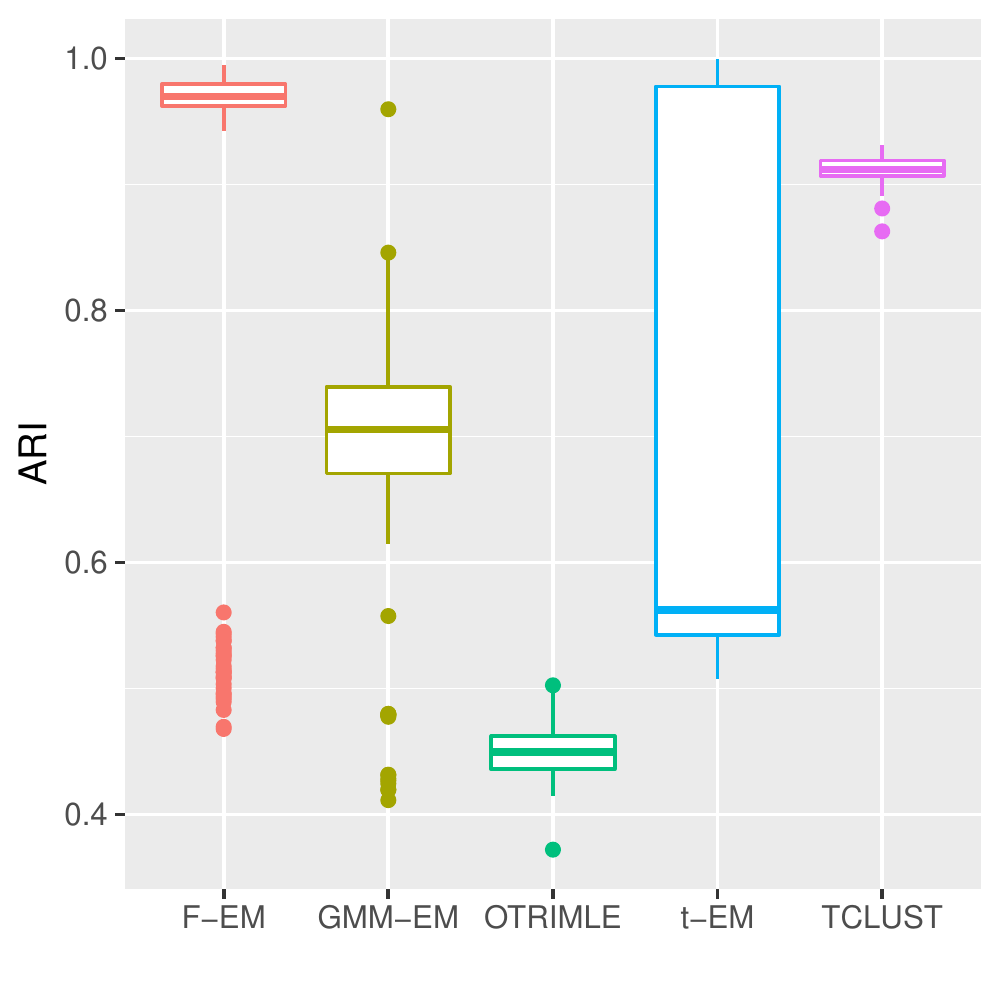}
    
    \caption{Boxplots representing the performance of the algorithms in the estimation and classification for the different setups. Each row represents one setup. From the left to the right, the Figure summarizes the estimation error of the scatter matrix up to a constant, the estimation error of the mean and the AR index of the classification when comparing to the ground truth.}
    \label{fig.boxplots}
\end{figure}

\begin{figure}
    \centering
    
    Setup 4
    
    \includegraphics[width=0.28\textwidth]{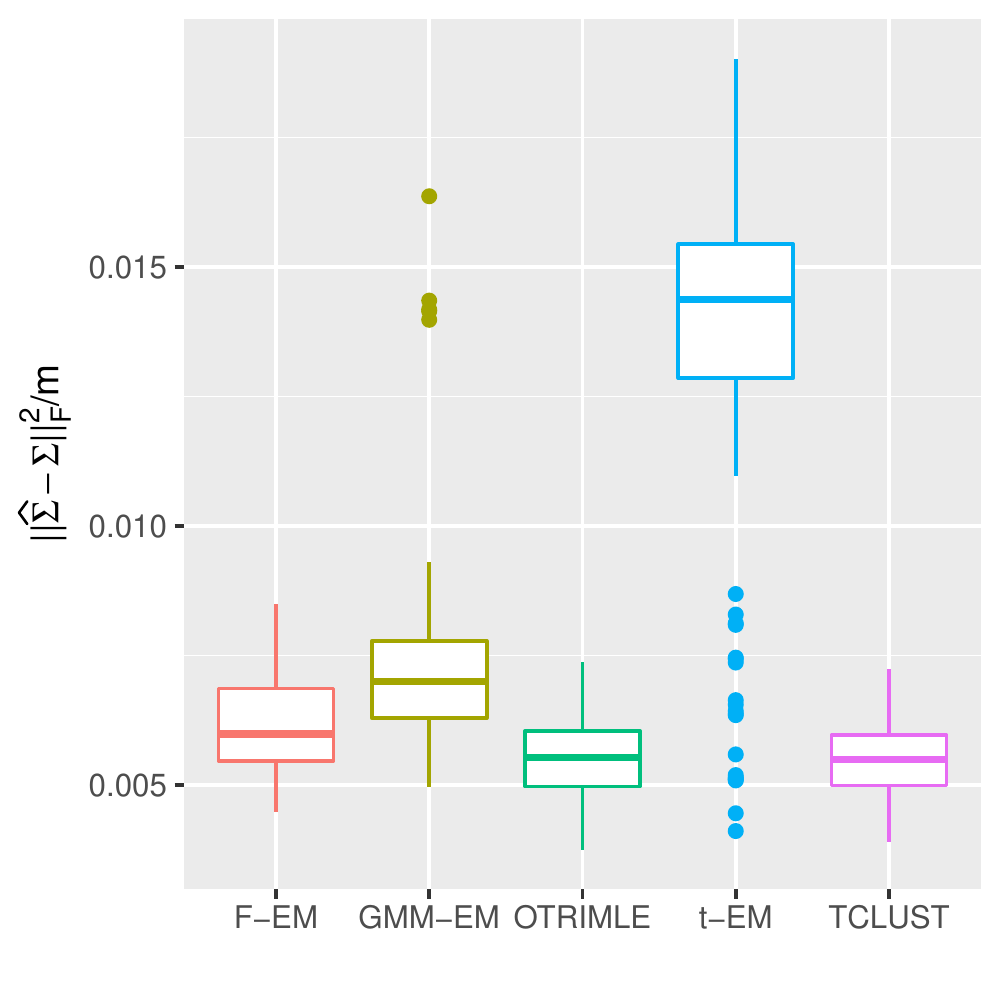}
    \includegraphics[width=0.28\textwidth]{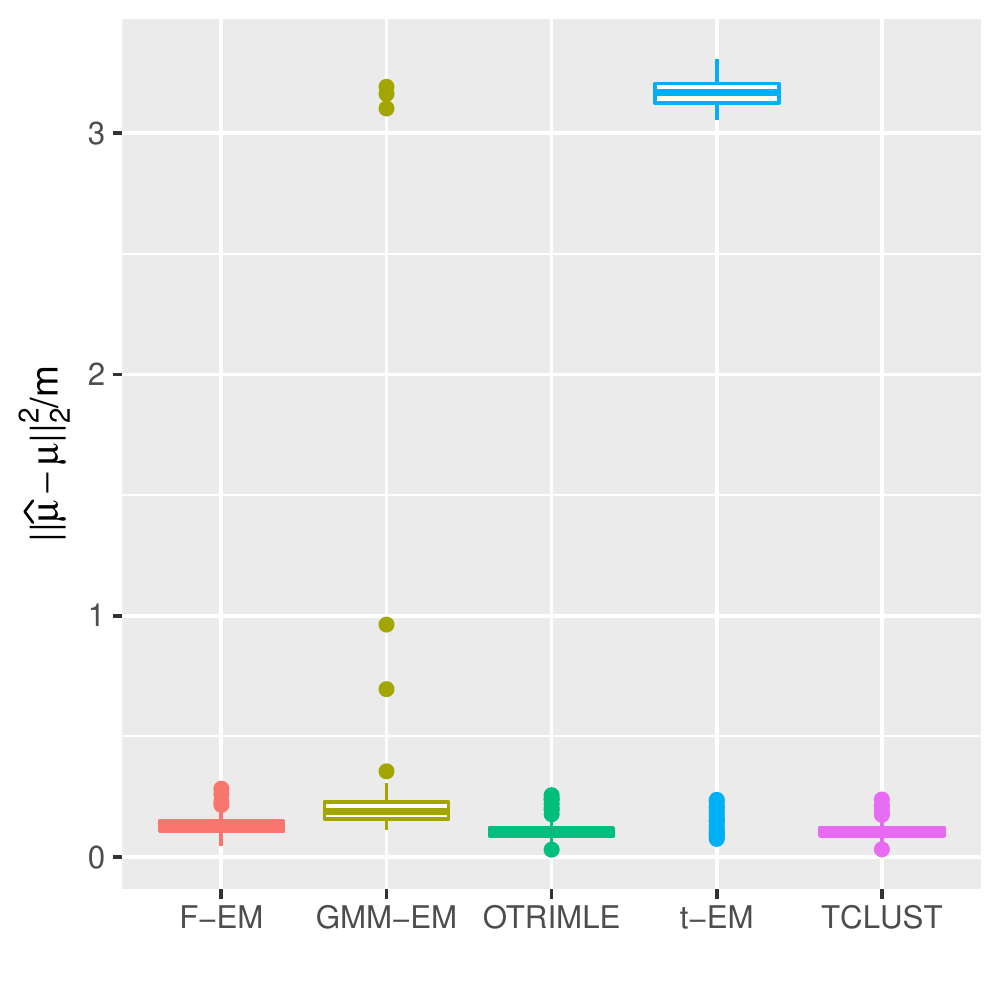}
    \includegraphics[width=0.28\textwidth]{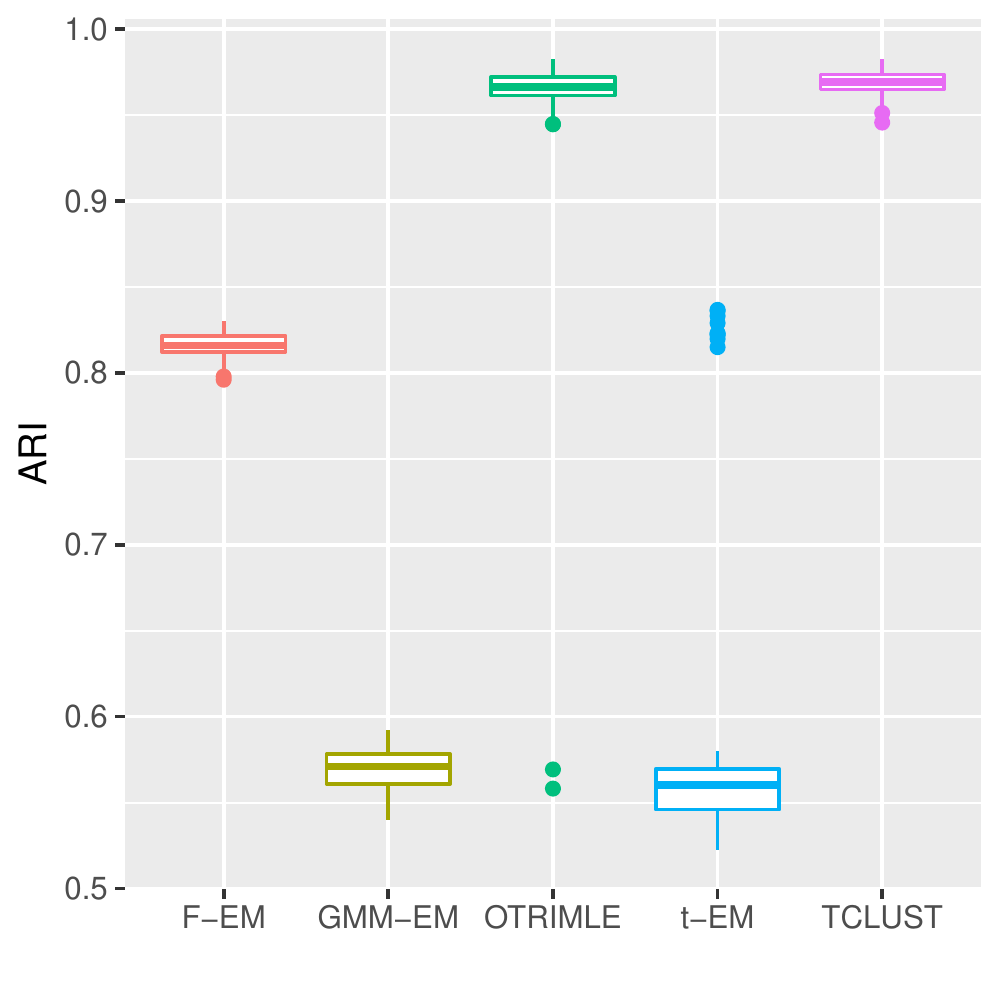}
    
    Setup 5
    
    \includegraphics[width=0.28\textwidth]{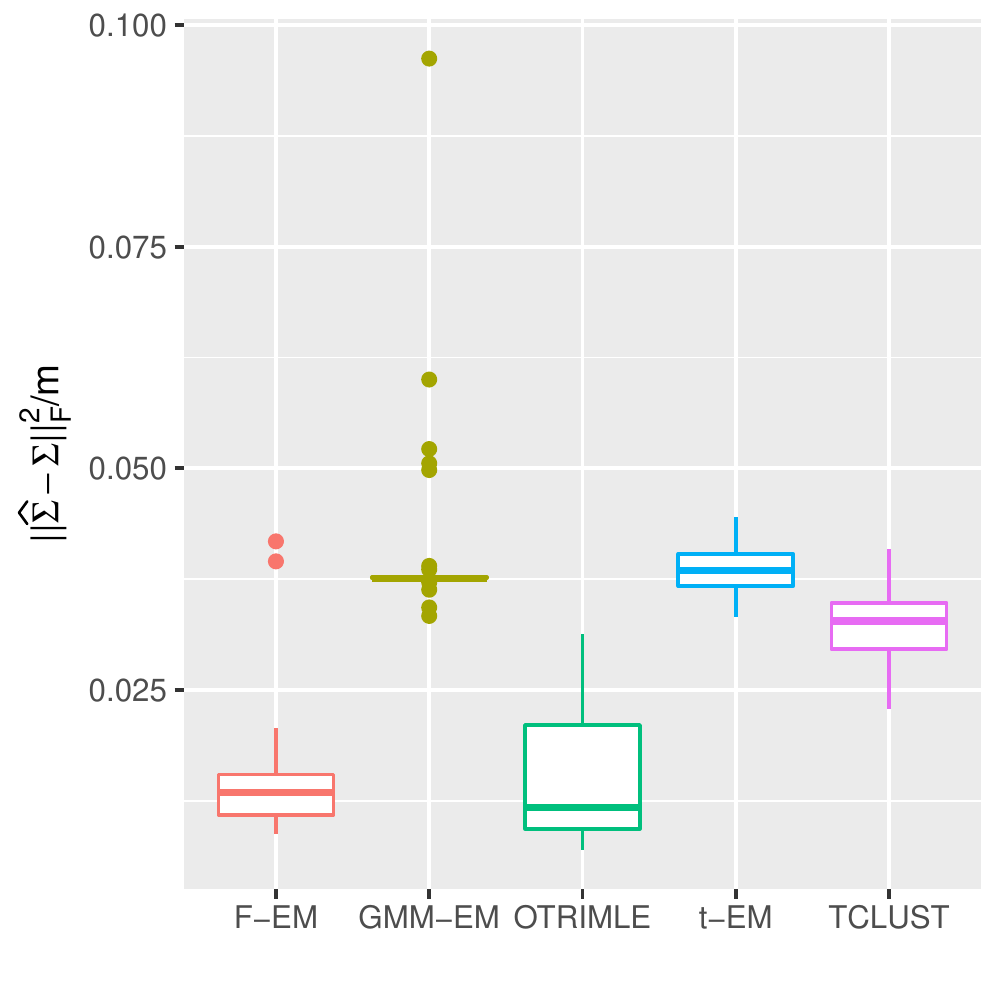}
    \includegraphics[width=0.28\textwidth]{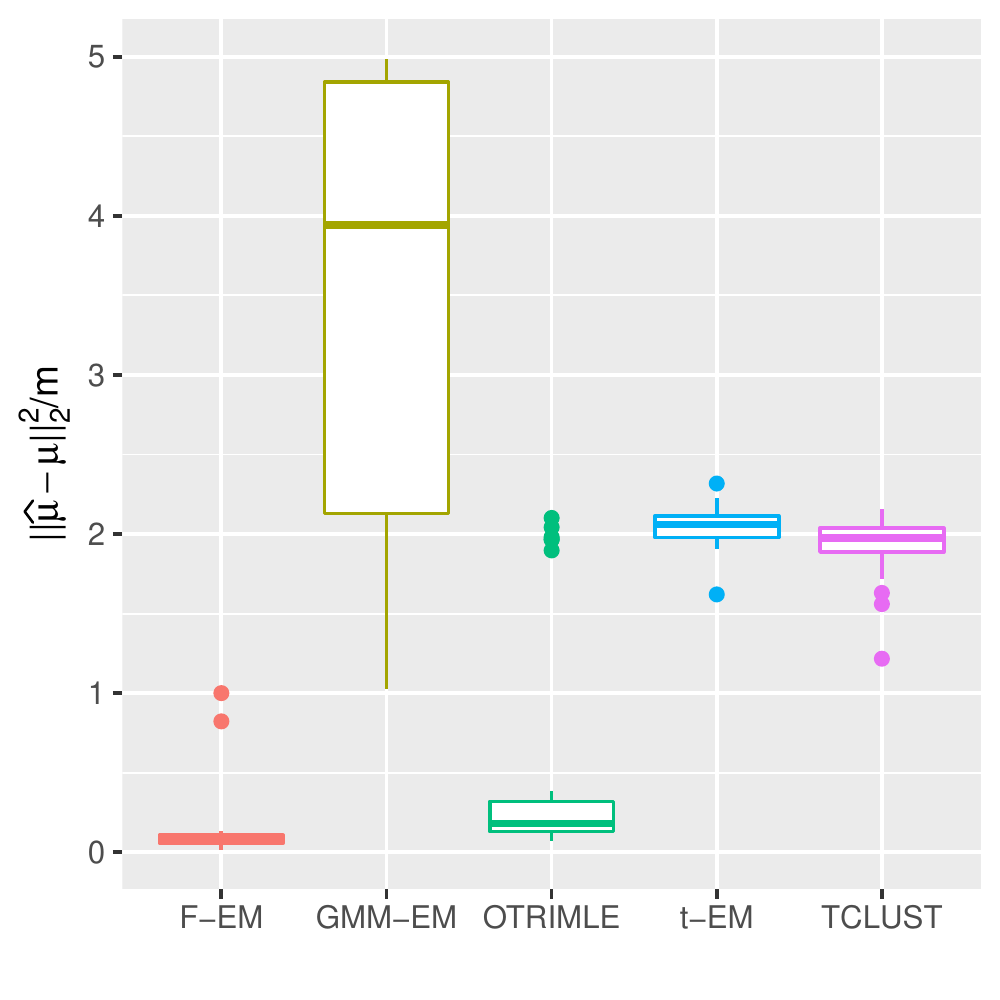}
    \includegraphics[width=0.28\textwidth]{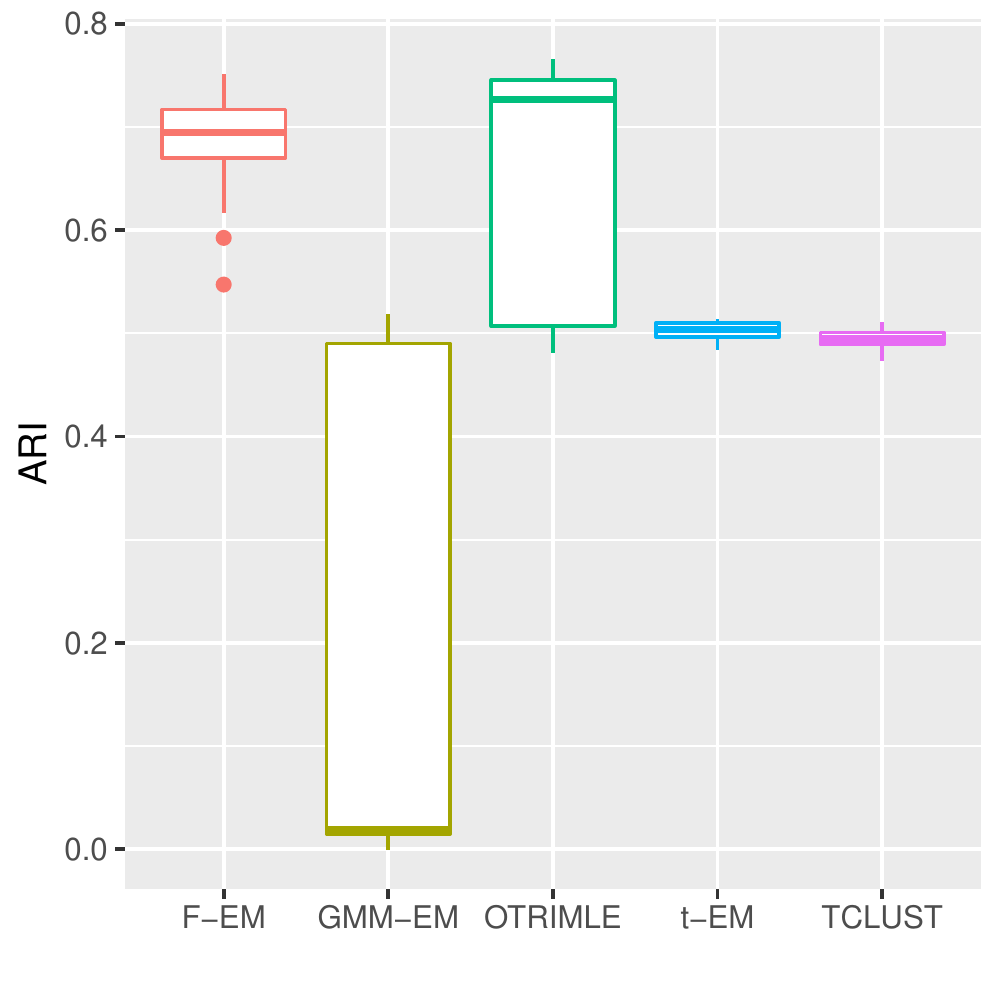}
    
    \caption{Boxplots representing the performance of the algorithms in the estimation and classification for the different setups. Each row represents one setup. From the left to the right, the Figure summarizes the estimation error of the scatter matrix up to a constant, the estimation error of the mean and the AR index of the classification when comparing to the ground truth.}
    \label{fig.boxplots2}
\end{figure}

\begin{table}[ht]
\centering
\begin{tabular}{ccccccc}
  \hline
Setup & Error & GMM-EM & t-EM & F-EM & TCLUST & OTRIMLE \\ 
  \hline
1 & $\Sig_1$ & 0.0094 & 0.0070 & 0.0073 & 0.0102 & 0.0101 \\ 
  1 & $\Sig_2$ & 0.0371 & 0.0132 & 0.0151 & 0.0371 & 0.0260 \\ 
  1 & $\Sig_3$ & 0.0128 & 0.0046 & 0.0046 & 0.0090 & 0.0062 \\ 
  1 & $\mub_1$ & 0.2110 & 0.1628 & 0.1670 & 0.2304 & 0.1808 \\ 
  1 & $\mub_2$ & 1.6277 & 0.2100 & 0.2513 & 1.2186 & 0.3079 \\ 
  1 & $\mub_3$ & 1.2004 & 0.1238 & 0.1401 & 0.8618 & 0.1420 \\

2 & $\Sig_1$ & 0.0098 & 0.0104 & 0.0083 & 0.0135 & 0.0372 \\ 
  2 & $\Sig_2$ & 0.0083 & 0.0068 & 0.0076 & 0.0075 & 0.0099 \\ 
  2 & $\Sig_3$ & 0.0103 & 0.0074 & 0.0094 & 0.0088 & 0.0289 \\ 
  2 & $\mub_1$ & 0.2168 & 0.2200 & 0.1853 & 0.3115 & 2.1054 \\ 
  2 & $\mub_2$ & 0.1879 & 0.1379 & 0.1570 & 0.1405 & 0.2337 \\ 
  2 & $\mub_3$ & 0.2063 & 0.1532 & 0.2077 & 0.1895 & 1.0695 \\

3 & $\Sig_1$ & 0.0019 & 0.0025 & 0.0013 & 0.0016 & 0.0011 \\ 
  3 & $\Sig_2$ & 0.0016 & 0.0034 & 0.0014 & 0.0014 & 0.0000 \\ 
  3 & $\Sig_3$ & 0.0022 & 0.0012 & 0.0012 & 0.0011 & 0.0029 \\ 
  3 & $\mub_1$ & 0.5565 & 10.1226 & 0.2967 & 0.3469 & 0.3714 \\ 
  3 & $\mub_2$ & 0.3655 & 6.1910 & 0.3025 & 0.3374 & 4.6289 \\ 
  3 & $\mub_3$ & 0.6081 & 0.2885 & 0.3060 & 0.2781 & 1.7128 \\ 

4 & $\Sig_1$ & 0.0070 & 0.0144 & 0.0060 & 0.0055 & 0.0055 \\ 
  4 & $\Sig_2$ & 0.0085 & 0.0085 & 0.0064 & 0.0055 & 0.0056 \\ 
  4 & $\Sig_3$ & 0.0254 & 0.0076 & 0.0065 & 0.0057 & 0.0056 \\ 
  4 & $\mub_1$ & 0.1876 & 3.1683 & 0.1249 & 0.1048 & 0.1071 \\ 
  4 & $\mub_2$ & 0.7796 & 0.7783 & 0.1382 & 0.1179 & 0.1167 \\ 
  4 & $\mub_3$ & 3.1758 & 0.2386 & 0.1436 & 0.1080 & 0.1053 \\

5 & $\Sig_1$ & 0.0376 & 0.0384 & 0.0134 & 0.0328 & 0.0117 \\ 
  5 & $\Sig_2$ & 0.0481 & 0.0361 & 0.0191 & 0.0322 & 0.0181 \\ 
  5 & $\Sig_3$ & 0.0000 & 0.0091 & 0.0110 & 0.0093 & 0.0095 \\ 
  5 & $\mub_1$ & 3.9425 & 2.0617 & 0.0751 & 1.9759 & 0.1822 \\ 
  5 & $\mub_2$ & 0.6168 & 1.3857 & 0.3063 & 2.2991 & 0.2996 \\ 
  5 & $\mub_3$ & 5.1633 & 0.1457 & 0.1659 & 0.1434 & 0.1424 \\

   \hline
\end{tabular}
\caption{Average of the norm of the error in the estimation of the main parameters in the different setups.} 
\label{tab.error.1}
\end{table}

\begin{table}[ht]
\centering
\begin{tabular}{ccccccc}
  \hline
Setup & Error & GMM-EM & t-EM & F-EM & TCLUST & OTRIMLE \\ 
  \hline
  
1 & AMI & 0.4491 & 0.7095 & 0.6809 & 0.4036 & 0.4197 \\ 
  1 & ARI & 0.4373 & 0.7895 & 0.7513 & 0.4293 & 0.2851 \\

2 & AMI & 0.8784 & 0.8843 & 0.8836  & 0.7414 & 0.5342 \\ 
2 & ARI & 0.9156 & 0.9233 & 0.9208  & 0.8476 & 0.4809 \\ 
  
3 & AMI & 0.7753 & 0.5514 & 0.9597 & 0.8377 & 0.4936 \\ 
  3 & ARI & 0.7056 & 0.5624 & 0.9722 & 0.9120 & 0.4497 \\ 
  
4 & AMI & 0.7373 & 0.6115 & 0.7836 & 0.9551 & 0.9476 \\ 
  4 & ARI & 0.5709 & 0.5603 & 0.8159 & 0.9690 & 0.9661 \\

5 & AMI & 0.1058 & 0.5479 & 0.6711 & 0.5573 & 0.6426 \\ 
  5 & ARI & 0.0187 & 0.5038 & 0.6946 & 0.4947 & 0.7265 \\ 

   \hline
\end{tabular}
\caption{Average clustering metrics in the different setups.} 
\label{tab.metrics.1}
\end{table}

\subsection{Real Data}

The proposed F-EM algorithm has been tested on three different real data sets: MNIST \citep{MNIST}, small NORB \citep{NORB} and \textit{20newsgroup} \citep{20newsgroup}. The MNIST hand-written digits (Figure \ref{picMNIST}) data set has become a standard benchmark for classification/clustering methods. We apply F-EM to discover groups in balanced subsets of similar pairs of digits (3-8 and 1-7) and the set of digits (3-8-6). We additionally contaminate the later subset with a small proportion of noise by randomly adding some of the remaining different digits.    \\

\begin{figure}[!ht]
    \centering
\subfigure{
\includegraphics[width=1.0in]{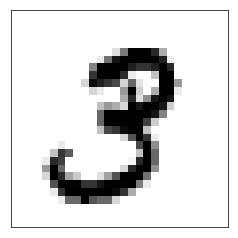}}
\subfigure{
\includegraphics[width=1.0in]{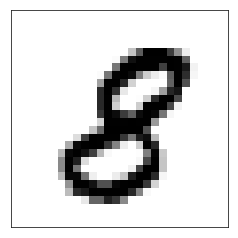}} 
\centering
\caption{Two samples of the pair 3-8 from the hand-written MNIST data set.\vspace*{1cm}}
\label{picMNIST}
\end{figure}

As in many application examples in the literature, we first applied PCA to work with some meaningful features instead of the original data {\citep{tSNE}}. We make a trade-off between explained variance and curse of dimensionality effects. The dimension of the reduced data is shown in Table \ref{table.char} under the column $m$. \textcolor{black}{Because of the stochastic character of the algorithms, we run each of them multiple times ($nrep = 50$) and we report the median value of the metrics. The metrics for the F-EM algorithm are almost always the same and this explains why we do not report the variance.} \\

As can be seen in Tables \ref{ami}, \ref{ar} and \ref{ac}, one obtains, in most cases, better values for all the metrics than those produced by the other partitioning techniques. This can be explained by the increment in flexibility and the smaller impact of outliers in the estimation process. More precisely, the F-EM algorithm does not provide the best results in these scenarios:
\begin{itemize} 
\item MNIST 7-1 scenario for AMI and AR indices, where the $t$-EM performs the best, 
\item MNIST 3-8-6 and its noisy variation for the three criteria where the spectral clustering and TCLUST respectively perform the better,
\end{itemize}
The loss in performance of the F-EM algorithm is, in most cases, around or less than $1\%$ highlighting the robustness of the approach: \textbf{``better or strongly better than existing methods in most cases and comparable in other cases''}. Moreover, those scenarios always correspond to the simpler scenarios, without noise and with well-separated clusters or completely designed to be managed by the best algorithm (MNIST 3-6-8 plus noise for the TCLUST).\\

\begin{table}[!ht]
\centering
\begin{tabular}{ccccc}
\hline
\textbf{Set} & \textbf{Set name} & \textbf{$m$} & \textbf{$n$} & \textbf{$k$} \\ \hline
\textbf{1} & \textbf{MNIST 3-8}   &  30 & 1600 & 2 \\
\textbf{2} & \textbf{MNIST 7-1}    &   30 & 1600 & 2 \\
\textbf{3} & \textbf{MNIST 3-8-6}    &  30 & 1800 & 3\\
\textbf{4} & \textbf{MNIST 3-8-6 + noise} &    30 & 2080 & 3  \\
\textbf{5} & \textbf{NORB}    &   30 & 1600 & 4 \\
\textbf{6} & \textbf{20newsgroup}    &    100 & 2000 & 4 \\
\hline 
\end{tabular}
\caption{Characteristics of the subsets of the data sets that have been used to compare the algorithms. The data sets are variations of the MNIST data set, small NORB and \textit{20newsgroup}.}
\label{table.char}
\end{table}

\begin{table}[ht]
\centering
\begin{tabular}{cccccccc}
\hline
\textbf{Set}  & \textbf{k-means} & \textbf{GMM} & \textbf{$t$-EM} & \textbf{F-EM} & \textbf{spectral} & \textbf{TCLUST} & \textbf{RIMLE}\\ \hline
\textbf{1}    &     0.2203          &       0.4878       &        0.5520     &   \textbf{0.5949} &            0.5839 & 0.5666 & 0.3875\\
\textbf{2}    &      0.7839            &     0.8414          &     \textbf{0.8947} &   0.8811 &    0.8852  & 0.5705 &  0.3875 \\
\textbf{3} &   0.6149               &     0.7159         &     0.7847        &   0.7918  &   \textbf{0.8272} & 0.7818 &  0.6077  \\
\textbf{4}      & 0.3622               &      0.4418         &     0.4596        &   0.4664 &    0.3511    & \textbf{0.6047} & 0.3553 \\
\textbf{5}     &   0.0012          &       0.0476       &    0.4370          &   \textbf{0.5321} &    $\sim 0$      & 0.1516 & 0.2312 \\
\textbf{6} & 0.2637          &        0.3526       &      0.4496         &   \textbf{0.4873}   & 0.1665 & 0.2604 &   0.0686 \\
\hline 
\end{tabular}
\caption{Median AMI index measuring the performance of k-means, GMM-EM, $t$-EM, TCLUST, RIMLE, spectral and our algorithm (F-EM) results for variations of the MNIST data set, small NORB and \textit{20newsgroup}.}
\label{ami}
\end{table}

 \begin{table}[ht]
 \centering
 \begin{tabular}{cccccccc}
 \hline
 \textbf{Set}  & \textbf{k-means} & \textbf{GMM} & \textbf{$t$-EM} & \textbf{F-EM} & \textbf{spectral} & \textbf{TCLUST} & \textbf{RIMLE}\\ \hline
 \textbf{1}   &     0.2884          &       0.5716       &        0.6397     &   \textbf{0.6887} &            0.6866 & 0.6847 & 0.2494\\
 \textbf{2}    & 0.8486 & 0.8905 & \textbf{0.9432}  &  0.9360 &    0.9384 & 0.6885 & 0.2493 \\
 \textbf{3}     & 0.6338 & 0.7332 &     0.8262 &  0.8306 &  \textbf{0.8542} & 0.8366 & 0.4274\\
 \textbf{4}   & 0.4475 &      0.4909 & 0.5296 &  0.5548 & 0.3115 & \textbf{0.6908} & 0.1498 \\
 \textbf{5}     &  0.0015 & 0.0468 & 0.4223 &  \textbf{0.5067}  &      $\sim 0$ & 0.1330 & 0.1472 \\
 \textbf{6}     & 0.1883   &      0.2739         &    0.4426         &   \textbf{0.5114} &   0.0987  & 0.2664 &  0.0026 \\
 \hline 
 \end{tabular}
 \caption{Median AR index measuring the performance of k-means, GMM-EM, $t$-EM, TCLUST, RIMLE, spectral and our algorithm (F-EM) results for variations of the MNIST data set, small NORB and \textit{20newsgroup}.}
 \label{ar}
\end{table}

\begin{table}[ht]
\centering
\begin{tabular}{cccccccc}
\hline
\textbf{Set} & \textbf{k-means} & \textbf{GMM} & \textbf{$t$-EM} & \textbf{F-EM} & \textbf{spectral} & \textbf{TCLUST} & \textbf{RIMLE}\\ \hline
\textbf{1}   &     0.7687          &       0.8781       &        0.9093     &   \textbf{0.9150} &            0.9050 & 0.8881 & 0.5193\\
\textbf{2}   & 0.9606 & 0.9718 & 0.9856  &  \textbf{0.9868} &    0.9844 & 0.8893 & 0.5193 \\
\textbf{3} & 0.8495 & 0.8976 &     0.9366  & 0.9390 &  \textbf{0.9476} & 0.9183 & 0.5157\\
\textbf{4}  & 0.8144 &     0.8700 & 0.8894 &  0.8966 & 0.5444 & \textbf{0.9247} & 0.4988\\
\textbf{5}     &  0.2725 & 0.3487 & 0.6528 &  \textbf{0.6975}  &     0.2600 & 0.4087 & 0.3887 \\
\textbf{6}    &    0.5755   &      0.7100         &     0.6900        &   \textbf{0.8030} &   0.5220 & 0.5740 & 0.2970\\
 \hline 
 \end{tabular}
 \caption{\textcolor{black} {Median accuracy measuring the performance (correct classification rate) of k-means, GMM-EM, $t$-EM, TCLUST, RIMLE, spectral and our algorithm (F-EM) results for variations of the MNIST data set, small NORB and \textit{20newsgroup}.}}
 \label{ac}
 \end{table}

We collected the clustering results from the HDBSCAN algorithms fed with a grid of values for its two main parameters. All the computed metrics comparing the results with the ground truth were poor, close to 0. We show the best clustering result of the 3-8 MNIST subset in Figure \ref{MNIST38}, where a high amount of data points is classified as noise by the algorithm. If the metric is computed only in the non-noise labeled data points then the clustering is almost perfect. This behavior might be explained by the dimension of the data, that seems to be too high for HDBSCAN to deal with.\\

Additionally, we have tested dimensional reduction techniques UMAP and t-SNE prior to the clustering task. All metrics were improved after carefully tuning the parameters. In this scenario, the proposed method performs similarly to the classical GMM-EM because these embedding methods tend to attract outliers and noise to clusters. However, these non-linear visualization approaches are not recommended to extract features before clustering because fictitious effects might appear depending on the parameters choice. \\

\begin{figure}[ht]
    \centering
\subfigure{
\includegraphics[width=1.9in]{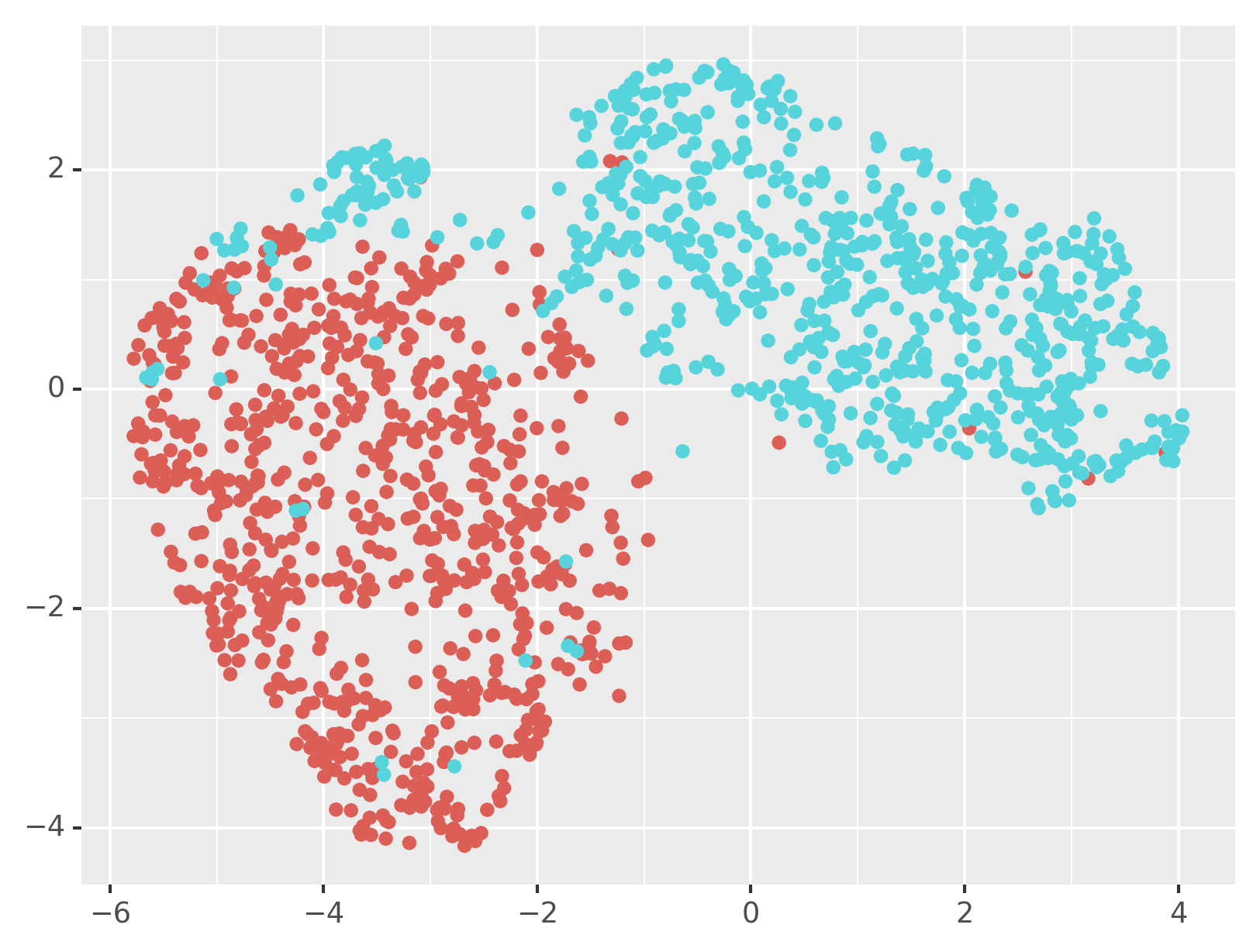}}
\subfigure{
\includegraphics[width=1.9in]{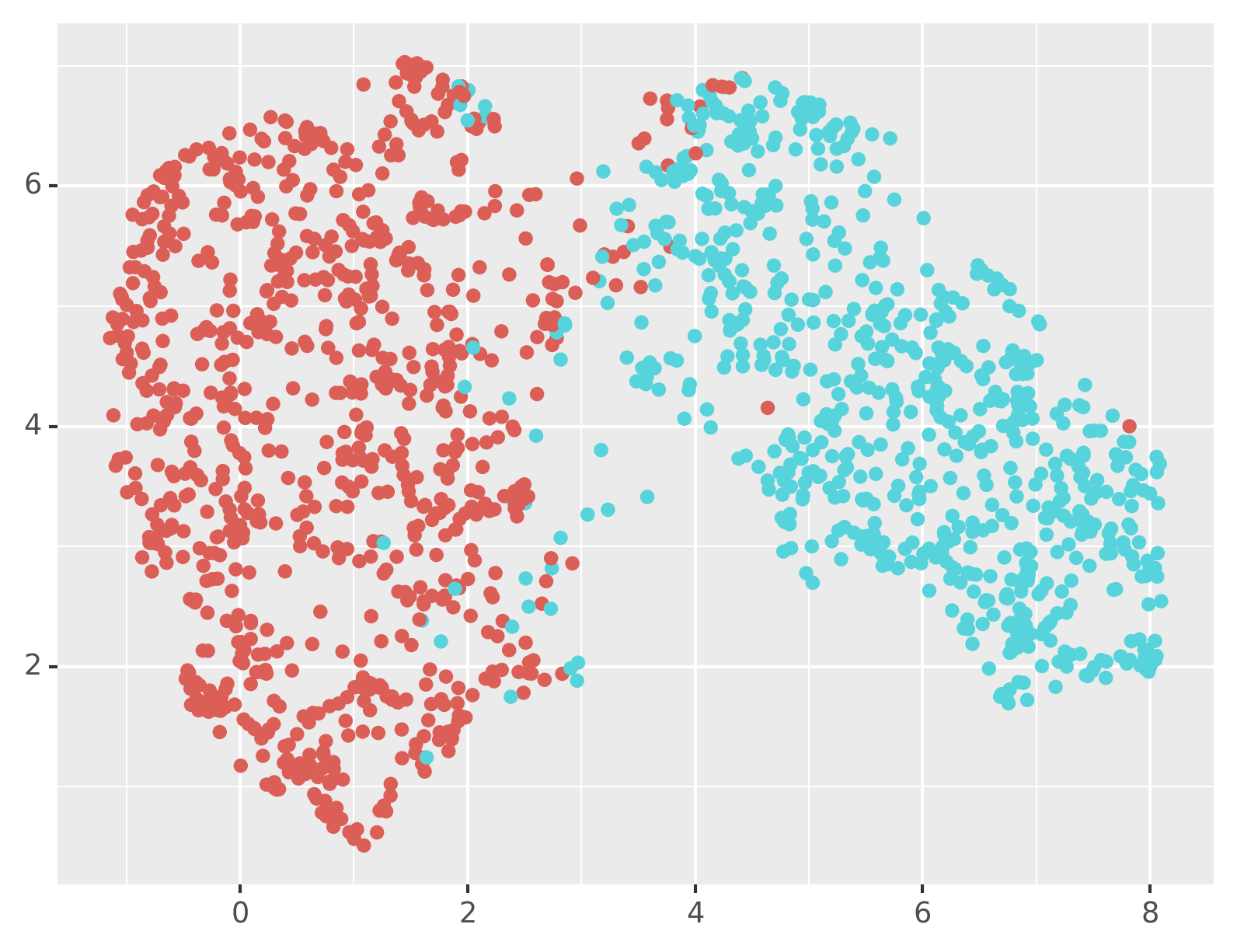}} 
\subfigure{
\includegraphics[width=1.9in]{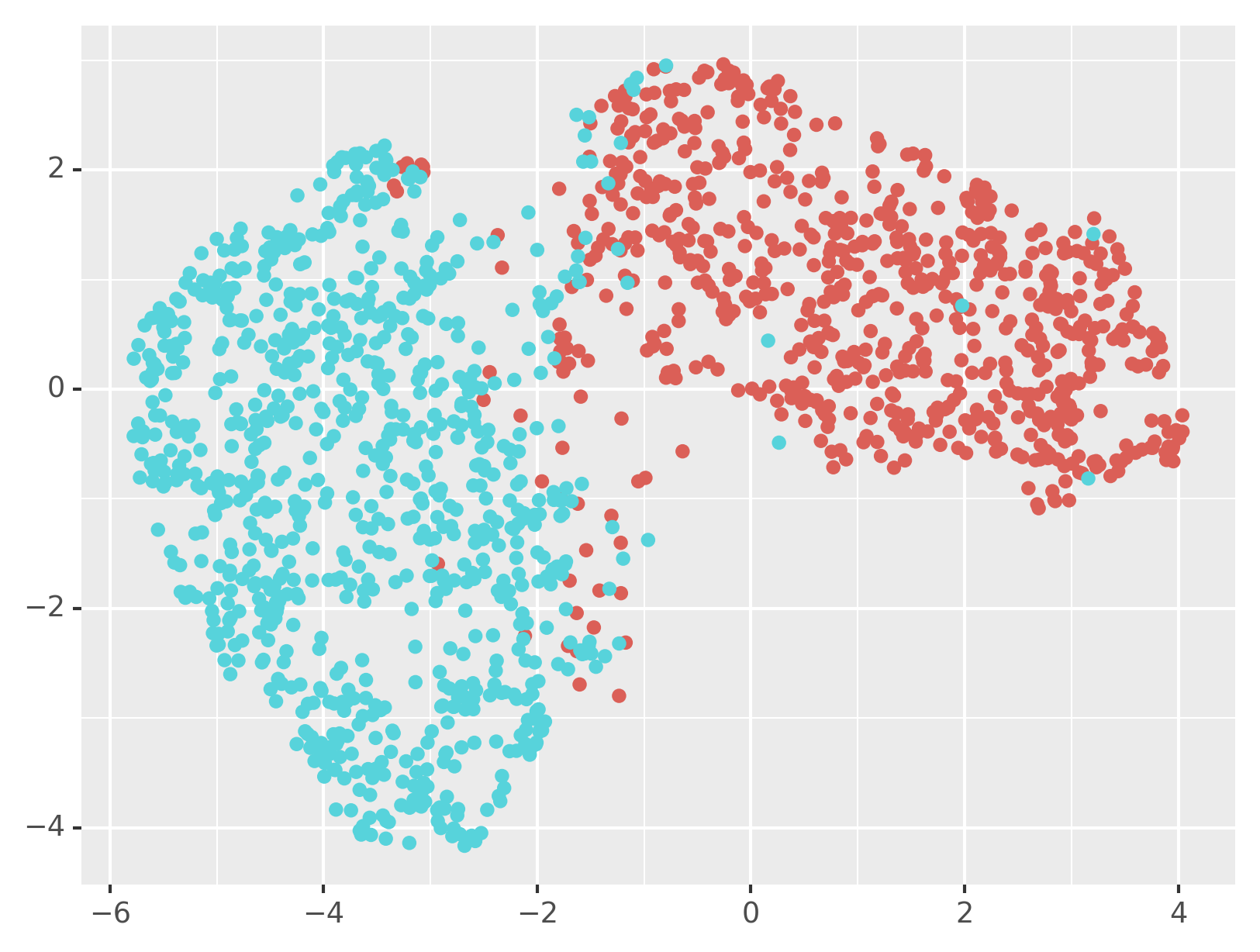}}

\subfigure{
\includegraphics[width=1.9in]{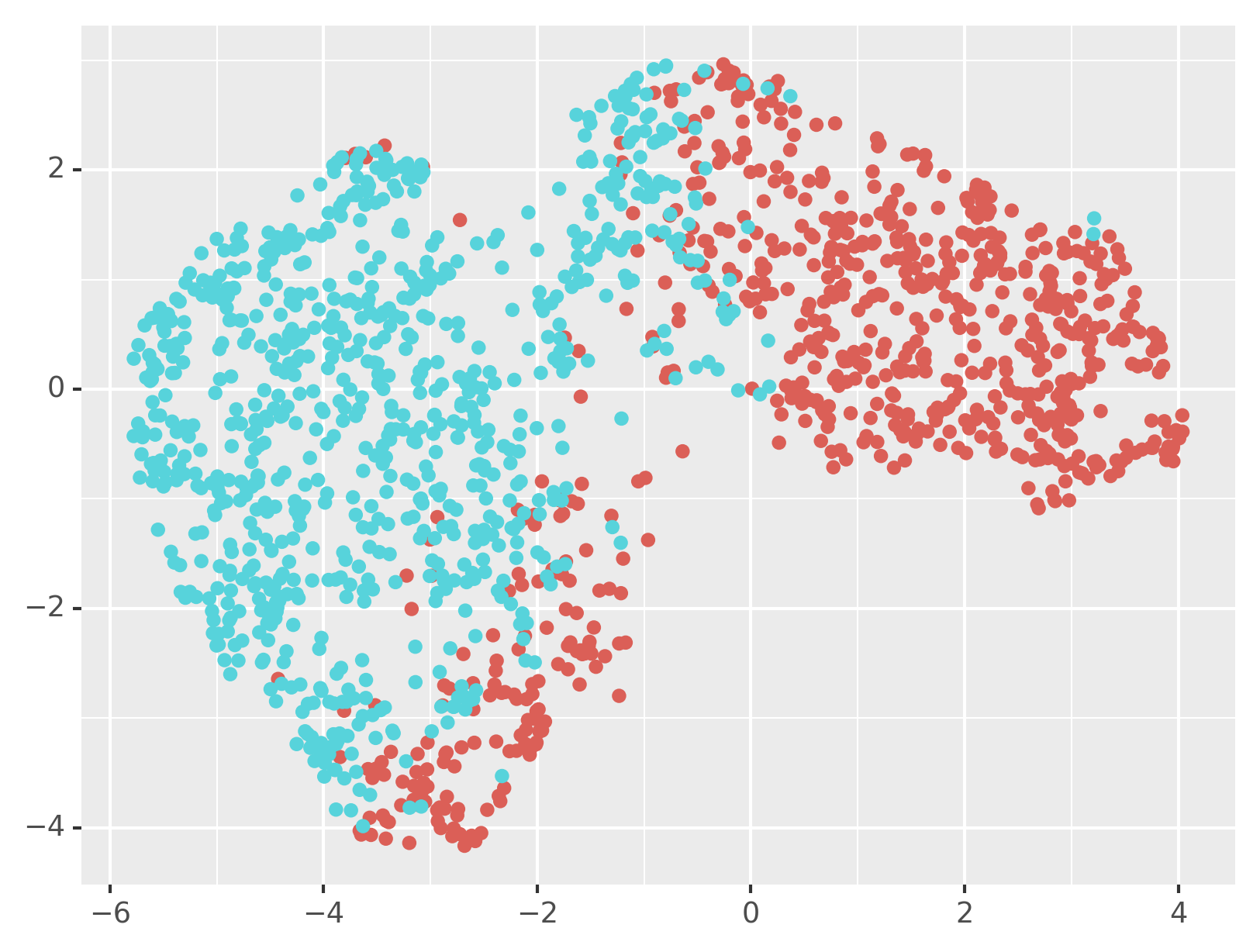}}
\subfigure{
\includegraphics[width=1.9in]{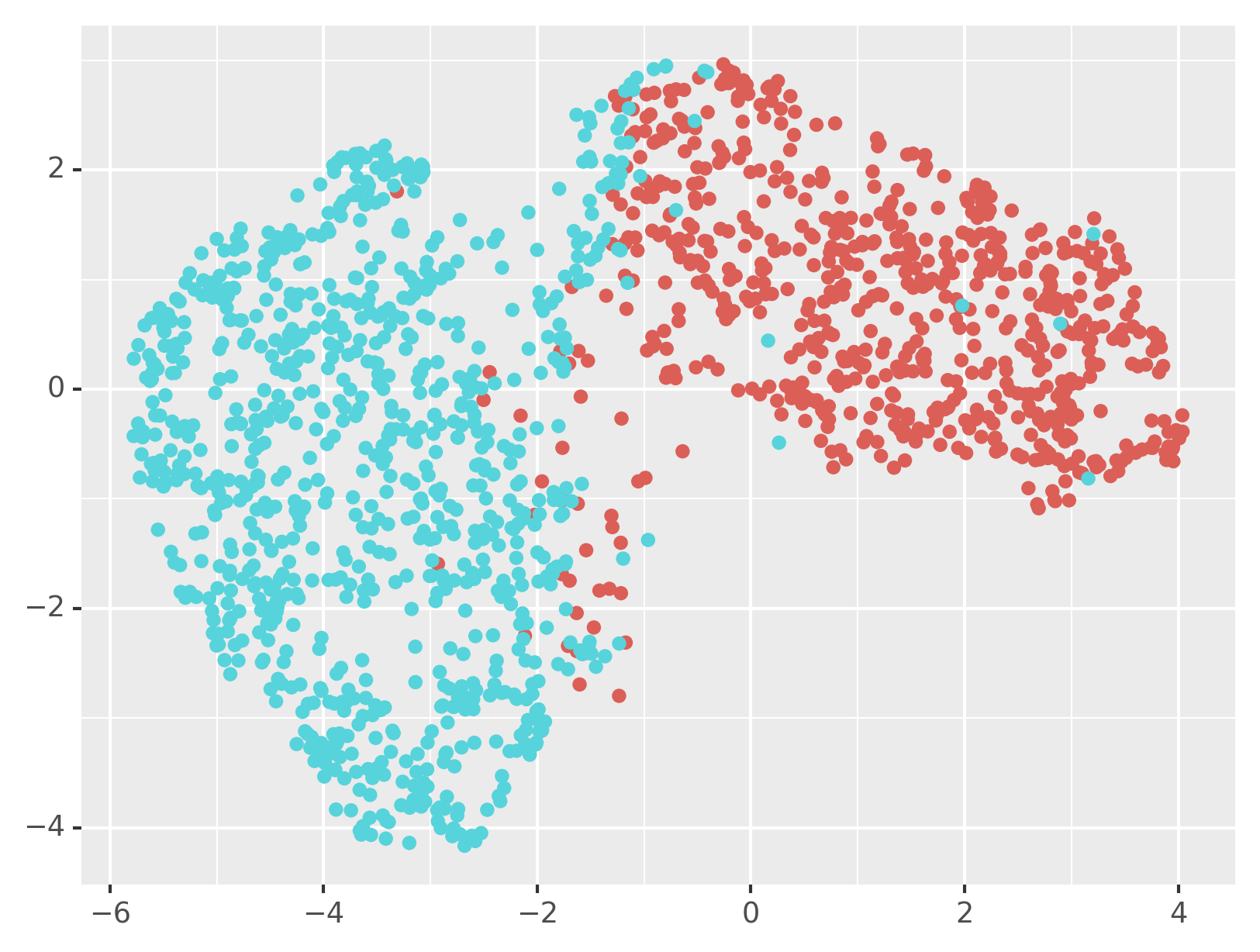}} 
\subfigure{
\includegraphics[width=1.9in]{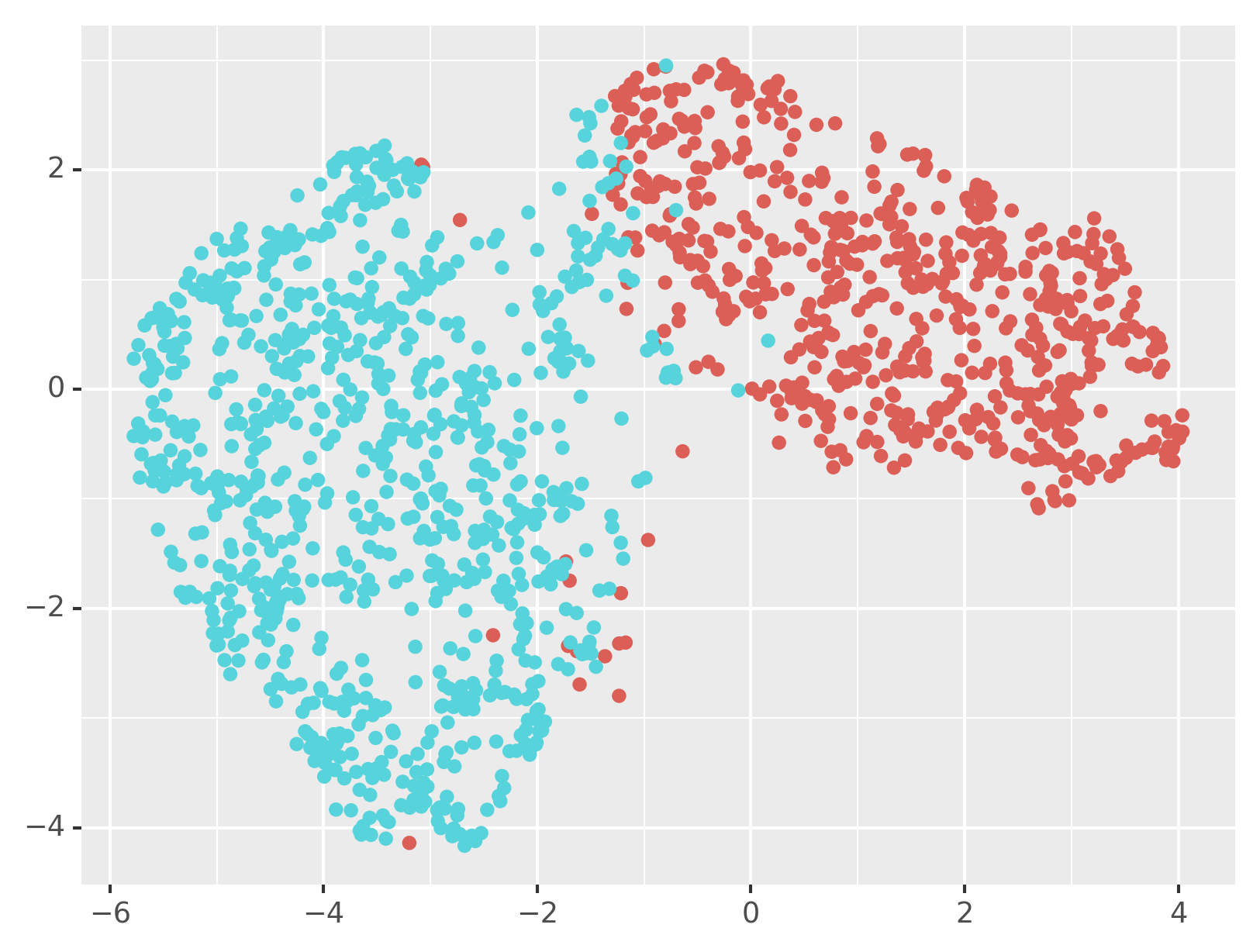}} 

\subfigure{
\includegraphics[width=1.9in]{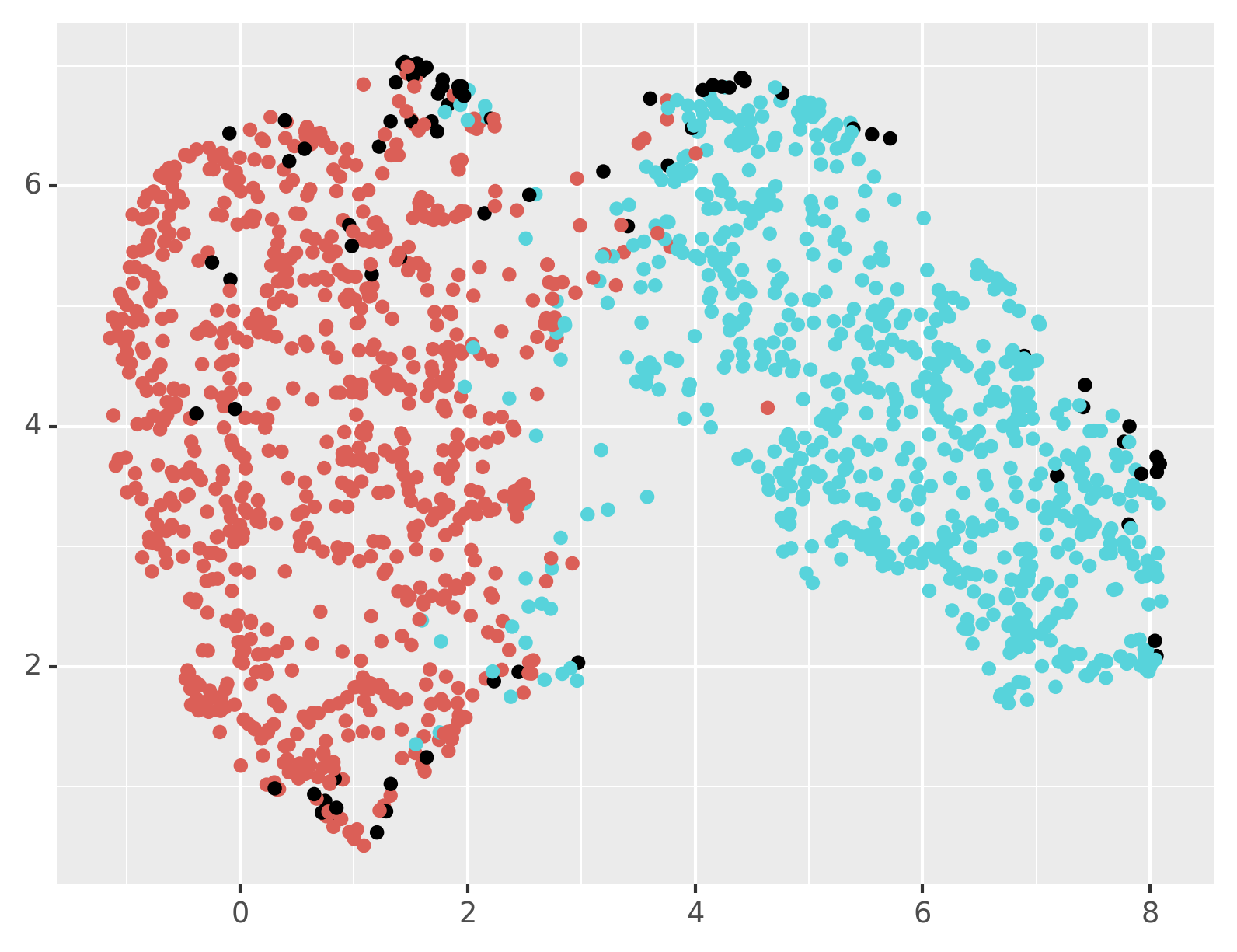}} 
\subfigure{
\includegraphics[width=1.9in]{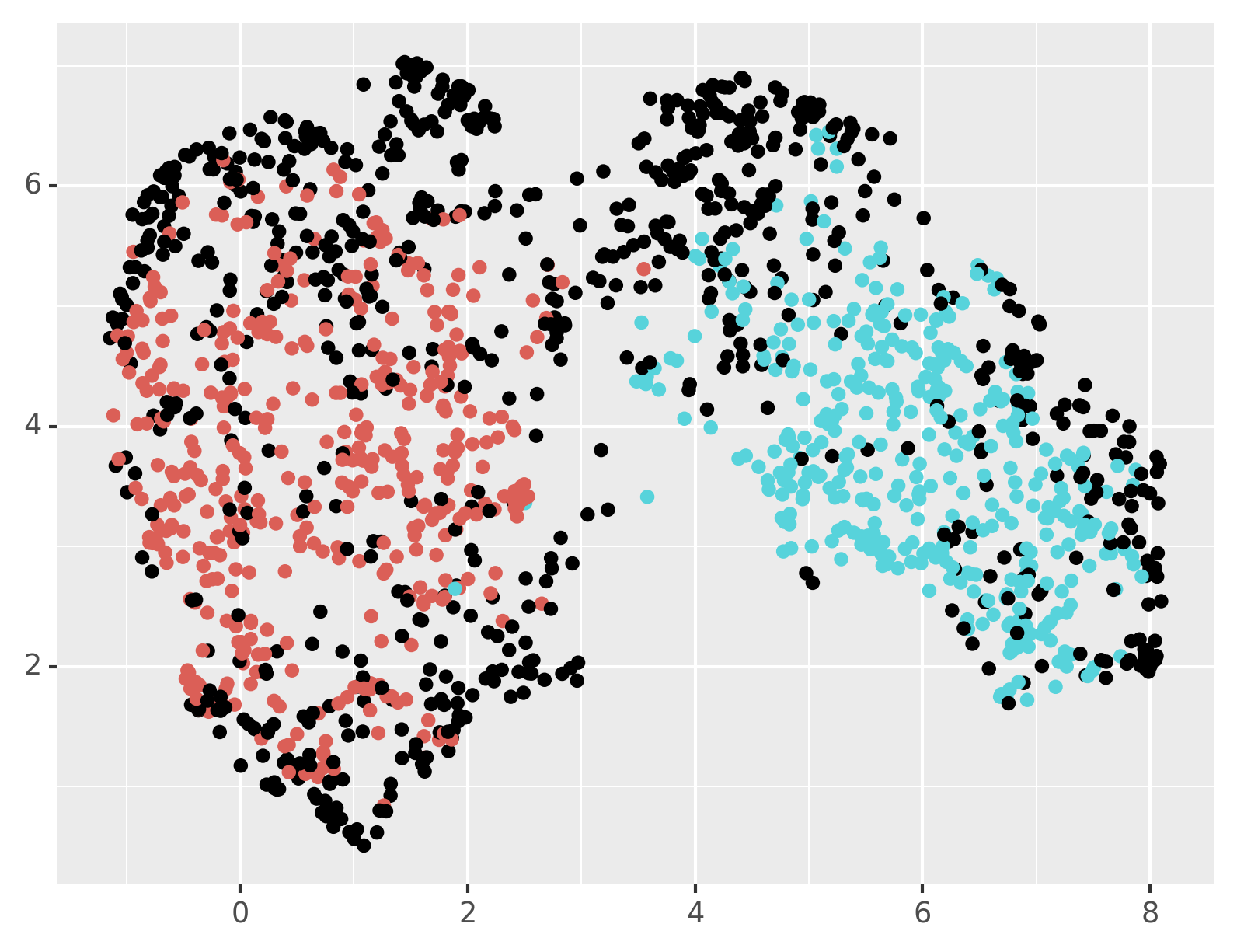}} 
\subfigure{
\includegraphics[width=1.9in]{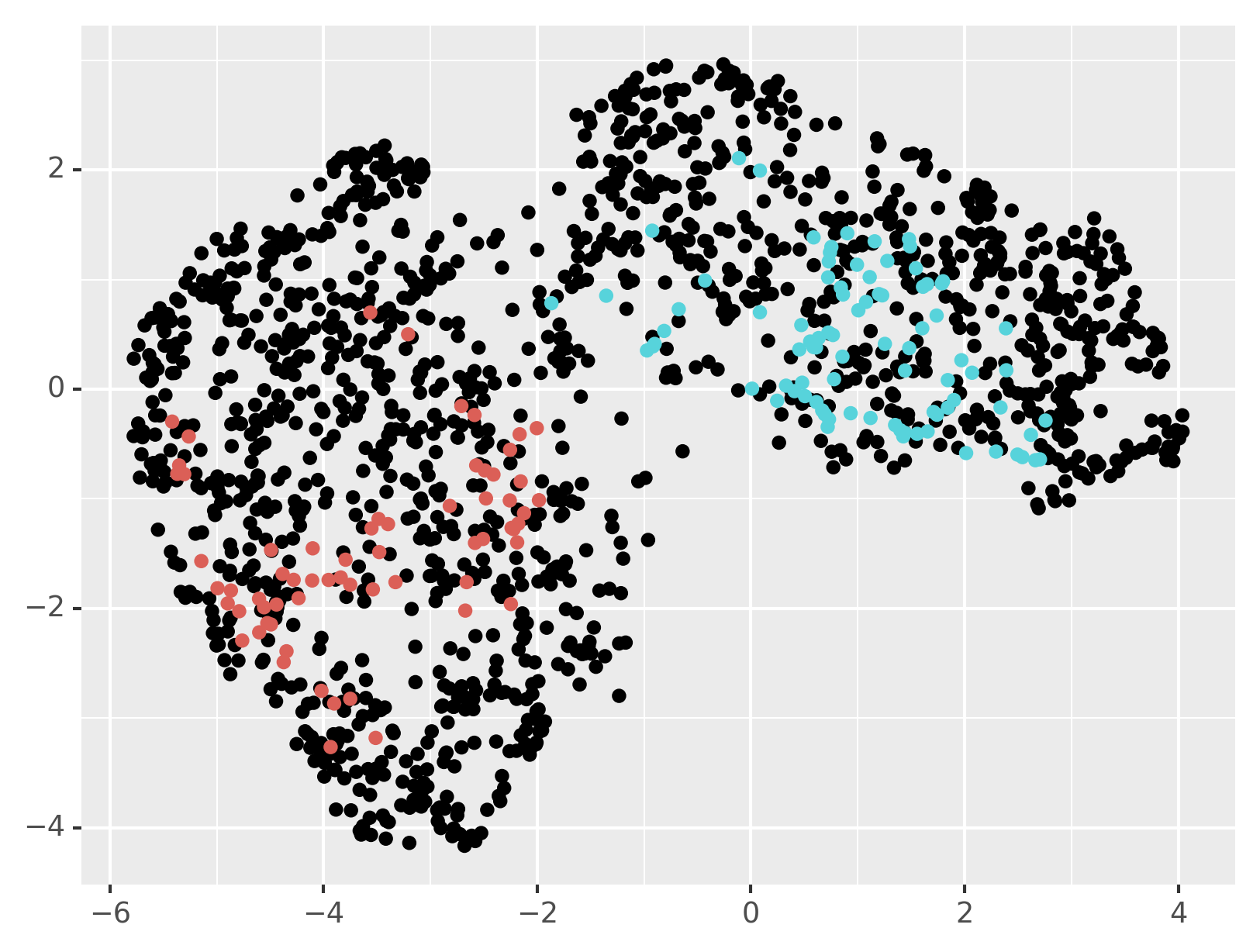}} 
\centering
\caption{UMAP embedding of the 3-8 pair MNIST subset colored with labels. On the first row, from left to right, the real ground truth labels, the F-EM clustering labels and the $t$-EM clustering labels. On the second row, from left to right,  the k-means clustering labels, the GMM-EM labels and the spectral clustering labels. On the bottom from the left to the right, the TCLUST labels, the RIMLE labels and the HDBSCAN labels. Points colored with black are labelled as noise.\vspace{-0.5cm}}
\label{MNIST38}
\end{figure}

For the NORB dataset (some representatives are shown in Figure \ref{picNORB}), k-means, GMM-EM, {spectral clustering and UMAP+HDBSCAN} do not perform in a satisfactory way since they end-up capturing the luminosity as the main classification aspect. In contrast, $t$-EM and the F-EM algorithm highly outperform them, as can be seen in Tables \ref{ami}, \ref{ar} and \ref{ac}. This can be emphasized thanks to results of Figure \ref{NORB1}, where label-colored two-dimensional embeddings of the data based on the classification produced by the different methods are shown. The effect of extreme light values seems to be palliated by the robustness properties of the estimators. \\


Finally, the \textit{20newsgroup} data set is a bag of words constructed from a corpus of news. Each piece of news is classified by topic modeling into twenty groups. Once again, we compare the performance of our methods with the ones of k-means, EM, $t$-EM, TCLUST , RIMLE and spectral clustering  algorithms after applying PCA. The corresponding results are also presented in Tables \ref{ami}, \ref{ar} and \ref{ac}. One can see that k-means, TCLUST, RIMLE and spectral clustering perform poorly, while GMM-EM and $t$-EM outperform them. Nevertheless, the proposed F-EM algorithm has  strongly better results than the others. It is not clear why spectral clustering is performing so badly on this data set, it could be due to the lack of separation between clusters and/or the presence of noise that breaks the performance. Finally, the very poor capability of the RIMLE algorithm in this dataset is explained by the choice of the parameter that highly over estimates the noise.
\vspace*{0.2cm}

\begin{figure}[!ht]
\centering
\subfigure{
\includegraphics[width=1.2in]{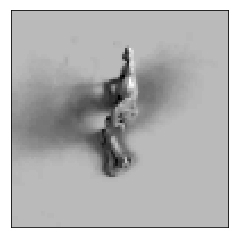}}
\subfigure{
\includegraphics[width=1.2in]{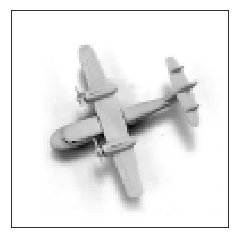}}
\subfigure{
\includegraphics[width=1.2in]{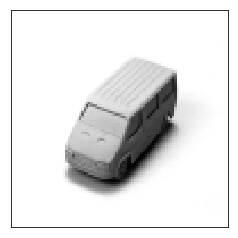}}
\subfigure{
\includegraphics[width=1.2in]{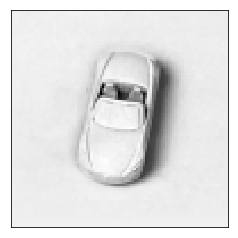}}
\centering
\caption{Four samples of the small NORB data set from the 4 considered categories. Differences in brightness between the pictures can be appreciated.}
\label{picNORB}
\end{figure}

\begin{figure}[h!]
    \centering
\subfigure{
\includegraphics[width=0.35\textwidth]{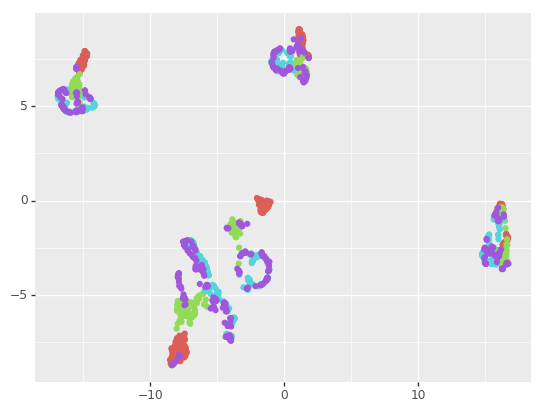}}
\subfigure{
\includegraphics[width=0.35\textwidth]{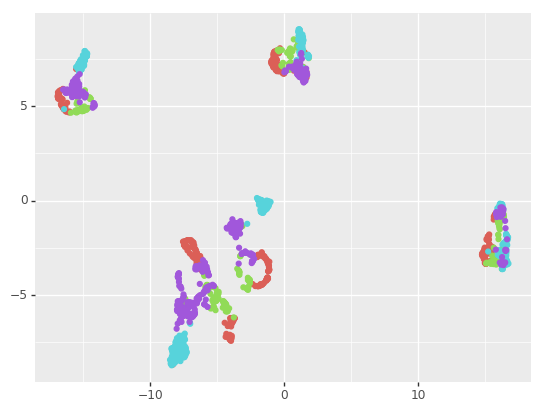}} 

\subfigure{
\includegraphics[width=0.35\textwidth]{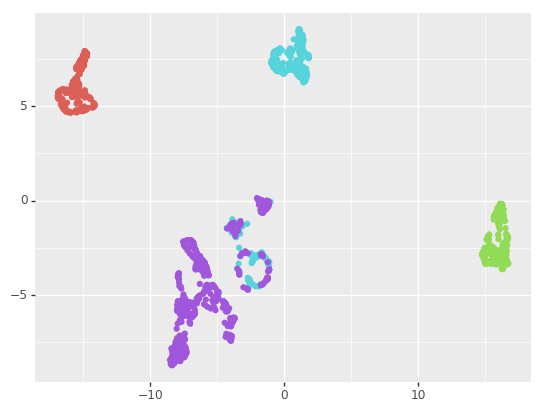}}
\subfigure{
\includegraphics[width=0.35\textwidth]{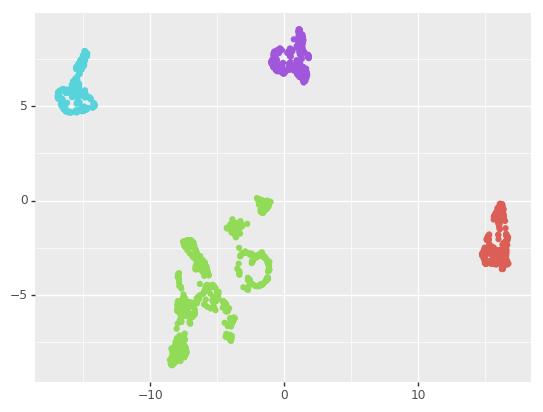}} 
\centering
\caption{NORB's UMAP embedding colored with relative labels. On the top, the real ground truth labels on the left and  the F-EM clustering labels on the right. On the bottom,  the k-means clustering labels on the left and the spectral clustering labels on the right.}
\label{NORB1}
\end{figure}

\section{Concluding Remarks}

In this paper we presented a robust clustering algorithm that outperforms several state of the art algorithms for both synthetic and real diverse data.
Its advantages stem from a general model for the data distribution, where each data point is generated by its own elliptical symmetric distribution. The good theoretical properties of this proposal have been studied and supported by simulations.
The flexibility of this model makes it particularly suitable for analyzing heavy-tailed distributed and/or noise-contaminated data. Interestingly, under mild assumptions on the data, the estimated probabilities of membership do not depend on the data distributions, making the algorithm simpler (no need to re-estimate the likelihood at each step), flexible and robust. Moreover, the original approach of estimating one scale parameter for each data point makes the algorithm competitive in relatively high-dimensional settings.\\



On simulated data, we obtained accurate estimations and good classification rates. Of course, the best model is the one that perfectly coincides with the distribution of the data, \textit{e.g.}, when the mixture is actually Gaussian, GMM-EM outperforms all other methods, including ours, but only marginally, and our method performs well on all considered scenarios.\\
For the real data sets that we considered, We have shown that the proposed  method offers better results compared to k-means, GMM-EM and $t$-EM. It is also competitive with spectral clustering, TCLUST and RIMLE and it still delivers very good results in situations where both HDBSCAN and spectral clustering completely break down.\\

Concerning future works, we consider studying in depth the convergence of the algorithm with a data-driven approach. 
Besides, it would be very interesting to study the impact of the $\tau$ parameters in the model when using them for classification and / or outlier rejection. Finally, we consider that including  a sparse regularization in the scatter estimation would be very useful to take advantage of the fact that the $\tau$ parameters are better estimated when the dimension increases with the number of observations.




\appendix
\section{Proofs}


This Appendix contains the different proofs of propositions provided in this paper (Section~\ref{sec.themodel}).\\

\subsection{Proof of Proposition \ref{proptaus}}
\label{app-proof-prop1}


\begin{proof}
Let us define $s_{ik}=\frac{(\textbf{\textit{x}}_i-\mub_k)^T\Sig_{k}^{-1}(\textbf{\textit{x}}_i-\mub_k)}{\tau_{ik}}$. Then we can rewrite the expression

\begin{equation}
 E_{Z|\textit{\textbf{x}}, \theta^*}[l(Z, \textit{\textbf{x}};\theta)] = \sum_{k = 1}^K l_{0k}\left(\pi_{k}, \mub_k, \Sig_k\right) + \sum_{k = 1}^K \sum_{i = 1}^n l_{ik}\left(\pi_{k}, \mub_k, \Sig_k, \tau_{ik}\right),
 \label{eq.exptwoterms}
\end{equation}
where the terms of the sum are

\begin{equation}
l_{0k}\left(\pi_{k}, \mub_k, \Sig_k\right) =  \sum_{i=1}^n p_{ik} [\log(\pi_k) + \log(A_{ik})
+ \frac{1}{2} \log(|\Sig_{k}^{-1}|) - \frac{m}{2} \log{((\textbf{\textit{x}}_i-\mub_k)^T\Sig_{k}^{-1}(\textbf{\textit{x}}_i-\mub_k))}],
\label{eq.l0}
\end{equation}
and

\begin{equation}
l_{ik}\left(\pi_{k}, \mub_k, 
\Sig_k, \tau_{ik}\right) =  p_{ik}\log(s_{ik}^{m/2} g_{i, k}(s_{ik})).
\end{equation}
If we fix the parameters $\left(\pi_{k}, \mub_k, \Sig_k\right)$ and we maximize $l_{ik}\left(\pi_{k}, \mub_k,  \Sig_k, \tau_{ik}\right)$ w.r.t. $\tau_{ik}$, one obtains that

\begin{equation}
\widehat{\tau}_{ik} = \frac{(\mathbf{{x}}_i-\mub_k)^T\Sig_{k}^{-1}(\mathbf{{x}}_i-\mub_k)}{a_{ik}},
\end{equation}
where $a_{ik} = \underset{t}{\arg\sup}\{t^{m/2}g_{i,k}(t)\}$. 
As we supposed that  $\int t^{m/2}g_{i,k}(t) dt < \infty$, it implies that $t^{m/2}g_{i,k}(t) \to 0$, when $t \to \infty$ so that
the $a_{ik}$ is finite.
Now, observing that 
$$\sum_{k = 1}^K \sum_{i = 1}^n l_{ik}\left(\pi_{k}, \mub_k, \Sig_k, \widehat{\tau}_{ik}\right) = \sum_{i = 1}^n \sum_{k = 1}^K a_{ik}^{m/2}g_{i,k}(a_{ik})  p_{ik},$$ 
and since $p_{ik}=P_{i,\theta^*}(Z_i=k|\mathbf{x}_i=\textit{\textbf{x}}_i)$ and $a_{ik}$ do not depend on $\theta$, then  $\displaystyle \sum_{k = 1}^K \sum_{i = 1}^n l_{ik}\left(\pi_{k}, \mub_k, \Sig_k, \widehat{\tau}_{ik}\right)$ does not depend on the parameters $\left(\pi_{k}, \mub_k, \Sig_k\right)$ for any $k \in {1,...,K}$. Thus,  estimating those parameters will only rely on the first term of the expected likelihood, i.e.,

\begin{equation}
S_0 = \sum_{k = 1}^K l_{0k}\left(\pi_{k}, \mub_k, \Sig_k\right).
\label{eq.suml0}
\end{equation}
Note that $S_0$ involves density functions that are proportional to the Angular Gaussian p.d.f. \citep{ollila2012complex}.
\end{proof}

\subsection{Proof of Proposition \ref{propmusigma}}
\label{app-proof-prop2}





\noindent
\begin{proof}
Let us maximize $E_{Z|\textbf{\textit{x}},\theta^{*}}[l(Z,\textbf{\textit{x}};\theta)]$ with respect to $\theta_k=\left(\pi_{k}, \mub_k,  \Sig_k\right)$, for $k=1,\hdots, K$. Note that the optimization problem is solved under the constraint on the $\{\pi_k\}_{k=1}^K$, which enforces to use a Lagrange multiplier. Cancelling the gradient of the expected conditional log-likelihood thus leads to the following  system of equations

\begin{equation*}
    \frac{\partial [E_{Z|\textbf{\textit{\textit{x}}},\theta^{*}}[l(Z,\textbf{\textit{x}};\theta)]-\lambda (1-\sum_{j=1}^K \pi_j)]}{\partial \pi_k}=\sum_{i=1}^n \frac{p_{ik}}{\pi_k}+\lambda =  0\,, \ \   \forall 1\leq k \leq K,
\end{equation*}
together with the conditions $\sum_{j=1}^K \pi_j=1$ and $ \sum_{j=1}^K p_{ij} = 1$. This is equivalent to $$\pi_k = -\,\cfrac{1}{\lambda} \sum_{i=1}^n p_{ik}.$$ Taking the summation over $k$, together with the constraints, leads to $\lambda =-n$, proving the expression given in Eq.~\eqref{pij}.\\

Let us now consider the derivative of the expected conditional log-likelihood with respect to $\mub_k$. One obtains, for $k=1, \hdots, K$,
\begin{eqnarray*}
 \frac{\partial [E_{Z|\textbf{\textit{x}},\theta^{*}}[l(Z,\textbf{\textit{x}};\theta)]-\lambda (1-\sum_{j=1}^K \pi_j)]}{\partial \mub_k}&=&
 \frac{\partial l_{0k}\left(\pi_{k}, \mub_k, \Sig_k\right)}{\partial \mub_k} \\&=& \frac{m}{2} \sum_{i=1}^n p_{ik} \frac{\Sig_{k}^{-1} (\textbf{\textit{x}}_i-\mub_k)}{(\textbf{\textit{x}}_i-\mub_k)^T\Sig_{k}^{-1}(\textbf{\textit{x}}_i-\mub_k)},
\end{eqnarray*}
where $l_{0k}$ is given in Eq. \eqref{eq.l0}.\\
Then, setting the previous expression to zero leads to 

\begin{eqnarray*}
\mub_k \sum_{i=1}^n \frac{p_{ik}}{(\textbf{\textit{x}}_i-\mub_k)^T\Sig_{k}^{-1}(\textbf{\textit{x}}_i-\mub_k)} = \sum_{i=1}^n \frac{p_{ik}}{(\textbf{\textit{x}}_i-\mub_k)^T\Sig_{k}^{-1}(\textbf{\textit{x}}_i-\mub_k)} \textbf{\textit{x}}_i,
\end{eqnarray*}
providing the result of Eq.~\eqref{cij}. \\

Now, in order to estimate $\Sig_{k}$, we differentiate the expected conditional log-likelihood w.r.t. $\Sig_{k}^{-1}$. One obtains, for $k=1, \hdots, K$,

\begin{eqnarray*}
 \frac{\partial [E_{Z|\textbf{\textit{x}},\theta^{*}}[l(Z,\textbf{\textit{x}};\theta)]-\lambda (1-\sum_{j=1}^K \pi_j)]}{\partial 
 \Sig_k^{-1}} = \frac{\partial l_{0k}\left(\pi_{k}, \mub_k, 
 \Sig_k^{-1}\right)}{\partial \Sig_k^{-1}} &=& \\
  \frac{\partial \left\{\sum_{i=1}^n p_{ik} 
  \left[\frac{1}{2}\log|\Sig_k^{-1}|-\frac{m}{2}\log 
  \left((\textbf{\textit{x}}_i-\mub_k)(\textbf{\textit{x}}_i-\mub_k)^T\Sig_k^
  {-1}\right)\right]\right\}}{\partial \Sig_k^{-1}} &=& \\
 \sum_{i=1}^n p_{ik} \left[\Sig_k - m 
 \frac{(\textbf{\textit{x}}_i-\mub_k)(\textbf{\textit{x}}_i-\mub_k)^T}{(\textbf{\textit{x}}_i-\mub_k)^T\Sig_{k}^{-1}(\textbf{\textit{x}}_i-\mub_k)}\right] &.&  
\end{eqnarray*}
Equating the latter expression to zero leads to Eq.~\eqref{wij} and concludes the proof.
\end{proof}

\subsection{Proof of Proposition \ref{prop.mono}}\label{app-proof-mono}

\begin{proof}
Similarly to the proof given by \cite{EMalgo}, we can decompose $E_{Z|\textbf{\textit{x}},\theta^*}[l(Z,\textbf{\textit{x}};\theta)]$ as follows:



\begin{equation}
E_{Z|\textbf{\textit{x}},\theta^*}[l(Z,\textbf{\textit{x}};\theta)] = l(\textbf{\textit{x}};\theta) + H(\theta, \theta^{*}),
\label{sum.decomp}    
\end{equation}
where we define 

$$H(\theta, \theta^{*}) = 
E_{Z|\textbf{\textit{x}},\theta^*}[l(Z,\textbf{\textit{x}};\theta)] - l(\textbf{\textit{x}};\theta).$$
We can re-write this expression as
\begin{eqnarray}
H(\theta, \theta^{*}) &=&
E_{Z|\textbf{\textit{x}},\theta^*}[l(Z,\textbf{\textit{x}};\theta)] - l(\textbf{\textit{x}};\theta) \nonumber \\ 
&=& \sum_{i=1}^n E_{Z_i|\textbf{\textit{x}}_i,\theta^*}\left[\log\left(\frac{f_{i, \theta}(Z_i,\textbf{\textit{x}}_i)}{f_{i}(\textbf{\textit{x}}_i)}\right)\right] \nonumber \\
&=& \sum_{i=1}^n E_{Z_i|\textbf{\textit{x}}_i,\theta^*}[\log(P_{i, \theta}(Z_i|\textbf{\textit{x}}_i)],\label{eq.hdef}
\end{eqnarray}
where $f_{i,\theta}(Z_i,\textbf{\textit{x}}_i) = \displaystyle \sum_{k=1}^K \Ind_{Z_i=k} f_{i, \theta_k} (\textbf{\textit{x}}_i)$, $\Ind_{\Omega}$ is the indicator function equal to 1 on $\Omega$ and 0 elsewhere. Moreover,  $P_{i,\theta}(a|b)= \frac{f_{i,\theta}(a,b)}{f_{i,\theta}(b)}$.
At this point, we use the fact that $\theta^{(t+1)}$, the set of estimations computed in iteration $t+1$ and derived in Proposition \ref{propmusigma}, fulfills the following equality:

\begin{equation}
  E_{Z|\textbf{\textit{x}},\theta^{(t)]}}[l(Z,\textbf{\textit{x}};{\theta^{(t+1)}})] = \underset{\theta}{\text{max}}  \ E_{Z|\textbf{\textit{x}},\theta^{(t)}}[l(Z,\textbf{\textit{x}};\theta)].  
  \label{eq.maxexp}
\end{equation}
Of course, we need to assume the convergence of the fixed-point equation system.
 Thus, using equation \eqref{sum.decomp} and the fact that $E_{Z|\textbf{\textit{x}},\theta^{(t)}}[l(Z,\textbf{\textit{x}};\theta^{(t+1)})] \geq E_{Z|\textbf{\textit{x}},\theta^{(t)}}[l(Z,\textbf{\textit{x}};\theta^{(t)})]$ from \eqref{eq.maxexp}, we derive the following inequality:

\begin{eqnarray}
l(\textbf{\textit{x}};\theta^{(t+1)}) - l(\textbf{\textit{x}};\theta^{(t)}) \geq H(\theta^{(t)}, \theta^{(t)}) - H(\theta^{(t+1)}, \theta^{(t)}) &=& \nonumber \\ \sum_{i = 1}^n E_{Z_i|\textbf{\textit{x}},\theta^{(t)}}\Bigg[\log{\left(\frac{P_{i, \theta^{(t)}}(Z_i|\textbf{\textit{x}}_i)}{P_{i, \theta^{(t+1)}}(Z_i|\textbf{\textit{x}}_i)}\right)}\Bigg] &\geq& \nonumber\\
-\sum_{i = 1}^n \log\Bigg[{E_{Z|\textbf{\textit{x}},\theta^{(t)}}\left[\frac{P_{i, \theta^{(t+1)}}(Z_i|\textbf{\textit{x}}_i)}{P_{i, \theta^{(t)}}(Z_i|\textbf{\textit{x}}_i)}\right]}\Bigg], \nonumber
\end{eqnarray}
where in the inequality we applied the Jensen inequality for the $-\log$ function. As the expectation is one and in consequence the sum is zero, then $l(\textbf{\textit{x}};\theta^{(t+1)})\geq l(\textbf{\textit{x}};\theta^{(t)})$ and that concludes the proof.
\end{proof}

\subsection{Proof of Proposition \ref{prop.pik}}
\label{app-proof-prop-pik}
\begin{proof}
By definition, one has $p_{ik} = P_{\theta}(Z_i = k | \mathbf{x}_i  =  \textbf{\textit{x}}_i)$. Using the Bayes theorem, one obtains, for $i=1, \hdots, n$ and $k=1, \hdots, K$:
\begin{equation*}
p_{ik} = \frac{{\pi}_k  f_{i,\theta_k}(\textbf{\textit{x}}_i) }{\sum\limits_{j = 1}^K \pi_j  f_{i,\theta_j}(\textbf{\textit{x}}_i)} \end{equation*}
Now, the estimated conditional probability can be written by replacing unknown parameters by their previously derived estimators as

\begin{eqnarray}
\widehat{p}_{ik} &=& \frac{\widehat{\pi}_k  f_{i,\widehat{\theta}_k}(\mathbf{{x}}_i) }{\sum\limits_{j = 1}^K \widehat{\pi}_j  f_{i,\widehat{\theta}_j}(\mathbf{{x}}_i)} \label{eq.bayes} \\ &=&  \frac{\widehat{\pi}_k  A_{ik} \widehat{\tau}_{ik}^{-m/2}|\widehat{\Sig}_k |^{-1/2}\,g_i\left(\frac{(\mathbf{{x}}_i-\widehat{\mub}_k)^T\widehat{\Sig}_k^{-1}(\mathbf{{x}}_i-\widehat{\mub}_k)}{\widehat{\tau}_{ik}}\right)}
{\sum\limits_{j = 1}^K \widehat{\pi}_j  A_{ik} \widehat{\tau}_{ij}^{-m/2} |\widehat{\Sig}_j |^{-1/2}\,g_i\left(\frac{(\mathbf{{x}}_i-\widehat{\mub}_j)^T\widehat{\Sig}_j^{-1}(\mathbf{{x}}_i-\widehat{\mub}_j)}{\widehat{\tau}_{ij}}\right)}.\label{eq.generalpij} \end{eqnarray}

Finally, using the expression of $\widehat{\tau}_{ik}$ given by Eq. \eqref{eq.taua} obtained in Proposition \ref{proptaus}, one obtains
\begin{equation*}
    \widehat{p}_{ik} = \frac{\widehat{\pi}_k \left((\mathbf{{x}}_i-\widehat{\mub}_k)^T\widehat{\Sig}_k^{-1}(\mathbf{{x}}_i-\widehat{\mub}_k)\right)^{-m/2}|\widehat{\Sig}_k |^{-1/2}\,\underset{t}{\max}\left(a_{i}\,t^{m/2}\,g_i(t)\right)}
{\sum_{j = 1}^K \widehat{\pi}_j \left((\mathbf{{x}}_i-\widehat{\mub}_j)^T\widehat{\Sig}_j^{-1}(\mathbf{{x}}_i-\widehat{\mub}_j)\right)^{-m/2}|\widehat{\Sig}_j |^{-1/2}\,\underset{t}{\max}\left(a_{i}\,t^{m/2}\,g_i(t)\right)},
\end{equation*}
where we use that $a_i^{m/2}\,g_i(a_i) = \underset{t}{\max}\left(\,t^{m/2}\,g_i(t)\right)$, by definition of the $a_i = \underset{t}{\arg\sup}\{t^{m/2}g_i(t)\}$. Thus one finally obtains
\begin{equation}
    \widehat{p}_{ik} = \frac{\widehat{\pi}_k \left((\mathbf{{x}}_i-\widehat{\mub}_k)^T\widehat{\Sig}_k^{-1}(\mathbf{{x}}_i-\widehat{\mub}_k)\right)^{-m/2}|\widehat{\Sig}_k |^{-1/2}} {\sum_{j = 1}^K \widehat{\pi}_j \left((\mathbf{{x}}_i-\widehat{\mub}_j)^T\widehat{\Sig}_j^{-1}(\mathbf{{x}}_i-\widehat{\mub}_j)\right)^{-m/2}|\widehat{\Sig}_j |^{-1/2}}.
\end{equation}

\end{proof}

\subsection{Proof of Proposition \ref{prop.student}}
\label{app-proof-prop4}


\noindent

\begin{proof}
We re-write  the $t-$distribution density as
$$f_k(\textbf{\textit{x}}_i) = A_k L_{0ik} s_{ik}^{m/2} g_k(s_{ik}),$$
with $s_{ik}=\frac{(\textbf{\textit{x}}_i-\mub_k)
^T\Sig_{k}^{-1}(\textbf{\textit{x}}_i-\mub_k
)}{\tau_{ik}}$, $A_k=\pi^{-m/2}$, 

\begin{equation}
g_k(t) = \frac{\Gamma(\frac{\nu_k+m}{2})}{\Gamma(\frac{\nu_k}{2})} \nu_k^{-m/2} \left[1 + \frac{t}{\nu_k}\right]^{-(\nu_k+m)/2},    
\end{equation}
and the distribution-free factor

\begin{equation}
L_{0ik} = |\Sig_k|^{-1/2}  [(\textbf{\textit{x}}_i - \mub_k)^T\Sig_k^{-1}(\textbf{\textit{x}}_i - \mub_k)]^{-m/2}.
\end{equation}
With this factorization of the density function, we can work on the two decoupled factors that let us write 
\begin{equation}
\widehat{p}_{ik} = \frac{ \widehat{\pi}_k A_k \widehat{L}_{0ik} \underset{t}{\sup}\{t^{m/2}g_k(t)\}}{\sum\limits_{j=1}^K \widehat{\pi}_j A_j \widehat{L}_{0ij} \underset{t}{\sup}\{t^{m/2}g_j(t)\}}.    
\label{pijhatt}
\end{equation}
First, we  compute the derivative of $t^{m/2}g_k(t)$ to get $\underset{t \geq 0}{\sup}\{t^{m/2}g_j(t)\}$

\begin{eqnarray}
\frac{d}{dt}\left(\frac{\Gamma(\frac{\nu_k+m}{2})}{\Gamma(\frac{\nu_k}{2})} \nu_k^{-m/2} t^{m/2} \left[1 + \frac{t}{\nu_k}\right]^{-(\nu_k+m)/2}\right) &=& \\ 
\left(\frac{m}{2} t^{m/2-1} \left[1+\frac{t}{\nu_k}\right]^{(\nu_k+m)/2} - \frac{t^{m/2}}{\nu_k} \frac{\nu_k+m}{2}\left(1+\frac{t}{\nu_k}\right)^{-(\nu_k+m)/2-1}\right) &=& \\
 t^{m/2-1} \left(1+\frac{t}{\nu_k}\right)^{-(\nu_k+m)/2-1} \left[\frac{m}{2}\left(1+\frac{t}{\nu_k}\right)-\frac{\nu_k+m}{2} \frac{t}{\nu_k} \right].
\end{eqnarray}
Equating the latter to 0, we get two possible solutions: $t = 0$ and $t = m$. Then, $m$ maximizes $t^{m/2}g_j(t)$ and the maximum is reached and is 

$$\frac{\Gamma(\frac{\nu_k+m}{2})}{\Gamma(\frac{\nu_k}{2})} \left(\frac{\nu_k}{m}\right)^{-m/2}  \left[1 + \frac{m}{\nu_k}\right]^{-(\nu_k+m)/2}.$$

\underline{Case large $\nu_k$'s and fixed $m$:}\\
For $\nu_k$ tending to infinity and fixed $m$, one retrieves the Gaussian case as follows

\begin{eqnarray}
\frac{\Gamma(\frac{\nu_k+m}{2})}{\Gamma(\frac{\nu_k}{2})} \left(\frac{\nu_k}{m}\right)^{-m/2}  \left[1 + \frac{m}{\nu_k}\right]^{-(\nu_k+m)/2} &\approx&\\
\frac{\Gamma(\frac{\nu_k}{2})(\frac{\nu_k}{2})^{-m/2}}{\Gamma(\frac{\nu_k}{2})} \left(\frac{\nu_k}{m}\right)^{-m/2}  \left[\left[1 + \frac{1}{\nu_k/m}\right]^{\nu_k/m}\right]^{-m(\nu_k+m)/(2\nu_k)} &\approx&\\
m^{m/2}(2e)^{-m/2},
\end{eqnarray}
where the $2^{-m/2}$ factor corresponds to the normalizing constant of the Gaussian case together with $A_k=\pi^{-m/2}$.\\

\underline{Case large $\nu_k$'s and $m$:}\\
If $\nu_k$ are fixed and $m$ tends to infinity, as in Gaussian case, one notice that this part of the likelihood diverges. Consequently, we assume $\nu_k$ and $m$ are large and of the same rate $\frac{\nu_k}{m} \rightarrow c_k$.

Then, using the Stirling approximation of the Gamma function
$$\Gamma(z) = \sqrt{\frac{2\pi}{z}}\left(\frac{z}{e}\right)^z \left(1 + O\left(\frac{1}{z}\right)\right),$$ 
we derive the following approximations.

\begin{eqnarray}
\frac{\Gamma(\frac{\nu_k+m}{2})}{\Gamma(\frac{\nu_k}{2})} \left(\frac{\nu_k}{m}\right)^{-m/2}  \left[1 + \frac{m}{\nu_k}\right]^{-(\nu_k+m)/2} &=& \nonumber\\
\frac{\Gamma\left(\frac{(1+c_k)m}{2}\right)}{\Gamma\left(\frac{c_km}{2}\right)} c_k^{-m/2}  \left[1 + \frac{1}{c_k}\right]^{-(1+c_k)m/2} &=& \nonumber \\
\frac{\sqrt{\frac{4\pi}{(1+c_k)m}} \left(\frac{(1+c_k)m}{2e}\right)^{(1+c_k)m/2} \left(1 + O\left(\frac{1}{m}\right)\right)}{\sqrt{\frac{4\pi}{c_km}} \left(\frac{c_k m}{2e}\right)^{c_km/2} \left(1 + O\left(\frac{1}{m}\right)\right)} c_k^{-m/2}  \left[1 + \frac{1}{c_k}\right]^{-(1+c_k)m/2}   &=& \nonumber \\
\sqrt{\frac{c_k}{1+c_k}} (2e)^{-\frac{m}{2}} \left(\frac{1+c_k}{c_k}\right)^{(1+c_k)m/2} m^{\frac{m}{2}} \frac{\left(1 + O\left(\frac{1}{m}\right)\right)}{ \left(1+O\left(\frac{1}{m}\right)\right)}   \left[ \frac{1+c_k}{c_k}\right]^{-(1+c_k)m/2}   &=& \nonumber \\
\sqrt{\frac{c_k}{1+c_k}} (2e)^{-\frac{m}{2}} m^{\frac{m}{2}} \frac{\left(1 + O\left(\frac{1}{m}\right)\right)}{ \left(1+O\left(\frac{1}{m}\right)\right)}.     && \nonumber
\end{eqnarray}
When replacing this expression in \eqref{pijhatt} we derived the following approximation,

\begin{eqnarray}
\widehat{p}_{ik} &=& \frac{ \widehat{\pi}_k \widehat{L}_{0ik} \sqrt{\frac{c_k}{1+c_k}} (2e)^{-\frac{m}{2}} m^{\frac{m}{2}} \frac{\left(1 + O\left(\frac{1}{m}\right)\right)}{ \left(1+O\left(\frac{1}{m}\right)\right)}}{\sum\limits_{j=1}^K \left(\widehat{\pi}_j \widehat{L}_{0ij} \sqrt{\frac{c_j}{1+c_j}} (2e)^{-\frac{m}{2}} m^{\frac{m}{2}} \frac{\left(1 + O\left(\frac{1}{m}\right)\right)}{ \left(1+O\left(\frac{1}{m}\right)\right)}\right)} \nonumber \\
&=& \frac{ \widehat{\pi}_k \widehat{L}_{0ik} \sqrt{\frac{c_k}{1+c_k}} \left(1 + O\left(\frac{1}{m}\right)\right)}{\sum\limits_{j=1}^K \left( \widehat{\pi}_j \widehat{L}_{0ij} \sqrt{\frac{c_j}{1+c_j}} \left(1 + O\left(\frac{1}{m}\right)\right)\right)} \nonumber \\
&=& \frac{ \widehat{\pi}_k \widehat{L}_{0ik} \sqrt{\frac{c_k}{1+c_k}} + O\left(\frac{1}{m}\right)}{\sum\limits_{j=1}^K \left( \widehat{\pi}_j \widehat{L}_{0ij} \sqrt{\frac{c_j}{1+c_j}} \right) + O\left(\frac{1}{m}\right)}. \nonumber \\
\end{eqnarray}
Finally, one obtains for $i=1, \hdots, n$ and $k=1, \hdots, K$
$$\widehat p_{ik} = \frac{ \widehat{\pi}_k \widehat{L}_{0ik} \sqrt{\frac{c_k}{1+c_k}} }{\sum\limits_{j=1}^K \left( \widehat{\pi}_j \widehat{L}_{0ij} \sqrt{\frac{c_j}{1+c_j}} \right)} + O\left(\frac{1}{m}\right). $$

If $c_k=c$ for all $k$, we retrieve the same result as the particular case developed in Section \ref{general-pik-case}. 


\end{proof}

\subsection{Proof of Proposition \ref{prop.hd}}
\label{app-proof-prop5}

\begin{proof}
\textcolor{black}{By Tyler's Theorem that applies to elliptical distributions under the assumptions included in the proposition, we have the convergence of the scatter matrix $\widehat{\Sig}_k$ to the true $\Sig_k$ in probability \citep[][Theorem 4.1]{tyler1987distribution}.}

Then, by applying the continuous mapping theorem, it follows that $\widehat{\Sig}_k^{-1} \overset{\mathcal{P}}{\longrightarrow} {\Sig_k}^{-1}$. Given that $\mathbf x_i$, one has

 \begin{eqnarray*}
\widehat{\tau}_{ik} &=& \frac{(\mathbf{x}_i-{\widehat{\mub}_k})^T\widehat{\Sig}_k^{-1}(\mathbf{x}_i-{\widehat{\mub}_k})}{m} \\
&=& \frac{(\sqrt{\tau_{ik}}\Ab_k\mathbf{q}_i+{\mub_k}-\widehat{\mub}_k)^T\widehat{\Sig}_k^{-1}(\sqrt{\tau_{ik}}\Ab_k\mathbf{q}_i+{\mub_k}-\widehat{\mub}_k)}{m}.
\end{eqnarray*}

Combining $\widehat{\Sig}_k^{-1} \overset{\mathcal{P}}{\longrightarrow} {\Sig_k}^{-1}$ and $\widehat{\mub}_k \overset{\mathcal{P}}{\longrightarrow} \mub_k $ and the Slutsky theorem leads to 

$$\widehat{\tau}_{ik} \overset{\mathcal{P}}{\longrightarrow} \frac{\sqrt{\tau_{ik}}\mathbf{q}_i^T\Ab_k^T \Sig^{-1}\sqrt{\tau_{ik}}\Ab_k\mathbf{q}}{m} = \frac{\tau_{ik} \mathbf{q}_i^T\mathbf{q}_i}{m}.$$
Furthermore, 

$$\frac{\tau_{ik} \mathbf{q}_i^T\mathbf{q}_i}{m}=\frac{\tau_{ik} \sum_{l=1}^m(\mathbf{q}_i)^2_l}{m},$$
with the components $(\mathbf{q}_i)_1^2,..., (\mathbf{q}_i)_m^2$ i.i.d. distributed as $\chi^2(1)$ because $\mathbf{q}_i\sim \mathcal{N}(0, \Iden_m)$.  Thus, 
$\widehat{\tau}_{ik}$ tends to $\tau_{ik} \frac{\chi^2(m)}{m}$. 

Now, to assess the behavior when m tends to infinity, one has thanks to the Central Limit Theorem that, since $E[(\mathbf{q}_i)_1^2]=1$ and $V[(\mathbf{q}_i)_1^2]=2$, for $m$ large enough
 
 $$\frac{\tau_{ik} \sum_{l=1}^m(\mathbf{q}_i)^2_l}{m} \underset{}{\sim} \mathcal{N}(\tau_{ik}, 2\tau_{ik}^2/m),$$ 

Finally, sequentially combining the approximations and imposing the condition $n>m(2m-1)$ to ensure the existence and uniqueness of the estimator, one obtains the limiting distribution for $(\widehat{\tau}_{ik} - \tau_{ik})$.
\end{proof}

\vskip 0.2in
\bibliographystyle{plainnat}
\bibliography{biblio}

\begin{thebibliography}{58}
\providecommand{\natexlab}[1]{#1}
\providecommand{\url}[1]{\texttt{#1}}
\expandafter\ifx\csname urlstyle\endcsname\relax
  \providecommand{\doi}[1]{doi: #1}\else
  \providecommand{\doi}{doi: \begingroup \urlstyle{rm}\Url}\fi

\bibitem[{Banfield} and {Raftery}(1993)]{unif}
J.~D. {Banfield} and A.~E. {Raftery}.
\newblock Model-based gaussian and non-gaussian clustering.
\newblock \emph{Biometrics}, 49\penalty0 (3):\penalty0 803--821, 1993.
\newblock ISSN 0006341X, 15410420.
\newblock URL \url{http://www.jstor.org/stable/2532201}.

\bibitem[Bilodeau and Brenner(1999)]{Bilodeau1999}
M.~Bilodeau and D.~Brenner.
\newblock \emph{Robustness}, pages 206--242.
\newblock Springer New York, New York, NY, 1999.
\newblock ISBN 978-0-387-22616-3.
\newblock \doi{10.1007/978-0-387-22616-3_13}.
\newblock URL \url{https://doi.org/10.1007/978-0-387-22616-3_13}.

\bibitem[Boente et~al.(2014)Boente, Barrera, and Tyler]{boente}
G.~Boente, M.~Salibi{\'a}n Barrera, and D.~E. Tyler.
\newblock A characterization of elliptical distributions and some optimality
  properties of principal components for functional data.
\newblock \emph{Journal of Multivariate Analysis}, 131:\penalty0 254 -- 264,
  2014.
\newblock ISSN 0047-259X.
\newblock \doi{https://doi.org/10.1016/j.jmva.2014.07.006}.
\newblock URL
  \url{http://www.sciencedirect.com/science/article/pii/S0047259X14001638}.

\bibitem[{Bouveyron} and {Brunet-Saumard}(2014)]{BOUVEYRON201452}
C.~{Bouveyron} and C.~{Brunet-Saumard}.
\newblock Model-based clustering of high-dimensional data: A review.
\newblock \emph{Computational Statistics \& Data Analysis}, 71:\penalty0 52 --
  78, 2014.
\newblock ISSN 0167-9473.
\newblock \doi{https://doi.org/10.1016/j.csda.2012.12.008}.
\newblock URL
  \url{http://www.sciencedirect.com/science/article/pii/S0167947312004422}.

\bibitem[Browne and McNicholas(2015)]{hyper2}
R.~P. Browne and P.~D. McNicholas.
\newblock A mixture of generalized hyperbolic distributions.
\newblock \emph{Canadian Journal of Statistics}, 43\penalty0 (2):\penalty0
  176--198, 2015.
\newblock \doi{10.1002/cjs.11246}.
\newblock URL \url{https://onlinelibrary.wiley.com/doi/abs/10.1002/cjs.11246}.

\bibitem[{Campbell}(1984)]{weights1}
N.~A. {Campbell}.
\newblock Mixture models and atypical values.
\newblock \emph{Journal of the International Association for Mathematical
  Geology}, 16\penalty0 (5):\penalty0 465--477, 1984.
\newblock ISSN 1573-8868.
\newblock \doi{10.1007/BF01886327}.
\newblock URL \url{https://doi.org/10.1007/BF01886327}.

\bibitem[{Campello} et~al.(2015){Campello}, {Moulavi}, {Zimek}, and
  {Sander}]{HDBSCAN1}
R.~J. G.~B. {Campello}, D.~{Moulavi}, A.~{Zimek}, and J.~{Sander}.
\newblock Hierarchical density estimates for data clustering, visualization,
  and outlier detection.
\newblock \emph{ACM Trans. Knowl. Discov. Data}, 10\penalty0 (1):\penalty0
  5:1--5:51, 2015.
\newblock ISSN 1556-4681.
\newblock \doi{10.1145/2733381}.
\newblock URL \url{http://doi.acm.org/10.1145/2733381}.

\bibitem[Celeux and Govaert(1995)]{CELEUX1995781}
G.~Celeux and G.~Govaert.
\newblock Gaussian parsimonious clustering models.
\newblock \emph{Pattern Recognition}, 28\penalty0 (5):\penalty0 781 -- 793,
  1995.
\newblock ISSN 0031-3203.
\newblock \doi{https://doi.org/10.1016/0031-3203(94)00125-6}.
\newblock URL
  \url{http://www.sciencedirect.com/science/article/pii/0031320394001256}.

\bibitem[{Conte} and {Longo}(1987)]{CG2}
E.~{Conte} and M.~{Longo}.
\newblock Characterisation of radar clutter as a spherically invariant random
  process.
\newblock \emph{IEE Proceedings F - Communications, Radar and Signal
  Processing}, 134\penalty0 (2):\penalty0 191--197, 1987.
\newblock ISSN 0143-7070.
\newblock \doi{10.1049/ip-f-1.1987.0035}.

\bibitem[{Conte} et~al.(2002{\natexlab{a}}){Conte}, {De Maio}, and
  {Ricci}]{CFAR}
E.~{Conte}, A.~{De Maio}, and G.~{Ricci}.
\newblock Covariance matrix estimation for adaptive {CFAR} detection in
  compound-gaussian clutter.
\newblock \emph{IEEE Transactions on Aerospace and Electronic Systems},
  38\penalty0 (2):\penalty0 415--426, 2002{\natexlab{a}}.
\newblock ISSN 0018-9251.
\newblock \doi{10.1109/TAES.2002.1008976}.

\bibitem[{Conte} et~al.(2002{\natexlab{b}}){Conte}, {De Maio}, and {Ricci}]{c1}
E.~{Conte}, A.~{De Maio}, and G.~{Ricci}.
\newblock Recursive estimation of the covariance matrix of a compound-gaussian
  process and its application to adaptive cfar detection.
\newblock \emph{IEEE Transactions on Signal Processing}, 50\penalty0
  (8):\penalty0 1908--1915, 2002{\natexlab{b}}.
\newblock ISSN 1053-587X.
\newblock \doi{10.1109/TSP.2002.800412}.

\bibitem[{Coretto} and {Hennig}(2017)]{RIMLE}
P.~{Coretto} and C.~{Hennig}.
\newblock Consistency, breakdown robustness, and algorithms for robust improper
  maximum likelihood clustering.
\newblock \emph{J. Mach. Learn. Res.}, 18\penalty0 (1):\penalty0 5199--5237,
  January 2017.
\newblock ISSN 1532-4435.
\newblock URL \url{http://dl.acm.org/citation.cfm?id=3122009.3208023}.

\bibitem[Coretto and Hennig(2019)]{otrimle}
Pietro Coretto and Christian Hennig.
\newblock \emph{otrimle: Robust Model-Based Clustering}, 2019.
\newblock R package version 1.3.

\bibitem[{Couillet} et~al.(2014){Couillet}, {Pascal}, and
  {Silverstein}]{6891244}
R.~{Couillet}, F.~{Pascal}, and J.~W. {Silverstein}.
\newblock Robust estimates of covariance matrices in the large dimensional
  regime.
\newblock \emph{IEEE Transactions on Information Theory}, 60\penalty0
  (11):\penalty0 7269--7278, Nov 2014.
\newblock ISSN 1557-9654.
\newblock \doi{10.1109/TIT.2014.2354045}.

\bibitem[Couillet et~al.(2015)Couillet, Pascal, and
  Silverstein]{COUILLET201556}
R.~Couillet, F.~Pascal, and J.~W. Silverstein.
\newblock The random matrix regime of {M}aronna's {M}-estimator with
  elliptically distributed samples.
\newblock \emph{Journal of Multivariate Analysis}, 139:\penalty0 56 -- 78,
  2015.
\newblock ISSN 0047-259X.
\newblock \doi{https://doi.org/10.1016/j.jmva.2015.02.020}.
\newblock URL
  \url{http://www.sciencedirect.com/science/article/pii/S0047259X15000676}.

\bibitem[{Dempster} et~al.(1977){Dempster}, {Laird}, and {Rubin}]{EMalgo}
A.~P. {Dempster}, N.~M. {Laird}, and D.~B. {Rubin}.
\newblock Maximum likelihood from incomplete data via the {EM} algorithm.
\newblock \emph{Journal of the Royal Statistical Society: Series B},
  39:\penalty0 1--38, 1977.
\newblock URL \url{http://web.mit.edu/6.435/www/Dempster77.pdf}.

\bibitem[Fraley and Raftery(2002)]{badperf}
C.~Fraley and A.~E. Raftery.
\newblock Model-based clustering, discriminant analysis, and density
  estimation.
\newblock \emph{Journal of the American Statistical Association}, 97\penalty0
  (458):\penalty0 611--631, 2002.
\newblock ISSN 01621459.
\newblock URL \url{http://www.jstor.org/stable/3085676}.

\bibitem[Frontera-Pons et~al.(2016)Frontera-Pons, Veganzones, Pascal, and
  Ovarlez]{frontera2014hyperspectral}
J.~Frontera-Pons, M.~Veganzones, F.~Pascal, and J-P. Ovarlez.
\newblock {Hyperspectral Anomaly Detectors using Robust Estimators}.
\newblock \emph{IEEE Journal of Selected Topics in Applied Earth Observations
  and Remote Sensing (JSTARS)}, 9\penalty0 (2):\penalty0 720--731, february
  2016.

\bibitem[{Garc{\'i}a-Escudero} et~al.(2008){Garc{\'i}a-Escudero}, {Gordaliza},
  {Matr{\'a}n}, and {Mayo-Iscar}]{Tclust}
L.~A. {Garc{\'i}a-Escudero}, A.~{Gordaliza}, C.~{Matr{\'a}n}, and
  A.~{Mayo-Iscar}.
\newblock A general trimming approach to robust cluster analysis.
\newblock \emph{Ann. Statist.}, 36\penalty0 (3):\penalty0 1324--1345, 2008.
\newblock \doi{10.1214/07-AOS515}.
\newblock URL \url{https://doi.org/10.1214/07-AOS515}.

\bibitem[{Gebru} et~al.(2016){Gebru}, {Alameda-Pineda}, {Forbes}, and
  {Horaud}]{weighted}
I.~D. {Gebru}, X.~{Alameda-Pineda}, F.~{Forbes}, and R.~{Horaud}.
\newblock {EM} algorithms for weighted-data clustering with application to
  audio-visual scene analysis.
\newblock \emph{IEEE Trans. Pattern Anal. Mach. Intell.}, 38\penalty0
  (12):\penalty0 2402--2415, 2016.
\newblock ISSN 0162-8828.
\newblock \doi{10.1109/TPAMI.2016.2522425}.
\newblock URL \url{https://doi.org/10.1109/TPAMI.2016.2522425}.

\bibitem[{Gini} and {Farina}(2002)]{CG}
F.~{Gini} and A.~{Farina}.
\newblock Vector subspace detection in compound-gaussian clutter. part {I}:
  survey and new results.
\newblock \emph{IEEE Transactions on Aerospace and Electronic Systems},
  38\penalty0 (4):\penalty0 1295--1311, 2002.
\newblock ISSN 0018-9251.
\newblock \doi{10.1109/TAES.2002.1145751}.

\bibitem[{Gini} et~al.(2000){Gini}, {Greco}, {Diani}, and {Verrazzani}]{radar}
F.~{Gini}, M.~V. {Greco}, M.~{Diani}, and L.~{Verrazzani}.
\newblock Performance analysis of two adaptive radar detectors against
  non-gaussian real sea clutter data.
\newblock \emph{IEEE Transactions on Aerospace and Electronic Systems},
  36\penalty0 (4):\penalty0 1429--1439, 2000.
\newblock ISSN 0018-9251.
\newblock \doi{10.1109/7.892695}.

\bibitem[Gonzalez(2019)]{Gonzalezmezcla}
J.~D. Gonzalez.
\newblock \emph{M\'etodos de clustering robustos}.
\newblock PhD thesis, Universidad de Buenos Aires, 2019.

\bibitem[{Gonzalez} et~al.(2019){Gonzalez}, {Yohai}, and {Zamar}]{ktau}
J.~D. {Gonzalez}, V.~J. {Yohai}, and R.~H. {Zamar}.
\newblock {Robust Clustering Using Tau-Scales}.
\newblock \emph{arXiv e-prints}, art. arXiv:1906.08198, Jun 2019.

\bibitem[Hennig(2015)]{henning}
C.~Hennig.
\newblock Clustering strategy and method selection.
\newblock In Christian Hennig, Marina Meila, Fionn Murtagh, and Roberto Rocci,
  editors, \emph{Handbook of Cluster Analysis}, chapter~31. CRC Press, 2015.

\bibitem[Kent et~al.(1991)Kent, Tyler, et~al.]{kent1991redescending}
John~T Kent, David~E Tyler, et~al.
\newblock Redescending $ m $-estimates of multivariate location and scatter.
\newblock \emph{The Annals of Statistics}, 19\penalty0 (4):\penalty0
  2102--2119, 1991.

\bibitem[{LeCun} and {Bottou}(2004)]{NORB}
Y.~{LeCun} and L.~{Bottou}.
\newblock Learning methods for generic object recognition with invariance to
  pose and lighting.
\newblock In \emph{Proceedings of the 2004 IEEE Computer Society Conference on
  Computer Vision and Pattern Recognition, 2004. CVPR 2004.}, volume~2, pages
  II--104 Vol.2, 2004.
\newblock \doi{10.1109/CVPR.2004.1315150}.

\bibitem[{Lecun} et~al.(1998){Lecun}, {Bottou}, {Bengio}, and {Haffner}]{MNIST}
Y.~{Lecun}, L.~{Bottou}, Y.~{Bengio}, and P.~{Haffner}.
\newblock Gradient-based learning applied to document recognition.
\newblock \emph{Proceedings of the IEEE}, 86\penalty0 (11):\penalty0
  2278--2324, 1998.
\newblock ISSN 0018-9219.
\newblock \doi{10.1109/5.726791}.

\bibitem[Lee and McLachlan(2014)]{hyper3}
S.~Lee and G.~J. McLachlan.
\newblock Finite mixtures of multivariate skew t-distributions: some recent and
  new results.
\newblock \emph{Statistics and Computing}, 24\penalty0 (2):\penalty0 181--202,
  March 2014.
\newblock ISSN 1573-1375.
\newblock \doi{10.1007/s11222-012-9362-4}.
\newblock URL \url{https://doi.org/10.1007/s11222-012-9362-4}.

\bibitem[{Liao} and {Couillet}(2017)]{Liao2017ALD}
Z.~{Liao} and R.~{Couillet}.
\newblock A large dimensional analysis of least squares support vector
  machines.
\newblock \emph{IEEE Transactions on Signal Processing}, 67:\penalty0
  1065--1074, 2017.

\bibitem[Maronna(1976)]{maronna}
R.~A. Maronna.
\newblock Robust {M-Estimators} of multivariate location and scatter.
\newblock \emph{The Annals of Statistics}, 4\penalty0 (1):\penalty0 51--67,
  1976.
\newblock ISSN 00905364.
\newblock URL \url{http://www.jstor.org/stable/2957994}.

\bibitem[{McInnes} and {Healy}(2017)]{HDBSCAN2}
L.~{McInnes} and J.~{Healy}.
\newblock Accelerated hierarchical density based clustering.
\newblock In \emph{2017 IEEE International Conference on Data Mining Workshops
  (ICDMW)}, pages 33--42, 2017.
\newblock \doi{10.1109/ICDMW.2017.12}.

\bibitem[McInnes et~al.(2017)McInnes, Healy, and Astels]{mcinnes2017hdbscan}
L.~McInnes, J.~Healy, and S.~Astels.
\newblock hdbscan: Hierarchical density based clustering.
\newblock \emph{J. Open Source Software}, 2\penalty0 (11):\penalty0 205, 2017.

\bibitem[{McInnes} et~al.(2018){McInnes}, {Healy}, and {Melville}]{UMAP}
L.~{McInnes}, J.~{Healy}, and J.~{Melville}.
\newblock {UMAP: Uniform Manifold Approximation and Projection for Dimension
  Reduction}.
\newblock \emph{arXiv e-prints}, 2018.

\bibitem[McLachlan(1982)]{MCLACHLAN1982199}
G.J. McLachlan.
\newblock 9 {T}he classification and mixture maximum likelihood approaches to
  cluster analysis.
\newblock In \emph{Classification Pattern Recognition and Reduction of
  Dimensionality}, volume~2 of \emph{Handbook of Statistics}, pages 199 -- 208.
  Elsevier, 1982.
\newblock \doi{https://doi.org/10.1016/S0169-7161(82)02012-4}.
\newblock URL
  \url{http://www.sciencedirect.com/science/article/pii/S0169716182020124}.

\bibitem[McNicholas(2016)]{nongaussian2}
P.D. McNicholas.
\newblock \emph{Mixture model-based classification}.
\newblock 10 2016.
\newblock \doi{10.1201/9781315373577}.

\bibitem[{Mitchell}(1997)]{20newsgroup}
T.~M. {Mitchell}.
\newblock \emph{Machine Learning}.
\newblock McGraw-Hill, Inc., New York, NY, USA, 1 edition, 1997.
\newblock ISBN 0070428077, 9780070428072.

\bibitem[{Ng} et~al.(2001){Ng}, {Jordan}, and {Weiss}]{spectral}
A.~Y. {Ng}, M.~I. {Jordan}, and Y.~{Weiss}.
\newblock On spectral clustering: Analysis and an algorithm.
\newblock In \emph{Proceedings of the 14th International Conference on Neural
  Information Processing Systems: Natural and Synthetic}, NIPS'01, pages
  849--856, Cambridge, MA, USA, 2001. MIT Press.
\newblock URL \url{http://dl.acm.org/citation.cfm?id=2980539.2980649}.

\bibitem[{Ollila} and {Tyler}(2012)]{ollilatyler}
E.~{Ollila} and D.~E. {Tyler}.
\newblock Distribution-free detection under complex elliptically symmetric
  clutter distribution.
\newblock In \emph{2012 IEEE 7th Sensor Array and Multichannel Signal
  Processing Workshop (SAM)}, pages 413--416, 2012.

\bibitem[Ollila et~al.(2012)Ollila, Tyler, Koivunen, and
  Poor]{ollila2012complex}
E.~Ollila, D.E. Tyler, V.~Koivunen, and H.V. Poor.
\newblock Complex elliptically symmetric distributions: Survey, new results and
  applications.
\newblock \emph{Signal Processing, IEEE Transactions on}, 60\penalty0
  (11):\penalty0 5597--5625, November 2012.
\newblock ISSN 1053-587X.
\newblock \doi{10.1109/TSP.2012.2212433}.

\bibitem[{Pascal} et~al.(2008){Pascal}, {Chitour}, {Ovarlez}, {Forster}, and
  {Larzabal}]{Fred}
F.~{Pascal}, Y.~{Chitour}, J-P. {Ovarlez}, P.~{Forster}, and P.~{Larzabal}.
\newblock Covariance structure maximum-likelihood estimates in compound
  gaussian noise: Existence and algorithm analysis.
\newblock \emph{Trans. Sig. Proc.}, 56\penalty0 (1):\penalty0 34--48, January
  2008.
\newblock ISSN 1053-587X.
\newblock \doi{10.1109/TSP.2007.901652}.
\newblock URL \url{http://dx.doi.org/10.1109/TSP.2007.901652}.

\bibitem[Pascal et~al.(2013)Pascal, Bombrun, Tourneret, and
  Berthoumieu]{pascal2013parameter}
F.~Pascal, L.~Bombrun, J-Y. Tourneret, and Y.~Berthoumieu.
\newblock Parameter estimation for multivariate generalized gaussian
  distributions.
\newblock \emph{IEEE Transactions on Signal Processing}, 61\penalty0
  (23):\penalty0 5960--5971, 2013.

\bibitem[Pedregosa et~al.(2011)Pedregosa, Varoquaux, Gramfort, Michel, Thirion,
  Grisel, Blondel, Prettenhofer, Weiss, Dubourg, Vanderplas, Passos,
  Cournapeau, Brucher, Perrot, and Duchesnay]{scikit-learn}
F.~Pedregosa, G.~Varoquaux, A.~Gramfort, V.~Michel, B.~Thirion, O.~Grisel,
  M.~Blondel, P.~Prettenhofer, R.~Weiss, V.~Dubourg, J.~Vanderplas, A.~Passos,
  D.~Cournapeau, M.~Brucher, M.~Perrot, and E.~Duchesnay.
\newblock {Scikit-learn: Machine Learning in Python }.
\newblock \emph{Journal of Machine Learning Research}, 12:\penalty0 2825--2830,
  2011.

\bibitem[Peel and McLachlan(2000)]{Peel2000}
D.~Peel and G.~J. McLachlan.
\newblock Robust mixture modelling using the t distribution.
\newblock \emph{Statistics and Computing}, 10\penalty0 (4):\penalty0 339--348,
  Oct 2000.
\newblock ISSN 1573-1375.
\newblock \doi{10.1023/A:1008981510081}.
\newblock URL \url{https://doi.org/10.1023/A:1008981510081}.

\bibitem[{Roizman} et~al.(2020){Roizman}, {Jonckheere}, and
  {Pascal}]{mahalanobiseusip}
V.~{Roizman}, M.~{Jonckheere}, and F.~{Pascal}.
\newblock Robust clustering and outlier rejection using the mahalanobis
  distance distribution.
\newblock In \emph{2020 28th European Signal Processing Conference, EUSIPCO},
  2020.
\newblock To appear.

\bibitem[{Rousseeuw}(1987)]{silhouettes}
P.~J. {Rousseeuw}.
\newblock Silhouettes: A graphical aid to the interpretation and validation of
  cluster analysis.
\newblock \emph{Journal of Computational and Applied Mathematics}, 20:\penalty0
  53 -- 65, 1987.
\newblock ISSN 0377-0427.
\newblock \doi{https://doi.org/10.1016/0377-0427(87)90125-7}.
\newblock URL
  \url{http://www.sciencedirect.com/science/article/pii/0377042787901257}.

\bibitem[scikit{-}learn developers(2019)]{sklearn}
scikit{-}learn developers.
\newblock Clustering--scikit-learn v0.20.3, 2019.
\newblock URL
  \url{https://scikit-learn.org/stable/modules/clustering.html#clustering}.

\bibitem[{Tadjudin} and {Landgrebe}(2000)]{weights2}
S.~{Tadjudin} and D.~A. {Landgrebe}.
\newblock Robust parameter estimation for mixture model.
\newblock \emph{IEEE Transactions on Geoscience and Remote Sensing},
  38\penalty0 (1):\penalty0 439--445, 2000.
\newblock ISSN 0196-2892.
\newblock \doi{10.1109/36.823939}.

\bibitem[Tyler(1987)]{tyler1987distribution}
D.~E. Tyler.
\newblock {A distribution-free $M$-estimator of multivariate scatter}.
\newblock \emph{The Annals of Statistics}, 15\penalty0 (1):\penalty0 234--251,
  1987.

\bibitem[{van der Maaten} and {Hinton}(2008)]{tSNE}
L.~{van der Maaten} and G.~{Hinton}.
\newblock Visualizing data using {t-SNE}.
\newblock \emph{Journal of machine learning research}, 9\penalty0
  (Nov):\penalty0 2579--2605, 2008.

\bibitem[Vinh et~al.(2010)Vinh, Epps, and
  Bailey]{Vinh:2010:ITM:1756006.1953024}
N.~X. Vinh, J.~Epps, and J.~Bailey.
\newblock Information theoretic measures for clusterings comparison: Variants,
  properties, normalization and correction for chance.
\newblock \emph{J. Mach. Learn. Res.}, 11:\penalty0 2837--2854, December 2010.
\newblock ISSN 1532-4435.
\newblock URL \url{http://dl.acm.org/citation.cfm?id=1756006.1953024}.

\bibitem[{Wang} et~al.(2009){Wang}, {Ng}, and {McLachlan}]{emmix}
K.~{Wang}, S.~{Ng}, and G.~J. {McLachlan}.
\newblock Multivariate skew t mixture models: Applications to
  fluorescence-activated cell sorting data.
\newblock In \emph{2009 Digital Image Computing: Techniques and Applications},
  pages 526--531, 2009.

\bibitem[Weber and Robinson(2016)]{doi:10.1002/cyto.a.23030}
L.~M. Weber and M.~D. Robinson.
\newblock Comparison of clustering methods for high-dimensional single-cell
  flow and mass cytometry data.
\newblock \emph{Cytometry Part A}, 89\penalty0 (12):\penalty0 1084--1096, 2016.
\newblock \doi{10.1002/cyto.a.23030}.
\newblock URL
  \url{https://onlinelibrary.wiley.com/doi/abs/10.1002/cyto.a.23030}.

\bibitem[{Wei} et~al.(2017){Wei}, {Tang}, and {McNicholas}]{hyper}
Y.~{Wei}, Y.~{Tang}, and P.~D. {McNicholas}.
\newblock {Mixtures of Generalized Hyperbolic Distributions and Mixtures of
  Skew-t Distributions for Model-Based Clustering with Incomplete Data}.
\newblock \emph{arXiv e-prints}, Mar 2017.

\bibitem[Wu et~al.(2016)Wu, Yang, Zhao, and Zhu]{wu2016convergence}
Chong Wu, Can Yang, Hongyu Zhao, and Ji~Zhu.
\newblock On the convergence of the em algorithm: A data-adaptive analysis,
  2016.

\bibitem[Yao(1973)]{Yao73}
K.~Yao.
\newblock A representation theorem and its applications to spherically
  invariant random processes.
\newblock \emph{IEEE Trans.-IT}, 19\penalty0 (5):\penalty0 600--608, September
  1973.

\bibitem[{Yu} et~al.(2015){Yu}, {Dang}, {Bart}, and {Chen}]{Spatial-EM}
K.~{Yu}, X.~{Dang}, H.~{Bart}, and Y.~{Chen}.
\newblock Robust model-based learning via spatial-em algorithm.
\newblock \emph{IEEE Transactions on Knowledge and Data Engineering},
  27\penalty0 (6):\penalty0 1670--1682, 2015.
\newblock ISSN 1041-4347.
\newblock \doi{10.1109/TKDE.2014.2373355}.

\bibitem[Zhang et~al.(2016)Zhang, Cheng, and Singer]{ZHANG2016114}
T.~Zhang, X.~Cheng, and A.~Singer.
\newblock Marcenko pastur law for {Tyler}'s {M}-estimator.
\newblock \emph{Journal of Multivariate Analysis}, 149:\penalty0 114 -- 123,
  2016.
\newblock ISSN 0047-259X.
\newblock \doi{https://doi.org/10.1016/j.jmva.2016.03.010}.
\newblock URL
  \url{http://www.sciencedirect.com/science/article/pii/S0047259X16300069}.

\end{thebibliography}

\end{document}